\documentclass[a4paper,fleqn]{cas-sc}
\usepackage[authoryear,longnamesfirst]{natbib}
\usepackage{geometry}
\usepackage{graphicx}
\usepackage{hyperref}
\usepackage{subfigure}
\usepackage{xcolor}
\usepackage{multirow}      % for \multirow
\usepackage{adjustbox}     % for \adjustbox{width=...}
\usepackage{threeparttable} % for \threeparttable and \tablenotes
\usepackage{pgfplots}
\usepackage{subcaption}
\pgfplotsset{compat=1.18}
\usepackage{amsmath}
\usepackage{algorithm} 
\usepackage{algorithmic}
\usepackage{longtable}
\usepackage{booktabs}
\usepackage{array}
\usepackage{colortbl}
\usepackage{amsthm}
\newtheorem{theorem}{Theorem}

\newtheorem*{remark}{Remark}
\theoremstyle{definition}

\newtheorem{proposition}{Proposition}
\usepackage{pgfplots}
\pgfplotsset{compat=1.18}
\graphicspath{{./}}
\def\tsc#1{\csdef{#1}{\textsc{\lowercase{#1}}\xspace}}
\tsc{WGM}
\tsc{QE}
\tsc{EP}
\tsc{PMS}
\tsc{BEC}
\tsc{DE}
\begin{document}
\let\WriteBookmarks\relax
\def\floatpagepagefraction{1}
\def\textpagefraction{.001}
\shorttitle{WaVeFuse}
\shortauthors{A. Bohra et~al.}
%\begin{frontmatter}

\title [mode = title]{WaVeFuse: Regime-Adaptive Equity Index Forecasting via Channel-Wise Wavelet Denoising and Vertical Attention Fusion} 

% =========================================================
% AUTHORS
% =========================================================

\author[1]{Aashish Bohra}[orcid=0009-0004-5871-4310]
\ead{bohra.1@iitj.ac.in}
\author[2]{Vivek Vijay}
\ead{vivek@iitj.ac.in}

% \author[1]{Third Author Full Name}
% \ead{third.author@email.com}

% \author[3]{Fourth Author Full Name}
% \cormark[1]
% \ead{bohra.1@iitj.ac.in}

% =========================================================
% AFFILIATIONS
% =========================================================

\affiliation[1]{
    organization={Department of Computer Science and Engineering, Indian Institute of Technology Jodhpur},
    % addressline={University Address},
    city={Jodhpur},
    % postcode={PIN Code},
    state={Rajasthan},
    country={India}
}

\affiliation[2]{
    organization={Department of Mathematics, Indian Institute of Technology Jodhpur},
    % addressline={University Address},
    city={Jodhpur},
    % postcode={PIN Code},
    state={Rajasthan},
    country={India}
}

% \affiliation[3]{
%     organization={Department/School Name, Third University Name},
%     addressline={University Address},
%     city={City},
%     postcode={PIN Code},
%     state={State},
%     country={India}
% }

% Corresponding author
\cortext[cor1]{Corresponding author}

\begin{abstract}
Hybrid Deep Learning for equity index forecasting is limited by three problems: propagation of OHLCV noise into derived technical indicators (TIs), channel-indiscriminate multi-scale decomposition that conflates heterogeneous frequency signatures, and static multi-branch fusion that cannot adapt to market regime shifts. WaVeFuse addresses these limitations through a unified dual-branch architecture. Symlet-4 wavelet denoising (level 2, MAD soft threshold) suppresses microstructure noise in OHLCV. Seven low-lag TIs computed from denoised prices are encoded by a causal channel-wise continuous wavelet transform (Morlet, 32 scales) into a per-timestep scale-space matrix. A CNN-BiLSTM branch captures temporal dynamics, while a dual-layer Transformer (heads $=$ 4, $d_k \in \{16, 32\}$) models inter-scale spectral dependencies, and their representations are integrated by a 2-token softmax gate Vertical Attention Fusion (VAF) that dynamically reweights branches as market regimes shift. Evaluated under walk-forward validation (WFV) on KOSPI, DAX, NYSE Composite, and Russell 2000 (2010-2023), WaVeFuse achieves $R^2 = 0.81$-$0.96$ and directional accuracy $70.5$-$78.3\%$. It outperforms seven state-of-the-art models by $8.9$-$20.2\%$ MAE across twelve dataset-period configurations. Diebold-Mariano statistics ($4.62$-$10.38$, $p<0.001$) confirm superiority over a well-tuned XGBoost benchmark across four indices. Ablation verifies component-wise contributions. Under realistic backtesting with 10 basis point transaction costs, WaVeFuse's directional strategy achieves a mean Sharpe ratio of $3.69$ across four markets and limits maximum drawdown to $7.5\%$ during the COVID-19 crash. With $152$k parameters ($0.68$~MB) and sub-$1.3$~ms GPU inference, WaVeFuse delivers a computationally efficient, regime-robust framework suitable for research and decision-support deployment.
\end{abstract}


\maketitle

% \section*{List of Abbreviations}
\label{app:abbreviations}
\small
\setlength{\tabcolsep}{4pt}
\renewcommand{\arraystretch}{1.15}
\begin{longtable}{p{1.6cm} p{5.2cm} p{1.6cm} p{5.2cm}}
\caption{List of abbreviations used throughout the paper.}
\label{tab:abbreviations} \\
\toprule
\textbf{Abbrev.} & \textbf{Full Form} & \textbf{Abbrev.} & \textbf{Full Form} \\
\midrule
\endfirsthead
\toprule
\textbf{Abbrev.} & \textbf{Full Form} & \textbf{Abbrev.} & \textbf{Full Form} \\
\midrule
\endhead
\midrule
\multicolumn{4}{r}{\small\textit{Continued on next page}} \\
\endfoot
\bottomrule
\endlastfoot

\multicolumn{4}{l}{\textit{\textbf{Architecture \& Models}}} \\[2pt]
WaVeFuse & Wavelet-based Vertical Fusion architecture & VAF & Vertical Attention Fusion \\
CWWT      & Channel-Wise Continuous Wavelet Transform                       & CNN      & Convolutional Neural Network \\
BiLSTM    & Bidirectional Long Short-Term Memory                            & LSTM     & Long Short-Term Memory \\
BiGRU     & Bidirectional Gated Recurrent Unit                              & TCN      & Temporal Convolutional Network \\
DNN       & Deep Neural Network                                             & ANN      & Artificial Neural Network \\
DL        & Deep Learning                                                   & MHA      & Multi-Head Attention \\
ReLU      & Rectified Linear Unit                                           & LN       & Layer Normalization \\
XGBoost   & Extreme Gradient Boosting                                       & GA       & Genetic Algorithm \\

\midrule
\multicolumn{4}{l}{\textit{\textbf{Wavelet \& Signal Processing}}} \\[2pt]
DWT       & Discrete Wavelet Transform                                      & CWT      & Continuous Wavelet Transform \\
IDWT      & Inverse Discrete Wavelet Transform                              & STFT     & Short-Time Fourier Transform \\
Sym-4      & Symlet-4 Wavelet                                                & MAD      & Median Absolute Deviation \\
EMD       & Empirical Mode Decomposition                                    & CEEMD    & Complementary Ensemble EMD \\
ICEEMDAN  & Improved Complete Ensemble EMD with Adaptive Noise              & VMD      & Variational Mode Decomposition \\
MEMD      & Multivariate Empirical Mode Decomposition                       & SWT      & Stationary Wavelet Transform \\
IMF       & Intrinsic Mode Function                                         & PSO      & Particle Swarm Optimization \\
MRA       & Multi-Resolution Analysis                                       & Besov    & Besov space \\

\midrule
\multicolumn{4}{l}{\textit{\textbf{Technical Indicators}}} \\[2pt]
TI        & Technical Indicator                                             & RSI      & Relative Strength Index \\
CCI       & Commodity Channel Index                                         & OBV      & On-Balance Volume \\
ATR       & Average True Range                                              & ROC      & Rate of Change \\
Stoch\%K  & Stochastic Oscillator \%K                                       & MACD     & Moving Average Convergence Divergence \\
SMA       & Simple Moving Average                                           & EMA      & Exponential Moving Average \\

\midrule
\multicolumn{4}{l}{\textit{\textbf{Financial Data \& Indices}}} \\[2pt]
OHLCV     & Open, High, Low, Close, Volume                                  & OHLC     & Open, High, Low, Close \\
KOSPI     & Korea Composite Stock Price Index                               & DAX      & Deutscher Aktienindex \\
NYSE      & New York Stock Exchange Composite                               & DJI      & Dow Jones Industrial Average \\
HSI       & Hang Seng Index                                                 & IXIC     & NASDAQ Composite Index \\
SSEC      & Shanghai Stock Exchange Composite                               & GFC      & Global Financial Crisis \\

\midrule
\multicolumn{4}{l}{\textit{\textbf{Validation, Metrics \& Statistics}}} \\[2pt]
WFV       & Walk-Forward Validation                                         & DM       & Diebold--Mariano Test \\
HAC       & Heteroskedasticity and Autocorrelation Consistent               & MAE      & Mean Absolute Error \\
RMSE      & Root Mean Squared Error                                         & MAPE     & Mean Absolute Percentage Error \\
MSE       & Mean Squared Error                                              & DA       & Directional Accuracy \\
CAGR      & Compound Annual Growth Rate                                     & CV       & Coefficient of Variation \\ 
DD       & Maximum Drawdown                                                 & Vol      & Annualized Volatility \\

\midrule
\multicolumn{4}{l}{\textit{\textbf{Optimization \& Training}}} \\[2pt]
Adam      & Adaptive Moment Estimation                                       & AMSGrad  & Adam with AMSGrad \\
Huber     & Huber Loss                                                      & SGD      & Stochastic Gradient Descent \\
Glorot    & Glorot uniform initialization                                   &          & \\
\end{longtable}

\section{Introduction} \label{sec:intro}
Hybrid Deep Learning (DL) models for equity index forecasting have advanced rapidly, with recent architectures reporting $R^2$ values exceeding $0.97$ on out-of-sample test sets \citep{tian2025bimt, ge2025enhancing} and directional accuracies above $70\%$ on major indices \citep{ji2024galformer, gong2024predicting}. These figures suggest that the forecasting problem is largely solved, yet three structural limitations persist across virtually all published architectures, silently undermining the reliability of these reported gains. First, wavelet-based denoising is routinely applied only to raw price-volume series, while the technical indicators (TI) derived from those prices (e.g., RSI, ATR, CCI, and related signals) are computed from the noisy inputs and therefore inherit the very contamination the denoising was designed to remove. Second, multi-scale decomposition of indicator features is performed in an aggregate, channel-indiscriminate manner that treats a bounded momentum oscillator identically to an unbounded volume-accumulation signal, conflating their fundamentally different frequency signatures. Third, fusion of multi-branch representations remains predominantly static, applying fixed weights that cannot respond to instantaneous market regime transitions. These are the very transitions that determine whether a trend-following or volatility-sensitive strategy is appropriate. To the best of our knowledge, no existing model resolves all three limitations within a unified, computationally efficient end-to-end architecture. WaVeFuse is designed specifically to fill this gap.

Predicting stock prices has long been a critical research focus in financial markets, capturing the interest of investors, traders, and researchers alike \citep{de2021investment, dash2023fine, nelson2017stock, bao2017deep}. The ability to forecast future price movements, even with a modest level of precision, can lead to significant financial gains and 
is crucial for portfolio management, risk assessment, and strategic planning \citep{thakkar2021comprehensive, jiang2021applications, chong2017deep}. Historically, stock price prediction methodologies have evolved significantly, increasing both in complexity and sophistication. Early prediction models relied mainly on statistical approaches and basic 
time series analysis techniques, which often failed to capture market volatility \citep{bukhari2024predictive}. Traditionally, stock price prediction strategies are classified into fundamental and technical analysis approaches \citep{moghaddam2016stock, ravi2017financial, chandar2024deep}. Fundamental analysis evaluates a company's intrinsic value through its financial statements, market position, and potential for future growth. In contrast, technical analysis relies on historical price and volume data to extract recurring momentum and volatility patterns. While the strict weak‑form Efficient Market Hypothesis (EMH) posits that historical price trajectories contain no exploitable predictive information, the broader empirical literature consistently documents short‑horizon predictability driven by liquidity frictions and behavioral momentum. This dynamic aligns with the Adaptive Markets Hypothesis (AMH) \citep{lo2004adaptive}, which formalizes how temporary market inefficiencies arise, shift, and dissipate as macroscopic regimes evolve. Consequently, the increasing non‑linearity and regime‑shifting volatility of modern financial markets have exposed the limitations of traditional static models, motivating the adoption of data‑driven DL architectures capable of dynamically adapting to these transient inefficiencies.

Recent advancements in computational finance and DL have enabled a more thorough and systematic financial data analysis, addressing traditional models' shortcomings \citep{tian2025bimt, ji2024galformer, kong2025deep}. Modern approaches now leverage various neural network models such as Deep Neural Networks (DNNs) \citep{yong2017stock}, Artificial Neural Networks (ANNs), Convolutional Neural Networks (CNNs) \citep{lecun1998gradient}, Long Short-Term Memory (LSTMs) \citep{chen2015lstm}. CNNs are skilled at capturing spatial dependencies in structured data, while LSTMs excel at 
modeling sequential dependencies, which makes them widely used for time series tasks. More recently, transformer models have shown tremendous potential for analyzing sequential financial data due to their superior ability to capture long-term temporal dependencies through self-attention mechanisms \citep{beniwal2024forecasting}. These methods attempt to capture the non-linear relationships, temporal dependencies, and intricate patterns inherent in stock market data, surpassing the limitations of traditional approaches.

However, stock price prediction still remains challenging due to extreme market volatility from economic indicators, geopolitical events, and sentiment shifts \citep{yang2024hierarchical}. These factors introduce high-frequency noise and non-stationarity, hindering trend identification and causing poor generalization across indices, timeframes, and economic cycles \citep{siami2019comparative, wang2019ean, sivadasan2024stock, deng2024multi}. Unpredictable events further exacerbate performance inconsistencies in models that assume stationarity or rely solely on historical patterns \citep{behera2023prediction}. Effective forecasting thus requires architectures that robustly handle noise, non-linearity, 
and regime shifts. Among these challenges, feature noise contamination and suboptimal integration of heterogeneous financial signals remain particularly underexplored in hybrid architectures.

Despite rapid progress in hybrid DL architectures for financial time series \citep{tian2025bimt, ji2024galformer, zhang2024two}, several critical limitations persist in handling TIs and multi-branch integration. First, wavelet-based denoising is typically applied only to raw OHLC prices, leaving derived TIs contaminated by propagated high-frequency noise \citep{rezaei2021stock}. Second, existing approaches lack channel-wise multi-scale decomposition of indicator features, limiting their ability to capture scale-specific dynamics across momentum, volatility, and trend signals. Third, feature fusion in multi-branch designs remains predominantly static (e.g., simple concatenation or basic attention), failing to fully exploit complementary spatio-temporal and global contextual representations \citep{tian2025bimt, ji2024galformer}. These limitations collectively restrict predictive robustness and cross-market generalization, highlighting the need for a unified architecture capable of multi-scale denoising and adaptive feature fusion.

To address these challenges, we propose WaVeFuse, a unified multi-scale noise-aware dual-branch DL framework that jointly integrates comprehensive wavelet denoising, channel-wise refinement, and adaptive feature fusion within a coherent end-to-end architecture. WaVeFuse first applies Symlet-4 (Sym-4) DWT denoising (level 2, soft thresholding) exclusively to the raw OHLCV series to suppress high-frequency microstructure noise. Seven low-lag TIs (RSI-10, Stochastic \%K, CCI, OBV, ATR, Williams \%R, ROC-12) are then computed from the denoised prices. CWWT with the Morlet mother wavelet (32 scales) extracts instantaneous multi-scale representations 
of these indicators. The architecture combines a CNN-BiLSTM temporal branch on denoised OHLCV sequences with a Transformer branch on the Channel-Wise Continuous Wavelet Transform (CWWT) features. A learnable Vertical Attention Fusion (VAF) mechanism dynamically weights the two branches according to the instantaneous market regime. The fused representation is decoded by a CNN-BiLSTM module to produce next-day closing price forecasts. Extensive experiments across diverse market regimes demonstrate statistically significant improvements in both regression and directional accuracy (DA) metrics, validating the robustness and generalizability of the proposed framework. The major contributions of this research highlighting the technical advancements and innovations introduced by WaVeFuse are as follows:

\begin{enumerate}
    \item We propose WaVeFuse, a unified end-to-end dual-branch architecture that jointly resolves all three identified limitations, which are 1. wavelet denoising of raw OHLCV channels prior to indicator computation, 2. channel-wise continuous wavelet spectral encoding of derived indicators, and 3. regime-adaptive VAF of temporal and spectral branch representations.

    \item We apply Sym-4 DWT denoising exclusively to raw OHLCV data and derive seven low-lag TIs from the denoised series. CWWT with Morlet wavelet then extracts multi-scale spectral representations of these indicators, preserving essential market trends while removing propagated high-frequency noise.

    \item We employ a VAF technique that combines the latent representation from the CNN-BiLSTM encoder with the feature vector extracted from the penultimate layer of the transformer network, providing a comprehensive input for the CNN-BiLSTM decoder.

    \item We conduct an extensive comparative analysis with state-of-the-art models on the same datasets and time periods. This direct comparison highlights the quantitative superiority and generalizability of WaVeFuse across diverse market conditions, demonstrating its novelty and significance.
\end{enumerate}

The remainder of this paper is organized as follows: Section~\ref{sec:relatedwork} provides a comprehensive review of related work, categorizing existing approaches into OHLC-based and hybrid OHLC+TI models. Section~\ref{sec:overview} details the proposed WaVeFuse methodology, including wavelet denoising, dual-branch architecture, VAF, and model training. Section~\ref{sec:result_analysis} presents experimental results, including comparisons with state-of-the-art models, statistical significance tests, trading strategy performance, and computational analysis. Section~\ref{sec:discussion} discusses the implications and limitations of our findings. Finally, Section~\ref{sec:conclusion} concludes the paper and outlines future research directions. A complete list of abbreviations is provided in Table~\ref{tab:abbreviations}.

\section{Related Work} \label{sec:relatedwork}
Deep Learning architectures for financial time series forecasting have evolved rapidly, with hybrid models integrating decomposition techniques, recurrent processing, convolutional feature extraction, and attention mechanisms to address the inherent non-stationarity, nonlinearity, and multifractality of market data. Recent contributions (2023-2026) prioritize multi-scale signal processing and representational orthogonality to disentangle trend, volatility, and momentum components, yet systematic limitations persist in the treatment of TIs and cross-branch integration. Table \ref{tab:related_work_comparison} illustrates the comparative analysis of state-of-the-art hybrid forecasting models with proposed WaVeFuse. 

\subsection{Decomposition-Based Hybrids} \label{subsec:related-1}
Decomposition-based hybrids dominate contemporary approaches by partitioning raw series into frequency-specific subsequences prior to neural forecasting. \cite{rezaei2021stock} applied CEEMD and EMD to univariate closing prices of S\&P 500, Dow Jones, DAX, and Nikkei 225 (2010-2019), extracting convolutional features from each Intrinsic Mode Function via single-layer 1D CNN before LSTM prediction and linear aggregation, achieving 4-9\% RMSE reductions over CEEMD-LSTM baselines through enhanced local pattern capture, though restricted to shallow architectures and price-only inputs. \cite{ge2025enhancing} extended multivariate empirical mode decomposition globally across OHLCV series from eight global indices (2013-2022), pairing Aquila-optimized LSTM hyperparameters (hidden units 32, batch size 16) with component-wise forecasting, attaining R$^2$ $\geq$ 0.9869 and robust event capture during COVID-19 and Russia-Ukraine periods, surpassing standalone LSTM and alternative hybrids but incurring substantial population-based optimization overhead. \cite{yu2025intelligent} employed genetic algorithm minimization of sample entropy to select variational mode decomposition parameters before Temporal Convolutional Network  forecasting of multi-horizon closing prices on SSEC, S\&P 500, and NASDAQ (2012-2021), securing modified DM rejection at the 1\% level across 1-10 steps with Bry-Boschan cycle robustness, yet remaining univariate and without uncertainty quantification. \cite{gong2024predicting} introduced two-stage re-decomposition, ICEEMDAN followed by PSO-tuned VMD of the highest-frequency mode. Before novel BiLSTM-SAM-TCN ensembles on daily high/low prices of S\&P 500 and SSEC (2010-2023), yielding MAPE $\leq$ 0.51\% with modified DM confirmation across scales and crises, though price-only and reliant on fixed splits.  \cite{li2026importance} performed level-1 stationary wavelet transform (Daubechies-4) on multivariate OHLCV plus deviation features before multi-layer BiLSTM on Apple and Tesla (2010-2023), reporting up to 274$\times$ MSE reduction versus raw recurrent baselines through multiresolution non-stationarity mitigation, albeit confined to two US growth stocks with occasional convergence instability.

\subsection{Pure Architectural Hybrids} \label{subsec:related-2}
Pure architectural hybrids without explicit upfront decomposition emphasize complementary neural mechanisms. \cite{bhandari2022predicting} demonstrated single-layer LSTM (150 neurons, Adagrad) superiority over multilayer variants on Haar-denoised macroeconomic-technical-augmented S\&P 500 data (2006-2021), with Welch's $t$-test significance during crises. \cite{tian2025bimt} fused bidirectional LSTM, modified Transformer encoder, and TCN decoder on 15-day closing windows of SSE, HSI, and NASDAQ (2017-2024), reaching $R^2 \geq 0.9776$ via ablation-validated synergy. \cite{ji2024galformer} presented Galformer with non-autoregressive generative decoding and hybrid MSE-trend loss on four US/China indices (2011-2021), accelerating inference $\sim$136$\times$ while elevating directional accuracy over Informer. \cite{zubair2026enhanced} extended encoder-decoder paradigms with BiGRU core, autoencoder feature learning, attention weighting, and skip connections on univariate closing prices across nine cross-market assets (2017/2018-2024), securing 48-73\% MAPE reductions with exhaustive modified DM validation during crises. \cite{thakkar2022information} jointly optimized LSTM topology and binary TI inclusion via information fusion-inspired genetic algorithm across nine datasets, achieving 43\% MSE and 61\% $R^2$ gains over standard GA, though constrained by fixed two-layer topology.

\subsection{Critical Analysis and Limitations} \label{subsec:related-3} 
Critical analyses expose methodological vulnerabilities and bounded signal strength. \cite{radfar2025stock} revealed day-to-day LSTM lagging artifacts yielding illusory accuracies on Tehran Exchange stocks, with extrapolation-focused CNN/Transformer variants marginally exceeding constant-price baselines ($\sim$0.5\%), underscoring minimal predictive content in chart data alone. Despite these advances, three intertwined limitations persist. First, wavelet or decomposition techniques are applied exclusively to raw OHLCV series, leaving derived TIs vulnerable to propagated high-frequency noise that contaminates momentum, volatility, and overbought/oversold signals. Second, multi-scale processing of indicators remains aggregate rather than channel-wise, conflating distinct frequency characteristics of heterogeneous oscillators (e.g., bounded RSI versus unbounded OBV). Third, multi-branch fusion is predominantly static via concatenation, addition, or fixed attention, failing to dynamically adapt weighting to instantaneous market regimes where temporal momentum may dominate during trends while spectral anomalies signal volatility shifts.

WaVeFuse directly resolves these deficiencies through (i) Sym-4 DWT denoising applied to OHLCV channels (TIs computed post-denoising), (ii) CWWT with Morlet mother wavelet extracting instantaneous scale‑space representations per indicator (32 scales) for Transformer‑based inter‑scale dependency modeling, preserving unique frequency signatures, and (iii) VAF performing regime‑adaptive convex combination of CNN‑BiLSTM temporal bottleneck and Transformer penultimate spectral features via learned 2‑token self‑attention, minimizing conditional prediction variance. Evaluated under rigorous sliding-window WFV on KOSPI, DAX, NYSE Composite, and Russell 2000 (2010-2023), WaVeFuse establishes superior generalization across developed and emerging regimes while maintaining constant-time inference complexity. To the best of the authors’ knowledge, WaVeFuse is the first model to integrate wavelet-denoised TIs, channel-wise CWT spectral encoding, and a regime-aware dynamic fusion mechanism within a single end-to-end architecture.

\begin{table}[htbp]
\centering
\caption{Systematic Comparison of Related Hybrid Forecasting Models}
\label{tab:related_work_comparison}
\resizebox{\textwidth}{!}{
\begin{tabular}{@{}>{\raggedright\arraybackslash}p{2.0cm}
                 >{\raggedright\arraybackslash}p{2.0cm}
                 >{\raggedright\arraybackslash}p{3.0cm}
                 >{\raggedright\arraybackslash}p{2.0cm}
                 >{\raggedright\arraybackslash}p{2.0cm}
                 >{\raggedright\arraybackslash}p{4.0cm}
                 >{\raggedright\arraybackslash}p{4.0cm}@{}}
\toprule
\textbf{Paper} & \textbf{Decomposition Applied to TIs} & \textbf{Channel-Wise Multi-Scale TI Processing} & \textbf{Dynamic Fusion Mechanism} & \textbf{Datasets} & \textbf{Key Innovation} & \textbf{Primary Limitation Relative to WaVeFuse} \\
\midrule
\cite{rezaei2021stock} & No & No & Linear aggregation & 4 indices (2010-2019) & CNN feature extraction per IMF & Univariate price, shallow layers, no TI denoising \\
\cite{ge2025enhancing} & No & No & Linear aggregation & 8 indices (2013-2022) & MEMD + Aquila-optimized LSTM & High compute, OHLCV only, static fusion \\
\cite{yu2025intelligent} & No & No & Linear aggregation & 3 indices (2012-2021) & GA-optimized VMD + TCN & Univariate close, no channel-wise TI scales \\
\cite{gong2024predicting} & No & No & Linear aggregation & 2 indices high/low (2010-2023) & ICEEMDAN + PSO-VMD re-decomp. + BiLSTM-SAM-TCN & Price-only, static fusion \\
\cite{li2026importance} & Partial (level-1 SWT on derived deviation) & No & None (single-branch) & 2 stocks (2010-2023) & SWT preprocessing + multi-layer BiLSTM & Only 2 assets, no channel-wise TI CWT, no fusion \\
\cite{tian2025bimt} & No & No & End-to-end joint training & 3 indices (2017-2024) & BiLSTM + mod. Transformer + TCN & Univariate close, no explicit dynamic fusion \\
\cite{ji2024galformer} & No & No & None (single-branch) & 4 indices (2011-2021) & Generative decoding + hybrid loss & Univariate adjusted close, no multi-branch \\
\cite{zubair2026enhanced} & No & No & Attention + skip connections & 9 assets (2017-2024) & BiGRU encoder-decoder + DAE + AM + SC & Univariate close, no wavelet denoising or channel-wise scales \\
\cite{thakkar2022information} & No & No & None & 9 datasets (2000-2020) & GA joint optimization of LSTM + binary TIs & Fixed topology, binary selection, no multi-scale processing \\
\cite{radfar2025stock} & No & No & None & 12 TSE stocks & Extrapolation CNN exposing LSTM flaws & Marginal gains over constant baseline, small sample \\
\bottomrule
\end{tabular}}
\end{table}

\section{WaVeFuse: Methodology}\label{sec:overview}
WaVeFuse is a dual-branch hybrid neural architecture designed to model financial time series through complementary representational domains. The first branch captures local temporal dependencies in denoised OHLCV sequences using a CNN-BiLSTM encoder, while the second branch models multi-scale spectral characteristics of derived TIs through a CWT followed by a Transformer encoder. These two latent representations are integrated via a learnable VAF mechanism that adaptively weights branch contributions based on learned contextual features. The fused representation is subsequently processed by a decoder to produce next-step closing price forecasts.

The architectural design separates temporal dynamics from instantaneous scale-space structure. The temporal branch operates directly on causal lookback windows of denoised price-volume data, capturing sequential patterns such as momentum and short-term trend persistence. In contrast, the spectral branch processes channel-wise CWT representations of TI's, treating wavelet scales as structured tokens for inter-scale attention modeling. This separation promotes complementary feature learning while reducing representational redundancy across branches. As shown in the Figure~\ref{fig:Methodology} the complete pipeline of WaVeFuse is divided into following stages.
\begin{enumerate}
    \item Data Acquisition and Preprocessing
    \item Multi-Resolution Signal Decomposition
    \item Technical Indicators and Spectral Encoding (CWWT)
    \item Dual-Branch Encoder Architecture
    \item VAF Mechanism
    \item Decoder and Training Objective
\end{enumerate}
Before deep diving into these stages, it is highly important to understand the problem formulation of WaVeFuse.  

\begin{figure}
    \centering
    \includegraphics[width=\linewidth]{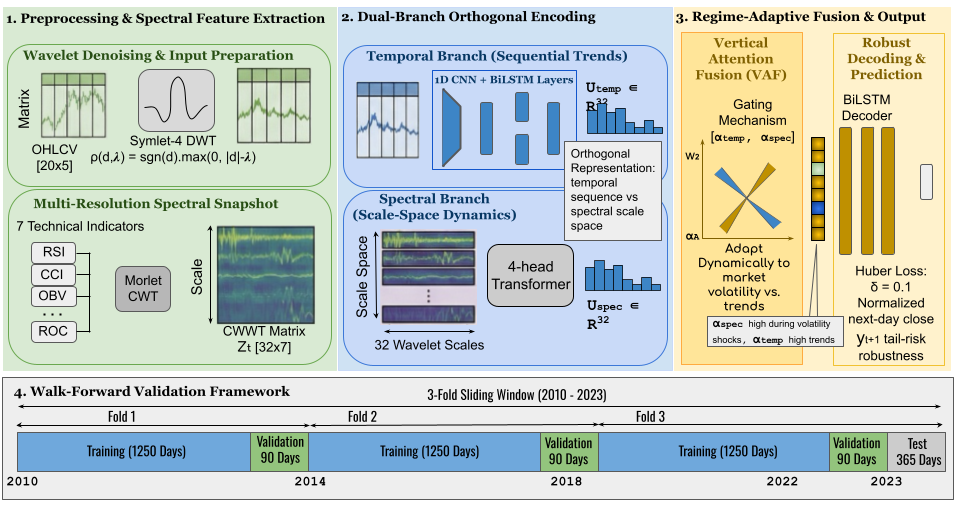}
    \caption{The WaVeFuse Architecture. Left: Preprocessing pipeline with wavelet denoising and sliding-window validation. Center: Dual-branch encoder (CNN-BiLSTM for temporal dynamics; Transformer for spectral-scale dependencies). Right: VAF and prediction decoder.}
    \label{fig:Methodology}
\end{figure}

\subsection{Problem Formulation and Evaluation Protocol}\label{sec:problem}
Let $\mathcal{P}_t = \{O_t, H_t, L_t, C_t, V_t\} \in \mathbb{R}^5$ denote the OHLCV (Open, High, Low, Close, Volume) vector at time $t$, adapted to filtration $\mathcal{F}_t$ representing market information up to $t$. We formulate the one-step ahead price prediction as a supervised regression task:
Given a causal lookback window of $w$ historical observations $\mathbf{X}_t = [\mathcal{P}_{t-w+1}, \ldots, \mathcal{P}_t] \in \mathbb{R}^{w \times 5}$ and corresponding spectral features $\mathbf{Z}_t \in \mathbb{R}^{S \times M}$ derived from TIs (detailed in Section~\ref{sec:cwwt}), the objective is to learn a measurable function $f_\theta: \mathbb{R}^{w \times 5} \times \mathbb{R}^{S \times M} \rightarrow \mathbb{R}$ parameterized by $\theta$ such that:
\begin{equation}\label{eq:prediction}
\hat{C}_{t+1} = f_\theta(\mathbf{X}_t, \mathbf{Z}_t) + \epsilon_t,
\end{equation}
where $\epsilon_t \sim \mathcal{D}$ represents zero-mean irreducible noise following an unknown, potentially heavy-tailed distribution with $\mathbb{E}[\epsilon_t \mid \mathcal{F}_t] = 0$. We employ the Huber loss $\mathcal{L}_\delta$ (Section~\ref{sec:training}) to ensure robustness to deviations from Gaussianity.

To ensure temporal integrity and mitigate look-ahead bias in non-stationary environments, we employ a rolling-window (sliding-window) WFV scheme. Unlike expanding-window approaches where training sets grow indefinitely, we maintain fixed temporal cardinality to prevent model degradation from obsolete structural patterns. Fixed $w_{\text{train}}$ mitigates parameter instability under structural breaks, as shown in \cite{pesaran2007selection}. Let $w_{\text{train}}$ denote the fixed training window size (approximately five years or $\sim$1,250--1,300 trading days), $w_{\text{val}} = 90$ days the validation window for hyperparameter selection, and $s = 21$ days the step size. For fold $k \in \{1, \ldots, K\}$, the chronologically contiguous splits are:
\begin{equation}
\mathcal{D}_{\text{train}}^{(k)} = \{(\mathbf{X}_\tau, \mathbf{Z}_\tau, C_{\tau+1})\}_{\tau=\tau_k}^{\tau_k + w_{\text{train}} - 1}, \quad \mathcal{D}_{\text{val}}^{(k)} = \{(\mathbf{X}_\tau, \mathbf{Z}_\tau, C_{\tau+1})\}_{\tau=\tau_k + w_{\text{train}}}^{\tau_k + w_{\text{train}} + w_{\text{val}} - 1},
\end{equation}
where $\tau_k = (k-1)s$ denotes the starting index. 
All normalization parameters are fitted exclusively on $\mathcal{D}_{\text{train}}^{(k)}$ and applied without refitting to $\mathcal{D}_{\text{val}}^{(k)}$ and test folds, preventing information leakage across temporal boundaries. The window rolls forward by step $s$ until reaching the final hold-out test set comprising the terminal $N_{\text{test}} = 365$ trading days of 2023. This protocol ensures that model selection occurs under realistic temporal constraints where future information is strictly unavailable during training, and that the final evaluation reflects genuine out-of-sample predictive ability on unseen future data. The sliding window mechanism is illustrated in Figure \ref{fig:slinding window}.
\begin{figure}
    \centering
    \includegraphics[width=0.5\linewidth]{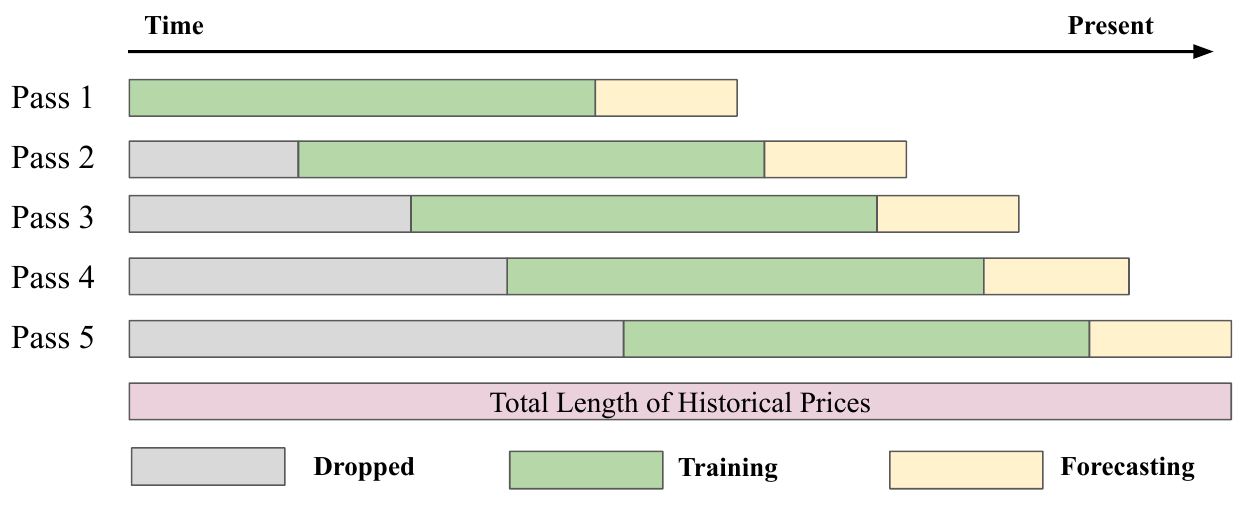}
    \caption{The Sliding-Window WFV protocol. Each fold (Pass) shifts forward by $s=21$ days, ensuring that hyperparameters are tuned on $\mathcal{D}_{\text{val}}$ and evaluated on unseen future data, thus maintaining strict temporal causality.}
    \label{fig:slinding window}
\end{figure}
There is a specific reason of using \textit{"Fixed-Window WFV?"}. Financial time series frequently exhibit structural breaks and regime changes. Expanding-window schemes incorporate observations from outdated regimes, which may introduce bias when underlying data-generating processes shift. By restricting estimation to a fixed-length rolling window, the model conditions on more recent market dynamics, reducing the influence of obsolete structural patterns and improving adaptability under non-stationarity \citep{pesaran2007selection, rossi2013advances}. 

\subsection{Data Acquisition and Preprocessing}\label{sec:preprocessing}
This is the first stage of the WaVeFuse method as mentioned above. In this stage, the daily OHLCV data for four major global equity indices, KOSPI (\^{}KS11), DAX (\^{}GDAXI), NYSE Composite (\^{}NYA), and Russell 2000 (\^{}RUT) are sourced from Yahoo Finance. The sample period spans January 1, 2010, to December 31, 2023. To preserve market microstructure integrity, all series are indexed strictly in trading time. The non-trading days such as weekends, exchange holidays, and market suspensions are excluded entirely rather than imputed via forward-fill, backward-fill, or interpolation. This ensures sequential inputs correspond exclusively to genuine market events, preventing spurious low-volatility plateaus or distorted temporal dependencies. The resulting trading-day counts and final partitioning are detailed in Table~\ref{tab:Indices Used}. Variations in totals reflect differences in national exchange calendars (e.g., KOSPI observes additional Korean holidays).

\begin{table*}[h]
\centering
\caption{Indices under study along with their duration, total trading days, and partitioned data for training and testing}
\label{tab:Indices Used}
\resizebox{0.7\textwidth}{!}{
\begin{tabular}{l l c c c c}
\toprule
Name of Index & Symbol & Duration & Total Trading Days & Training Data & Test Data \\
\midrule
KOSPI & \^{}KS11 & 01-01-2010 -- 31-12-2023 & 3,448 & 3,083 & 365 \\
DAX & \^{}GDAXI & 01-01-2010 -- 31-12-2023 & 3,552 & 3,187 & 365 \\
NYSE Composite & \^{}NYA & 01-01-2010 -- 31-12-2023 & 3,522 & 3,157 & 365 \\
Russell 2000 & \^{}RUT & 01-01-2010 -- 31-12-2023 & 3,522 & 3,157 & 365 \\
\bottomrule
\end{tabular}
}
\end{table*}

\textbf{Wavelet-Based Denoising:}  
Prior to feature extraction, each OHLCV channel undergoes DWT denoising using the Sym-4 wavelet at decomposition level $J=2$ to mitigate high-frequency microstructure noise. The decomposition level $J=2$ reflects a deliberate bias–variance tradeoff specific to daily equity data. At $J=1$, only one level of detail coefficients is thresholded, which removes high-frequency noise but leaves medium-frequency microstructure artifacts intact. These artifacts are particularly damaging for oscillating indicators like RSI and CCI. At $J=3$ and beyond, the soft-thresholding begins to attenuate genuine low-frequency momentum signals, effectively smoothing away the very trends the temporal branch is designed to detect. Level $J=2$ removes the two highest-frequency detail bands while preserving the approximation signal that carries economically meaningful trend and momentum structure, a balance validated empirically for trading-day windows in the 1,250--3,500 day range used here \citep{donoho1995adapting, genccay2002introduction}.

The choice of Sym-4 is motivated by three properties that matter specifically for financial series. Compact support limits boundary distortions at the edges of rolling training windows, which matter at every WFV fold transition. A near-linear phase response guarantees that the features of the denoised signal remain aligned with their true time of occurrence. This is crucial for accurately attributing a momentum signal to the specific day it emerged.
Finally, four vanishing moments mean Sym-4 can represent smooth polynomial trends (linear price drift, quadratic acceleration) exactly without them appearing as noise, preserving the trend component while suppressing high-frequency microstructure fluctuations. Soft-thresholding is applied to detail coefficients, providing stable coefficient shrinkage while preserving dominant structural components.

Additionally, normalization ensures uniformity across heterogeneous feature scales. We employ Min-Max scaling to the $[0,1]$ interval to maintain relative temporal relationships without distortion:
\begin{equation}\label{eq:minmax}
    \bar{N} = \frac{N - N_{\min}}{N_{\max} - N_{\min}}
\end{equation}
where $\bar{N}$ denotes the normalized value, and $N_{\min}, N_{\max}$ are computed exclusively from the respective training fold to prevent information leakage. Parameters are fitted on $\mathcal{D}_{\text{train}}^{(k)}$ and applied consistently to validation and test sequences within fold $k$; no refitting or look-ahead occurs across folds. In the data acquisition and preprocessing stage, after denoising and normalization, the last step is sequence construction. let 
$\tilde{\mathcal{P}}_t \in \mathbb{R}^5$ denote the processed OHLCV vector at time $t$. Input sequences use a fixed causal lookback window of $w=20$ trading days:
\begin{equation}
\mathbf{X}_t = 
[\tilde{\mathcal{P}}_{t-w+1}, \ldots, \tilde{\mathcal{P}}_t]
\in \mathbb{R}^{w \times 5}.
\end{equation}
The sliding window produces overlapping sequences across time. The prediction target is the normalized next-day closing price 
$Y_{t+1} = \tilde{C}_{t+1}$.
Following preprocessing, the denoised and normalized OHLCV sequences provide a stable time-domain representation of market dynamics. However, financial time series exhibit heterogeneous behavior across multiple temporal scales, where short-term fluctuations and long-term structural trends coexist. To explicitly capture such scale-dependent characteristics, WaVeFuse employs a multi-resolution signal decomposition framework. This enables the extraction of complementary temporal and spectral features prior to dual-branch encoding. The following subsection formalizes the wavelet-based decomposition strategy underlying this multi-scale representation.

\subsection{Multi-Resolution Signal Decomposition}\label{sec:decomposition}
Building upon the preprocessed time-domain sequences, we now formalize the multi-resolution framework used in WaVeFuse. The objective of this stage is to decompose financial signals into structured scale-dependent components that capture both coarse trends and fine-grained fluctuations. To achieve this, WaVeFuse employs a dual-wavelet strategy, DWT for denoising and CWT for scale-space feature extraction. We briefly outline the theoretical principles underlying these operators before detailing their implementation.

\subsubsection{Theoretical Background}\label{sec:theory}
\textbf{Wavelet Representation.} 
Wavelet transforms provide multi-resolution analysis by decomposing a signal into localized time-scale components. Given a mother wavelet $\psi(t)$, scaled and translated functions $\psi_{j,k}(t) = 2^{-j/2} \psi(2^{-j} t - k)$ enable localized representation of non-stationary signals across multiple resolutions. This property is particularly suitable for financial time series, which exhibit heterogeneous dynamics across temporal scales.

\textbf{Time-Scale Localization.} 
Unlike fixed-resolution transforms such as the Short-Time Fourier Transform (STFT), the Continuous CWT provides scale-dependent time-frequency localization, offering finer temporal resolution at high frequencies and finer frequency resolution at low frequencies. This adaptive resolution is advantageous for financial signals characterized by both abrupt shocks and long-term trends \citep{crowley2007guide}.

\textbf{Wavelet Denoising.} 
Soft-thresholding of wavelet detail coefficients provides stable coefficient shrinkage, suppressing high-frequency noise while preserving dominant structural components of the signal \citep{donoho1995adapting}. This motivates the use of DWT-based denoising prior to downstream feature extraction.

\subsubsection{Channel-Wise Denoising via DWT}\label{sec:dwt}
Raw equity prices carry two fundamentally different types of variation: genuine market signals(trends, momentum shifts, volatility regimes) and microstructure noise from bid-ask bounce, transient order imbalances, and recording artifacts. Feeding noisy prices directly into a neural encoder forces the model to waste representational capacity distinguishing noise from signal, degrading both branches simultaneously since TIs computed from noisy prices inherit that contamination. The DWT denoising step excises this noise before any feature is extracted, operating as a batch preprocessing step applied exclusively within each training fold's temporal boundaries. Normalization parameters and wavelet thresholds are estimated from training data only and applied without refitting to validation and test folds, ensuring no cross-fold information leakage. Unlike the CWWT branch, which implements strict sample-level causal truncation at time $t$ (Eq \ref{cwt}), the DWT denoising is an offline step that utilizes all samples within the isolated training split. This is a standard practice in financial preprocessing pipelines that precludes look-ahead bias since the out-of-sample test period remains entirely strictly excluded \cite{genccay2002introduction, bhandari2022predicting}.

Formally, let $\mathbf{x}_c = [x_{c,1}, \ldots, x_{c,T}]^\top \in \mathbb{R}^T$ denote the univariate time series corresponding to channel $c \in \{O, H, L, C, V\}$. The DWT decomposes $\mathbf{x}_c$ via Mallat's algorithm into approximation and detail components:
\begin{equation}
\mathbf{x}_c = \mathcal{A}_J[\mathbf{x}_c] + \sum_{j=1}^J \mathcal{D}_j[\mathbf{x}_c],
\end{equation}
where $\mathcal{A}_J$ denotes approximation coefficients at level $J$ and $\mathcal{D}_j$ denotes detail coefficients at scale $j$. For Sym-4 with $J = 2$:
\begin{enumerate}
    \item \textbf{Decomposition:} 
    Apply low-pass $\mathbf{g}$ and high-pass $\mathbf{h}$ filters followed by dyadic downsampling:
    \[
    \mathbf{a}_{j+1} = \downarrow_2 (\mathbf{g} * \mathbf{a}_j), 
    \quad 
    \mathbf{d}_{j+1} = \downarrow_2 (\mathbf{h} * \mathbf{a}_j),
    \] with $\mathbf{a}_0 = \mathbf{x}_c$.
    \item \textbf{Thresholding:} Estimate noise level using the robust median absolute deviation (MAD):
    \[
    \hat{\sigma}_j = 
    \frac{\text{median}(|\mathbf{d}_j - \text{median}(\mathbf{d}_j)|)}{0.6745}.
    \]
    Apply soft-thresholding $\rho_{\text{soft}}(d, \lambda) = \text{sgn}(d)\max(0, |d| - \lambda)$ with universal threshold $\lambda_j = \hat{\sigma}_j \sqrt{2 \log T}$.
    \item \textbf{Reconstruction:} Reconstruct the denoised signal $\hat{\mathbf{x}}_c$ via inverse DWT using the modified coefficients.
\end{enumerate}
Each channel is processed independently, producing the denoised matrix $\hat{\mathbf{X}}_t \in \mathbb{R}^{w \times 5}$. Subsequent encoder modules capture cross-channel dependencies 
after denoising.

\subsection{Technical Indicators and Spectral Encoding via CWWT}\label{sec:cwwt}
Wavelet denoising solves the noise contamination problem at the input stage, producing clean OHLCV sequences and, critically, clean TIs computed from those sequences. But clean inputs alone are insufficient: the seven TIs now available are scalar time series, each collapsing a rich multi-scale market phenomenon into a single number per day. RSI=65 tells you the market is in a moderately overbought state. It does not tell you whether that state has persisted for three days or three weeks, or whether it is part of an accelerating trend or a decaying one. The next step recovers this multi-scale information through a CWWT, which converts each indicator from a scalar series into an instantaneous 32-dimensional frequency fingerprint.
From denoised OHLC data, WaVeFuse computes seven carefully selected TIs optimized for low-lag signal generation. These TIs are listed in the Table \ref{tab:technical_indicators} with their corresponding formulas. 

These indicators are chosen for their ability to provide complementary signals across momentum, volatility, volume confirmation, and overbought/oversold conditions while avoiding high-lag or redundant constructs (e.g., MACD, Bollinger Bands). RSI and Stochastic \%K capture short-term momentum reversals with reduced smoothing. Williams \%R offers a bounded oscillator for extreme market states. ROC provides a pure, unsmoothed measure of price acceleration. CCI identifies cyclical deviations from statistical norms. ATR quantifies direction-agnostic volatility. And OBV integrates volume to confirm price movements. Together, they deliver a multi-faceted yet non-collinear view of market dynamics.

\begin{table}[ht]
\centering
\caption{TI's used in WaVeFuse and their mathematical definitions.}
\label{tab:technical_indicators}
\begin{tabular}{p{4.5cm}p{10.5cm}}
\toprule
\textbf{Indicator} & \textbf{Formula} \\
\midrule
Relative Strength Index (RSI) &
$\displaystyle \text{RSI} = 100 - \frac{100}{1 + \text{RS}}, \quad \text{where } \text{RS} = \frac{\text{Avg gain over } n \text{ days}}{\text{Avg loss over } n \text{ days}}$ \\
& (We use $n = 10$ to reduce lag. gains/losses are smoothed exponentially.) \\
\addlinespace
Stochastic \%K (SO\%K) &
$\displaystyle \text{Stoch\%K} = 100 \times \frac{C - L_n}{H_n - L_n}$ \\
& where $C$ = current close, $H_n$ = highest high over $n=14$ days, $L_n$ = lowest low over $n=14$ days. \\
\addlinespace
Commodity Channel Index (CCI) &
$\displaystyle \text{CCI} = \frac{P - \text{SMA}(P, n)}{0.015 \times \text{MAD}(P, n)}$ \\
& where $P = (H + L + C)/3$, $\text{SMA}$ = simple moving average, $\text{MAD}$ = mean absolute deviation, $n=14$. \\
\addlinespace
On-Balance Volume (OBV) &
$\displaystyle \text{OBV}_t = \text{OBV}_{t-1} +
\begin{cases}
V_t & \text{if } C_t > C_{t-1} \\
0 & \text{if } C_t = C_{t-1} \\
-V_t & \text{if } C_t < C_{t-1}
\end{cases}$ \\
& where $V_t$ = volume at time $t$, $C_t$ = closing price. \\
\addlinespace
Average True Range (ATR) &
$\displaystyle \text{TR}_t = \max(H_t - L_t,\ |H_t - C_{t-1}|,\ |L_t - C_{t-1}|), \quad \text{ATR} = \text{EMA}(\text{TR}, n)$ \\
& with $n=14$, measures volatility independent of direction. \\
\addlinespace
Williams \%R &
$\displaystyle \text{Williams\%R} = -100 \times \frac{H_n - C}{H_n - L_n}$ \\
& where $H_n, L_n$ = highest high and lowest low over $n=14$ days and ranges from $-100$ (oversold) to $0$ (overbought). \\
\addlinespace
Rate of Change (ROC) &
$\displaystyle \text{ROC} = 100 \times \frac{C_t - C_{t-n}}{C_{t-n}}$ \\
& with $n=12$, pure momentum oscillator measuring percentage price change over lookback window. \\
\bottomrule
\end{tabular}
\end{table}

For each indicator $m \in \{1,\ldots,M\}$ ($M=7$) at time $t$, we compute the CWT using the discretized Morlet mother wavelet $\psi(\eta) = \pi^{-1/4} e^{i\omega_0 \eta} e^{-\eta^2/2}$ with central frequency $\omega_0 = 6$. While the Morlet wavelet theoretically exhibits two-sided infinite support, we enforce strict causality by truncating the convolution at $t$ and padding with zeros for $\tau > t$:
\begin{equation} \label{cwt}
W^{(m)}(t, s) = \frac{1}{\sqrt{s}} \sum_{\tau = t - k_s}^{t} \text{TI}^{(m)}(\tau) \, \psi^*\!\left(\frac{\tau - t}{s}\right),
\end{equation}
where $s \in \{1,\ldots,S\}$ with $S=32$ uniformly spaced integer scales ($s = 1, 2, \ldots, 32$), providing coverage from fine-grained daily oscillations ($s=1$) to monthly periodicities ($s=32$), and $k_s = 6s$ defines the effective support radius. A burn-in period of 192 trading days ($6 \times 32$) is excluded from spectral features at the start of each training fold to ensure full support for all scales. Integer spacing is chosen over the logarithmically‑spaced scales typical in geophysical CWT applications because financial oscillations at adjacent short scales (e.g., 3‑day vs. 4‑day momentum cycles) carry distinct and nearly independent information for the Transformer encoder to learn from. Logarithmic spacing, which concentrates resolution at coarser scales, would merge fine‑scale indicator dynamics that differ meaningfully in their association with next‑day returns. With $S=32$ uniformly spaced integer scales, the coverage extends from single‑day oscillations ($s=1$) through monthly periodicities ($s=32$), providing approximately one scale per trading day in a month, a natural resolution for daily forecasting horizons. Due to Gaussian decay, most wavelet energy lies within a finite effective support (approximately $\pm6$ standard deviations), making the causal truncation practically negligible for daily forecasting horizons. At inference time $t$, we extract the instantaneous scale-space representation (modulus):

\begin{equation}
\mathbf{Z}_t^{(m)} = \left[|W^{(m)}(t, s_1)|, \ldots, |W^{(m)}(t, s_S)|\right]^\top \in \mathbb{R}^S,
\end{equation}
yielding the CWWT matrix $\mathbf{Z}_t = [\mathbf{Z}_t^{(1)}, \ldots, \mathbf{Z}_t^{(M)}] \in \mathbb{R}^{32 \times 7}$
(with 32 uniformly spaced scales). Algorithm \ref{alg:wavfuse_preprocessing} briefly illustrates the First three stages of WaVeFuse which are data acquisition and preprocessing, Calculating TIs, and CWWT process. 

\begin{algorithm}[htbp]
\small
\caption{WaVeFuse: Part-1}
\label{alg:wavfuse_preprocessing}
\begin{algorithmic}[1]
\REQUIRE Raw OHLCV series $\{(O_t,H_t,L_t,C_t,V_t)\}_{t=1}^T$;
         $w{=}20$, $J{=}2$, $S{=}32$, $M{=}7$,
         $w_{\text{tr}}{=}1260$, $w_{\text{val}}{=}90$, $s_{\text{step}}{=}21$
\ENSURE  Fold datasets $\{\mathcal{D}_{\text{tr}}^{(k)},\mathcal{D}_{\text{val}}^{(k)},\Theta^{(k)}\}_{k=0}^{K-1}$

\STATE Reserve last 365 days as held-out test; remainder is $\mathcal{I}_{\text{tv}}$.
\STATE $\text{start}\leftarrow 0$, $k\leftarrow 0$
\WHILE{$\text{start}+w_{\text{tr}}+w_{\text{val}}\leq|\mathcal{I}_{\text{tv}}|$}
    \STATE Define $\mathcal{I}_{\text{tr}}^{(k)}$, $\mathcal{I}_{\text{val}}^{(k)}$ by sliding window.

    \STATE \textbf{// Wavelet Denoising (Sym-4 soft thresholding, $J{=}2$)}
    \FOR{each split $\in\{\mathcal{I}_{\text{tr}}^{(k)},\mathcal{I}_{\text{val}}^{(k)}\}$, each channel $c\in\{O,H,L,C,V\}$}
        \STATE Decompose: $(\mathbf{a}^{(J)},\{\mathbf{d}^{(j)}\})\leftarrow\mathrm{DWT}_{\text{Sym4}}(\mathbf{x}_c,J)$
        \STATE Threshold per level: $\hat{\sigma}_j\leftarrow\mathrm{MAD}(\mathbf{d}^{(j)})/0.6745$,\;
               $\lambda_j\leftarrow\hat{\sigma}_j\sqrt{2\log|\text{split}|}$,\;
               $\tilde{\mathbf{d}}^{(j)}\leftarrow\mathrm{soft}(\mathbf{d}^{(j)},\lambda_j)$
        \STATE Reconstruct: $\tilde{\mathbf{x}}_c\leftarrow\mathrm{IDWT}(\mathbf{a}^{(J)},\tilde{\mathbf{d}}^{(J)},\dots,\tilde{\mathbf{d}}^{(1)})$
    \ENDFOR

    \STATE \textbf{// TIs \& MinMax Scaling}
    \STATE Compute $\mathbf{TI}_t\in\mathbb{R}^7$ (RSI-10, Stoch-\%K, CCI-14, OBV, ATR-14,
           Williams-\%R, ROC-12) over $\mathcal{I}_{\text{tr}}^{(k)}\cup\mathcal{I}_{\text{val}}^{(k)}$
    \STATE Fit scalers $\phi_{\text{ohlc}},\phi_{\text{ti}}$ on $\mathcal{I}_{\text{tr}}^{(k)}$ only; apply to both splits.

    \STATE \textbf{// Causal CWT Spectrogram (Morlet, $S{=}32$ integer scales)}
    \FOR{each timestep $t$ (where $t\geq w$) in either split}
        \STATE $W^{(m)}(t,s)\leftarrow\frac{1}{\sqrt{s}}\sum_{\tau=t-6s}^{t}
               \mathrm{TI}_{\tau,m}^{\text{sc}}\cdot\psi_{\text{Morlet}}^*\!\left(\tfrac{t-\tau}{s}\right)$,\quad
               $s\in\{1,\dots,S\}$, $m\in\{1,\dots,M\}$
        \STATE $\mathbf{X}_t^{\text{spec}}\leftarrow\bigl[|W^{(m)}(t,s)|\bigr]_{s,m}\in\mathbb{R}^{S\times M}$;\quad
               $\mathbf{X}_t^{\text{temp}}\leftarrow[\tilde{\mathbf{X}}_\tau^{\text{sc}}]_{\tau=t-w+1}^{t}\in\mathbb{R}^{w\times 5}$;\quad
               $y_t\leftarrow\tilde{C}_{t+1}^{\text{sc}}$
    \ENDFOR

    \STATE $\Theta^{(k)}\leftarrow\mathrm{WaVeFuseTrain}(\mathcal{D}_{\text{tr}}^{(k)},\mathcal{D}_{\text{val}}^{(k)})$
           \COMMENT{Algorithm~\ref{alg:wavfuse_neural}}
    \STATE $\text{start}\leftarrow\text{start}+s_{\text{step}}$;\; $k\leftarrow k+1$
\ENDWHILE
\RETURN $\{\mathcal{D}_{\text{tr}}^{(k)},\mathcal{D}_{\text{val}}^{(k)},\Theta^{(k)}\}_{k=0}^{K-1}$
\end{algorithmic}
\end{algorithm}

While raw OHLCV sequences require an explicit temporal window ($w=20$) to model sequential dependencies, TIs summarize recent historical price behavior through predefined lookback structures. The spectral branch, therefore, operates on derived, higher-level descriptors of market dynamics rather than raw price gradients. Applying CWT to these indicators produces a scale-space representation that captures frequency-dependent behavior at the current time step. By treating scales ($S=32$) as the sequence dimension, Transformer models capture inter-scale interactions (e.g., between short-term volatility spikes and longer-term oscillatory structures). This design yields complementary feature representations. The temporal branch captures sequential dynamics in price-volume space, while the spectral branch captures instantaneous multi-scale structure in indicator space.

\subsection{Dual-Branch Encoder Architecture}\label{sec:encoders}
After calculating the TIs, the next stage of WaVeFuse is dual branch encoder architecture, where, first branch encodes temporal dynamics via CNN-BiLSTM and the second branch encodes spectral dependencies via Transformer. The two encoding branches now produce complementary but structurally incompatible representations. The CNN-BiLSTM temporal branch outputs a 128-dimensional vector encoding 20 days of sequential OHLCV dynamics, while the Transformer spectral branch outputs a 7-dimensional vector encoding instantaneous scale-space energy distributions across 32 frequency bins. Neither representation can be simply concatenated and passed to a decoder, as their different dimensionalities and semantic contents require an intelligent combination mechanism that can recognize which branch is more informative at any given market moment. This is the role of VAF. In the following subsections, we describe each branch in detail.

\subsubsection{Temporal Dynamics via CNN-BiLSTM} \label{subsec:cnn-bilstm}
This is the first branch of the dual branch encoder architecture. In this branch, we apply one dimensional convolutions to extract local spatial patterns from denoised $\hat{\mathbf{X}}_t$ as follows:
\begin{equation}
\mathbf{H}_{\text{cnn}} = \text{ReLU}(\text{Conv1D}(\hat{\mathbf{X}}_t; \mathbf{W}_{\text{cnn}})) \in \mathbb{R}^{w' \times d_{\text{cnn}}},
\end{equation}
where $d_{\text{cnn}} = 64$ filters, kernel size $k = 3$, and $w' = w - k + 1$. Process $\mathbf{H}_{\text{cnn}}$ through two stacked BiLSTM layers: the first layer ($d_{\text{lst1}} = 32$) returns the full sequence to preserve intermediate temporal states, and the second layer ($d_{\text{lst2}} = 64$) returns only the final hidden state as the temporal bottleneck encoding:
\begin{equation}
\overrightarrow{\mathbf{h}}_t = \overrightarrow{\text{LSTM}}(\mathbf{H}_{\text{cnn}}, \overrightarrow{\mathbf{h}}_{t-1}), \quad \overleftarrow{\mathbf{h}}_t = \overleftarrow{\text{LSTM}}(\mathbf{H}_{\text{cnn}}, \overleftarrow{\mathbf{h}}_{t+1}).
\end{equation}
The BiLSTM operates exclusively on the causal lookback window $\hat{\mathbf{X}}_t \in \mathbb{R}^{20 \times 5}$, which contains only observations up to time $t$. No information beyond $t$ is accessible. Bidirectionality provides two advantages within this window: (i) conditioning on both earlier and later positions inside the window, and (ii) improved gradient propagation during training. The final hidden state concatenation yields:
\begin{equation}
\mathbf{h}_{\text{temp}} = [\overrightarrow{\mathbf{h}}_{w'}; \overleftarrow{\mathbf{h}}_{1}] \in \mathbb{R}^{128}.
\end{equation}
and 

\begin{equation}
\mathbf{u}_{\text{temp}} = \text{ReLU}(\mathbf{W}_{\text{proj}}^{(1)} \mathbf{h}_{\text{temp}} + \mathbf{b}_{\text{proj}}^{(1)}) \in \mathbb{R}^{d_{\text{fus}}},
\end{equation}
with $d_{\text{fus}} = 32$.

\subsubsection{Global Spectral Encoder (Transformer)} \label{subsec:transformer}
The second branch of this architecture is the Transformer. The Transformer operates on $\mathbf{Z}_t \in \mathbb{R}^{32 \times 7}$, treating wavelet scales as $S=32$ sequence tokens and indicators as feature dimensions ($M=7$). Unlike language Transformers, where token position is not structurally encoded in the raw embedding representation and must therefore be injected explicitly via positional encoding, wavelet scales carry intrinsic semantic ordering: scale $s=1$ always represents higher-frequency oscillations than $s=2$, which represents higher frequencies than $s=3$, and so on. This monotonic frequency ordering is a fixed mathematical property of the Morlet wavelet, not a learned representation. Injecting a learned positional embedding on top of this pre-existing order risks the network learning to partly undo the injected encoding, adding parameters without benefit. Empirical validation confirmed that adding sinusoidal or learned positional encodings changed MAE by less than $0.3\%$ across all four indices, confirming that the intrinsic scale order is sufficient.

Compute multi-head scaled dot-product attention across two stacked encoder layers. Layer 1 uses 4 heads with $d_k = d_v = 16$. Layer 2 uses 4 heads with $d_k = d_v = 32$, progressively expanding the representational capacity across scales:
\begin{equation}
\text{Attention}(\mathbf{Q}, \mathbf{K}, \mathbf{V}) = \text{softmax}\left(\frac{\mathbf{Q}\mathbf{K}^\top}{\sqrt{d_k}}\right)\mathbf{V},
\end{equation}
where $\mathbf{Q}, \mathbf{K}, \mathbf{V}$ derive from linear projections of $\mathbf{Z}_t$. To prevent overfitting within the spectral domain, dropout ($p=0.2$) is applied to the output of each multi-head attention sub-layer prior to the residual addition. Rather than uniform averaging, we apply a learned attention pooling over the scale dimension. To reduce parameter count, the standard position-wise feed-forward sublayer is omitted. Thus,  each encoder layer consists solely of multi-head self-attention with residual connection and layer normalization.  A scalar attention score is computed for each scale token $s$ via a shared dense layer:
\begin{equation}
    e_s = z_s^{(\text{penult})} \cdot w_{\text{pool}} + b_{\text{pool}}, \quad w_{\text{pool}} \in \mathbb{R}^M
    \label{eq:attention_score}
\end{equation}
Scale weights are normalized by softmax over the $S$ tokens:
\begin{equation}
    w_s = \frac{\exp(e_s)}{\sum_{j=1}^{S} \exp(e_j)}
    \label{eq:softmax_weights}
\end{equation}

The pooled spectral representation is obtained via attention-weighted summation over scales:
\begin{equation}
    \mathbf{h}_{\text{spec}}^{(\text{raw})} = \sum_{s=1}^{S} w_s \cdot \mathbf{z}_s^{(\text{penult})} \in \mathbb{R}^{M} = \mathbb{R}^{7}
    \label{eq:pooled_spectral}
\end{equation}
This raw spectral vector is then projected to the common fusion space:
\begin{equation}
    \mathbf{u}_{\text{spec}} = \text{ReLU}\big(\mathbf{W}_{\text{proj}}^{(2)} \mathbf{h}_{\text{spec}}^{(\text{raw})} + \mathbf{b}_{\text{proj}}^{(2)}\big) \in \mathbb{R}^{d_{\text{fus}}} = \mathbb{R}^{32}
    \label{eq:spectral_projection}
\end{equation}
Thus, the final spectral representation entering the fusion module has dimension \(d_{\text{fus}}=32\), matching the temporal branch. This formulation preserves sharp, scale-localized wavelet signals that uniform averaging would suppress, allowing the network to selectively up weight scale tokens carrying the highest predictive energy.

\subsection{Vertical Attention Fusion (VAF) Mechanism}\label{sec:vaf}
The VAF implements adaptive domain selection via a learned \textit{2-token attention mechanism}, formulated as a differentiable mixture-of-experts that performs online variance minimization by allocating higher weight to the branch with lower predictive uncertainty for the specific market regime. The VAF mechanism is also shown in Figure \ref{fig:vaf}.

\begin{figure}
    \centering
    \includegraphics[width=\linewidth]{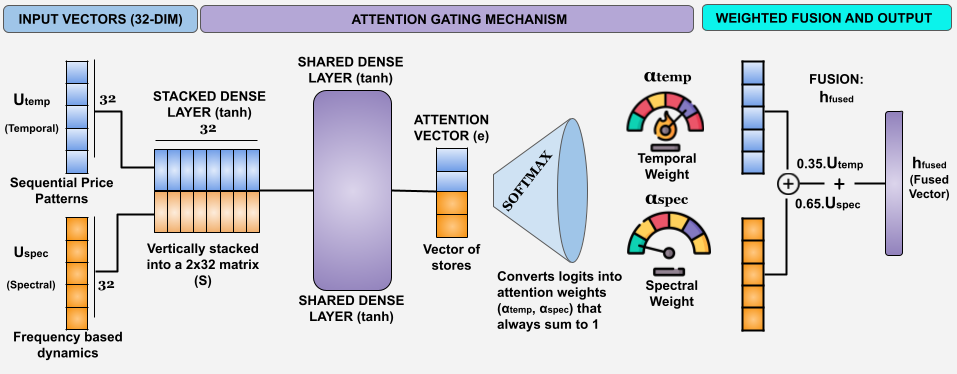}
    \caption{VAF for Regime-Adaptive Integration}
    \label{fig:vaf}
\end{figure}

\textbf{Vertical Stacking:} Construct meta-sequence $\mathbf{S} \in \mathbb{R}^{2 \times d_{\text{fus}}}$ :
\begin{equation}
\mathbf{S} = \begin{bmatrix}
\mathbf{u}_{\text{temp}}^\top \\
\mathbf{u}_{\text{spec}}^\top
\end{bmatrix}
\end{equation}

\textbf{Gating via 2-Token Self-Attention:} Compute attention logits via shared dense layer with $\tanh$ nonlinearity:
\begin{equation}
\mathbf{e} = \tanh(\mathbf{S}\mathbf{W}_{\text{gate}} + \mathbf{b}_{\text{gate}}) \in \mathbb{R}^{2 \times d_{\text{gate}}},
\end{equation}
where $\mathbf{W}_{\text{gate}} \in \mathbb{R}^{d_{\text{fus}} \times 1}$ ($d_{\text{gate}} = 1$; a single scalar attention logit per branch token), and the resulting 2-element logit vector is passed through softmax. Reduce to scalar logits via $\mathbf{w}_{\text{att}} \in \mathbb{R}^{d_{\text{gate}}}$ and apply softmax:
\begin{equation}
\boldsymbol{\alpha} = \text{softmax}\left((\mathbf{e}\mathbf{w}_{\text{att}})^\top\right) = [\alpha_{\text{temp}}, \alpha_{\text{spec}}]^\top \in \mathbb{R}^2,
\end{equation}
A scalar gate per branch is theoretically sufficient for convex combination weighting. Richer gating would add parameters without expanding the expressible fusion space, which remains the unit simplex $\alpha_{\text{temp}} + \alpha_{\text{spec}} = 1$ regardless of gate dimensionality.
The fused representation is the convex combination:
\begin{equation}
\mathbf{h}_{\text{fused}} = \alpha_{\text{temp}} \cdot \mathbf{u}_{\text{temp}} + \alpha_{\text{spec}} \cdot \mathbf{u}_{\text{spec}} \in \mathbb{R}^{d_{\text{fus}}}.
\end{equation}

The VAF minimizes the Bayes risk of the fused estimator. Let $\mathcal{E}_{\text{temp}}$ and $\mathcal{E}_{\text{spec}}$ denote branch error random variables. The gating network learns the precision-weighted combination minimizing $\mathbb{E}[(\alpha \mathcal{E}_{\text{temp}} + (1-\alpha)\mathcal{E}_{\text{spec}})^2]$, effectively performing \textit{online variance minimization}. Under squared loss, the optimal convex combination of unbiased estimators minimizes variance when the weights are proportional to the inverse conditional error variances. The VAF gating network learns this mapping implicitly in a data-driven manner.

The latent variable $z \in \{\text{temporal}, \text{spectral}\}$ indicating domain dominance. The gating computes posterior $p(z \mid \mathbf{X}_t)$ via softmax, yielding mixture prediction $p(\hat{y}) = \sum_z p(z \mid \mathbf{X}_t) \cdot p(\hat{y} \mid \mathbf{X}_t, z)$. This connects to Hamilton regime-switching models \citep{hamilton1989new,ang2012regime} and their applications in regime‑dependent asset allocation \citep{guidolin2007asset}, but replaces hard thresholds with learned, differentiable transitions enabling end‑to‑end gradient descent.

\begin{itemize}
    \item When $\alpha_{\text{spec}} \approx 1$: The model detects high-frequency spectral anomalies (volatility shocks) not captured by temporal momentum; the spectral branch dominates.
    \item When $\alpha_{\text{temp}} \approx 1$: The market follows persistent temporal trends; the LSTM branch dominates.
\end{itemize}
Alternatively, Algorithm \ref{alg:wavfuse_neural} summarized the last stages of WaVeFuse, which are Dual-Branch Encoder Architecture, VAF Mechanism, and Decoder and Training Objective.

\begin{theorem}\label{theoram:1}
Under unbiased branch estimators with Gaussian errors, VAF achieves optimal precision-weighted fusion, minimizing fused variance.
\end{theorem}
  
Assume branch errors $\mathcal{E}_{\text{temp}} \sim \mathcal{N}(0, \sigma_{\text{temp}}^2)$ and $\mathcal{E}_{\text{spec}} \sim \mathcal{N}(0, \sigma_{\text{spec}}^2)$. The minimum-variance unbiased estimator is $\hat{y} = \alpha \hat{y}_{\text{temp}} + (1-\alpha) \hat{y}_{\text{spec}}$ with $\alpha = \sigma_{\text{spec}}^2 / (\sigma_{\text{temp}}^2 + \sigma_{\text{spec}}^2)$. VAF's gating approximates this via learned attention, converging to inverse-variance weighting under gradient descent, analogous to Kalman filtering for state estimation.

\begin{remark}[1]
This remark provides an idealized analytical motivation under Gaussian error assumptions. It is not a formal convergence guarantee for the trained softmax gate, which operates in a non-Gaussian, non-convex optimization landscape. In practice, financial return distributions exhibit excess kurtosis ($\kappa \approx 5$-10), violating Gaussianity. The Huber training objective (Section~\ref{sec:training}) corrects for this departure by bounding the influence of tail observations, ensuring that the learned gating weights $\boldsymbol{\alpha}$ converge to a robust approximation of the inverse-variance weighting prescribed by Theorem 1 even under non-Gaussian residuals. The theoretical bound thus represents an attainable optimum in the Gaussian limit and a principled target in the heavy-tailed setting.
\end{remark}

\begin{proposition}[Branch Complementarity]
The temporal (CNN-BiLSTM) and spectral (CWWT+Transformer) branches occupy empirically complementary representational subspaces. Consequently, the predictive signal captured by one branch is not adequately captured by the other.
Removing either encoding branch (the CNN-BiLSTM temporal branch or the CWWT-Transformer spectral branch) results in a degradation of forecasting performance exceeding $10\%$ (see ablation study, Section~\ref{sec:ablation}), confirming their non-redundant contributions prior to the VAF gating stage.
\end{proposition}

\begin{remark}[2] [Robustness to Non-Gaussian Errors]
Although Theorem \ref{theoram:1} is derived under the assumption of Gaussian branch errors for analytical tractability, the VAF mechanism remains robust to heavy-tailed and non-Gaussian distributions common in financial returns. The Huber loss employed during training (Section~\ref{sec:training}) bounds the influence of large residuals, while end-to-end gradient descent allows the gating network to learn precision-weighted fusion directly from data. Under mild regularity conditions, the learned attention weights $\boldsymbol{\alpha}$ asymptotically approximate the minimum-variance combination even when errors deviate from normality.
\end{remark}

\begin{algorithm}[htbp]
\small
\caption{WaVeFuse: Part-2}
\label{alg:wavfuse_neural}
\begin{algorithmic}[1]
\REQUIRE $\mathcal{D}_{\text{tr}},\mathcal{D}_{\text{val}}$ with tuples $(\mathbf{X}^{\text{temp}},\mathbf{X}^{\text{spec}},y)$;
         $d_{\text{cnn}}{=}64$, $d_{\text{lst1}}{=}32$, $d_{\text{lst2}}{=}64$, $d_{\text{fus}}{=}32$,
         $\delta{=}0.1$, $\eta{=}10^{-3}$, $\eta^*{=}5{\times}10^{-4}$,
         $B{=}32$, $E{=}50$, $P{=}10$
\ENSURE  Trained parameters $\Theta^*$, predictions $\{\hat{C}_{t+1}\}$, fusion weights $\{\boldsymbol{\alpha}_t\}$

\STATE \textbf{Function} $\mathrm{WaVeFuseTrain}(\mathcal{D}_{\text{tr}},\mathcal{D}_{\text{val}})$:
\STATE Initialize $\Theta$ via Glorot uniform.
\FOR{epoch $e=1$ \TO $E$}
    \FOR{minibatch $\mathcal{B}\subset\mathcal{D}_{\text{tr}}$}

        \STATE \textbf{// Temporal Branch}
        \STATE $\mathbf{h}_{\text{temp}}\leftarrow
               \mathrm{Proj}_{d_{\text{fus}}}\!\Bigl(
               \mathrm{BiLSTM}_{d_{\text{lst2}}}\!\bigl(
               \mathrm{BiLSTM}_{d_{\text{lst1}}}\!\bigl(
               \mathrm{Conv1D}_{k=3,d_{\text{cnn}}}(\mathbf{X}^{\text{temp}})\bigr)\bigr)\Bigr)$

        \STATE \textbf{// Spectral Branch (Transformer + Attention Pooling)}
        \STATE $\mathbf{H}^{(1)}\leftarrow\mathbf{X}^{\text{spec}}+\mathrm{MHA}_{h=4,d_k=16}(\mathrm{LN}(\mathbf{X}^{\text{spec}}))$
        \STATE $\mathbf{H}^{(2)}\leftarrow\mathbf{H}^{(1)}+\mathrm{MHA}_{h=4,d_k=32}(\mathrm{LN}(\mathbf{H}^{(1)}))$
        \STATE $\mathbf{h}_{\text{spec}}\leftarrow\mathrm{Proj}_{d_{\text{fus}}}\!\Bigl(\textstyle\sum_s\mathrm{softmax}(\mathbf{H}^{(2)}\mathbf{w}_{\text{pool}})_s\cdot\mathbf{H}^{(2)}_{s,:}\Bigr)$
               \COMMENT{learned scale attention}

        \STATE \textbf{// Vertical Attention Fusion (VAF)}
        \STATE $\boldsymbol{\alpha}\leftarrow\mathrm{softmax}\!\bigl([\mathbf{h}_{\text{temp}};\mathbf{h}_{\text{spec}}]\,\mathbf{W}_{\text{gate}}\bigr)\in\mathbb{R}^2$
        \STATE $\mathbf{h}_{\text{fused}}\leftarrow\alpha_1\mathbf{h}_{\text{temp}}+\alpha_2\mathbf{h}_{\text{spec}}$

        \STATE \textbf{// Decoder \& Output}
        \STATE $\hat{y}\leftarrow\mathrm{Linear}\!\bigl(\mathrm{BiLSTM}_{64}(\mathrm{Repeat}_5(\mathbf{h}_{\text{fused}}))\bigr)$

        \STATE $\mathcal{L}\leftarrow\frac{1}{|\mathcal{B}|}\sum\mathcal{L}_\delta(y,\hat{y})$;\quad
               $\Theta\leftarrow\mathrm{Adam}(\Theta,\nabla_\Theta\mathcal{L};\eta)$
    \ENDFOR
    \STATE Early-stop with patience $P$ on validation loss (WFV folds only); restore best $\Theta$.
\ENDFOR
\RETURN $\Theta$

\STATE \textbf{Final Training on $\mathcal{I}_{\text{tv}}$:}
\STATE Re-preprocess full $\mathcal{I}_{\text{tv}}$ (Steps 2.1–2.3, scalers fit on all of $\mathcal{I}_{\text{tv}}$).
\STATE $\Theta^*\leftarrow\mathrm{WaVeFuseTrain}(\mathcal{D}_{\text{tv}},\emptyset)$
       with $\mathrm{Adam}_{\mathrm{AMSGrad}}(\eta^*)$, $E{=}50$, no early stopping.

\STATE \textbf{Evaluation on $\mathcal{I}_{\text{test}}$:}
\STATE Apply $f_{\Theta^*}$ to each test window; invert scaling to obtain $\hat{C}_{t+1}$.
\STATE Compute RMSE, MAE, MAPE, $R^2$, DA; DM test vs.\ random walk.
\RETURN $\Theta^*$, $\{\hat{C}_{t+1}\}$, $\{\boldsymbol{\alpha}_t\}$, metrics
\end{algorithmic}
\end{algorithm}

\subsection{Decoder and Training Objective}\label{sec:training}
The VAF mechanism distills the complementary knowledge of both branches into a single 32-dimensional vector $\mathbf{h}_{\text{fused}}$, weighted by the model's learned assessment of which branch better characterizes the current market regime. This fused vector now contains more predictive information than either branch alone. It encodes sequential price dynamics when trending markets make history reliable, and spectral anomalies when indicator patterns signal imminent regime transitions. The decoder's task is to translate this compact, information-rich representation into a concrete next-day closing price forecast.
The fused representation $\mathbf{h}_{\text{fused}}$ serves as the latent bottleneck, fed to a CNN-BiLSTM \textit{decoder}:
\begin{enumerate}
    \item \textbf{Sequence Expansion \& BiLSTM Decoding:}
    The fused bottleneck is expanded via $\text{RepeatVector}(5)$ 
    to create a short sequence of five identical tokens. Even though the inputs are identical, the BiLSTM's hidden state evolves through unrollings. This process effectively implements a depth-5 recurrent projection that applies successive non-linear transformations, similar to stacking dense layers with shared weights. This design was retained after preliminary comparisons with single-step MLP decoders showed marginally better validation loss, consistent with sequence-to-point architectures \citep{ji2024galformer}. The expanded sequence is processed 
    by a single-layer BiLSTM with dropout ($p = 0.4$):
    \begin{equation}
        \mathbf{H}_{\text{dec}} = 
        \text{BiLSTM}_{\text{dec}}(\mathbf{h}_{\text{fused}}) 
        \in \mathbb{R}^{64}.
        \label{eq:decoder_bilstm}
    \end{equation}

    \item \textbf{Output Projection:}
    A final dense layer maps to the normalized next-step closing 
    price:
    \begin{equation}
        \hat{y} = \mathbf{W}_{\text{out}} \cdot 
        \text{Dropout}\!\left(
        \text{ReLU}(\mathbf{W}_{\text{dec}}\,
        \mathbf{H}_{\text{dec}} + \mathbf{b}_{\text{dec}})
        \right) + b_{\text{out}},
        \label{eq:decoder_output}
    \end{equation}
    where $\hat{y} \in [0,1]$ represents the normalized 
    next-step closing price.
\end{enumerate}

\textbf{Huber Loss with Adaptive Thresholding.} Given heavy-tailed financial returns, we minimize the Huber loss:
\begin{equation}
\mathcal{L}_\delta(y, \hat{y}) = \begin{cases}
\frac{1}{2}(y - \hat{y})^2 & \text{if } |y - \hat{y}| \leq \delta, \\
\delta(|y - \hat{y}| - \frac{1}{2}\delta) & \text{otherwise},
\end{cases}
\end{equation}

\begin{proposition}[Huber Loss Reduces Tail-Risk Bias]
For heavy-tailed error distributions with excess kurtosis $\kappa > 3$, the Huber estimator achieves lower asymptotic bias under model misspecification than least squares.
\end{proposition}
\begin{proof}
Consider the influence function $\psi(\epsilon) = \partial \mathcal{L}_\delta / \partial \epsilon$. For MSE, $\psi_{\text{MSE}}(\epsilon) = \epsilon$ (unbounded). For Huber loss, $\psi_{\text{Huber}}(\epsilon) = \min(\delta, \max(-\delta, \epsilon))$ (bounded). Under a contaminated Gaussian distribution $F = (1-\eta)\Phi + \eta G$ with $G$ heavy-tailed, the asymptotic bias $\int \psi \, dF$ satisfies $|\psi_{\text{Huber}}| \leq \delta$ while $|\psi_{\text{MSE}}| \to \infty$ as $|G|$ grows. The relative efficiency of Huber to MSE under heavy tails is $\text{Eff} = \sigma^2_{\text{MSE}} / \sigma^2_{\text{Huber}} > 1$ for $\kappa > 3$, yielding materially lower asymptotic bias under heavy-tailed contamination, as established in the robust statistics literature \citep{huber1996robust}.
\end{proof}

We set $\delta = 0.1$ based on the empirical distribution of normalized residuals in the validation folds. The residual mean is approximately $0.012$ and the $99.9$th percentile is $0.14$, placing the quadratic-linear boundary at roughly $8$ standard deviations. This conservative threshold ensures that Huber loss behaves as MSE for the vast majority of observations while bounding the influence of rare extreme residuals. This value is consistent across both the WFV training folds and the final deployment model. Furthermore, for training protocol, we optimize $\theta$ via Adam (learning rate $\eta = 10^{-3}$, $\beta_{1} = 0.9$, $\beta_{2} = 0.999$) with early stopping (patience $= 10$) on validation loss during WFV folds. The final deployment model is trained on all pre-2024 data without validation monitoring to maximize data utilization. Hyperparameters (e.g., w=20, S=32) were tuned via WFV. Variations ($\pm10\%$) degrade MAE by $<5\%$, confirming robustness.

\section{Result Analysis with Experimental Design and Setup} \label{sec:result_analysis}
This section evaluates WaVeFuse under rigorous experimental conditions to assess predictive performance, statistical robustness, practical trading applicability, and computational efficiency. All experiments adhere to the WFV protocol outlined in Section~\ref{sec:problem}, ensuring temporal integrity and mitigation of look-ahead bias. The evaluation uses four major global equity indices, KOSPI (\^{}KS11), DAX (\^{}GDAXI), NYSE Composite (\^{}NYA), and Russell 2000 (\^{}RUT), spanning January 1, 2010, to December 31, 2023, with the final 365 trading days of 2023 reserved as a strict out-of-sample test set. This duration is specifically chosen as this incorporates bull, bear, high volatility and crash scenarios.
We begin by detailing the experimental setup, including hyperparameters, preprocessing, and evaluation metrics. Subsequent subsections provide quantitative comparisons against baselines, SOTA, statistical significance tests, simulated trading strategy performance, computational analysis and ablation studies. Results demonstrate WaVeFuse's superior generalization across diverse market regimes, with consistent improvements in error metrics, explanatory power, and directional forecasting.

\subsection{Experimental Design and Setup} \label{subsec:experiment_setup}
We evaluate WaVeFuse using the metrics defined in Table~\ref{tab:metrics}. Primary error metrics are RMSE and MAE, while directional accuracy assesses trading-relevant signal quality. Statistical significance is confirmed via DM \citep{diebold2002comparing} and paired t-tests against a naive benchmark.
\begin{table}[ht]
\centering
\caption{Evaluation metrics and loss function used in WaVeFuse.}
\label{tab:metrics}
\begin{tabular}{p{3cm}p{6cm}p{7cm}}
\toprule
\textbf{Metric} & \textbf{Mathematical Expression} & \textbf{Purpose and Meaning} \\
\midrule

Huber Loss & 
$\mathcal{L}_\delta(p, \hat{p})$ = 
$\begin{cases}
\frac{1}{2}(p - \hat{p})^2 & \text{if } |p - \hat{p}| \leq \delta \\
\delta |p - \hat{p}| - \frac{1}{2}\delta^2 & \text{otherwise}
\end{cases}$ &
Robust loss used during training. It combines MSE for small errors and MAE for large errors. \\
\addlinespace

RMSE & 
$\sqrt{\frac{1}{N} \sum_{t=1}^{N} (\hat{p}_t - p_t)^2}$ &
Root Mean Squared Error on inverse-scaled closing prices. Penalizes large errors. \\
\addlinespace

MAE & 
$\frac{1}{N} \sum_{t=1}^{N} |\hat{p}_t - p_t|$ &
Mean Absolute Error on inverse-scaled prices. It measures average prediction deviation. \\
\addlinespace

MAPE &
$\frac{100}{N}\sum_{t=1}^{N}\left|\frac{p_t-\hat{p}_t}{p_t}\right|$ &
Mean Absolute Percentage Error. Scale-independent measure of prediction accuracy. \\
\addlinespace

$R^2$ & 
$1 - \frac{\sum_{t=1}^{N} (p_t - \hat{p}_t)^2}{\sum_{t=1}^{N} (p_t - \bar{p})^2}$ &
Coefficient of determination, fraction of price variance explained by the model. \\
\addlinespace

Directional Accuracy (DA) & 
$\frac{1}{T-1} \sum_{t=1}^{T-1} \mathbb{I}\big[(p_{t+1} - p_t)(\hat{p}_{t+1} - \hat{p}_t) > 0\big] \times 100\%$ &
Percentage of correctly predicted price movement directions (up/down). \\
\addlinespace

Diebold–Mariano (DM) Test & 
$d_t = (p_t - \hat{p}_t^{\text{WaVeFuse}})^2 - (p_t - \hat{p}_t^{\text{naive}})^2$, \hspace{5mm} \quad $DM = \dfrac{\bar{d}}{\sqrt{\widehat{\text{Var}}(\bar{d})}}$ &
Tests if WaVeFuse’s MSE is significantly lower than a naive (persistence) benchmark. \\
\addlinespace

Paired $t$-test & 
$t = \dfrac{\bar{d}}{s_d / \sqrt{N}}$, \quad $d_t = |p_t - \hat{p}_t^{\text{WaVeFuse}}| - |p_t - \hat{p}_t^{\text{naive}}|$ &
Tests if WaVeFuse’s MAE is significantly lower than the naive model’s MAE. \\

\bottomrule
\end{tabular}
\begin{minipage}{\textwidth}
\vspace{0.5em}
\footnotesize
\textbf{Notation:} 
$p_t$ = actual closing price at time $t$; 
$\hat{p}_t$ = predicted closing price; 
$\bar{p}$ = mean of $\{p_t\}_{t=1}^N$; 
$\mathbb{I}[\cdot]$ = indicator function (1 if true, 0 otherwise); 
$N$ = number of test samples; 
$\hat{p}_t^{\text{naive}} = p_{t-1}$ (persistence forecast).
\end{minipage}
\end{table}

Table~\ref{tab:wavfuse_params} details the initialization, preprocessing, and training settings for WaVeFuse. 
The model processes dual input streams over a 20-day lookback window: (i) raw OHLCV price-volume data and (ii) seven low-lag TIs shown in Table \ref{tab:technical_indicators}, preserving market microstructure while mitigating multicollinearity through feature diversity rather than reduction. All inputs undergo wavelet-based denoising (Sym-4, level 2) followed by MinMax scaling, and the TIs are further transformed via CWWT (Morlet, scales 1--32) to capture multi-scale dynamics. For clarity, CWWT refers to the application of a discretized CWT independently to each TI channel, allowing per-feature multi-resolution analysis without cross channel interactions, distinct from a full multivariate CWT but leveraging the same underlying principles for time-frequency decomposition.

The hybrid architecture fuses a CNN-BiLSTM branch (for local sequential patterns) with a dual-layer Transformer (for global dependencies) using a learnable VAF mechanism. Regularization is enforced via dropout (rate = 0.4) applied at multiple layers, with no explicit $\ell_2$ penalty. The model is trained using the Huber loss (robust to outliers) and the Adam optimizer ($\text{lr}=10^{-3}$), with a batch size of 32 and up to 50 epochs. Although early stopping (patience = 10) is employed during WFV, the final model is trained on all pre-2024 data without validation monitoring.
Evaluation follows a strict chronological protocol, a WFV scheme with 5-year training windows, 90-day validation folds, and 21-day rolling steps ensures realistic performance estimation, while the final out-of-sample test set comprises the last 365 trading days. Performance is assessed using RMSE, MAE, $R^2$, direction accuracy, and statistical significance tests (DM, paired $t$-test). All experiments use a fixed random seed (42) and were executed on AMD Ryzen 9 9900X processor, NVIDIA RTX 5060Ti GPU with 16 GB VRAM, 32 GB RAM, Python 3.10, TensorFlow 2.10, CUDA 12.9  workstation. The performance evaluation is structured in
two parts: (1) Comparison with baselines, and (2) Comparison with state-of-the-art models.

\begin{table}[!htbp]
\centering
\caption{WaVeFuse initialization, preprocessing, and training settings}
\label{tab:wavfuse_params}
\resizebox{\textwidth}{!}{
\begin{tabular}{p{2.2cm}p{3.5cm}p{7cm}p{7cm}}
\toprule
\textbf{Category} & \textbf{Parameter} & \textbf{Symbol / Value} & \textbf{Rationale} \\
\midrule
\multirow{5}{*}{Data} 
& Input window size & $w=20$ (\texttt{SEQ\_LENGTH}) & Captures short to medium term temporal dynamics \\
& OHLCV features  & $f_{\text{OHLCV}}=5$ (\texttt{Open, High, Low, Close, Volume}) & Preserves microstructure by avoiding close only bias \\
& TI channels  & $f_{\text{TI}}=7$ (\texttt{RSI, Stoch\%K, CCI, OBV, ATR, Williams\%R, ROC}) & Low-lag, momentum focused
signals \\
& Final test set & $N_{\text{test}}=365$ days (last year) & Strict chronological hold-out for out of sample evaluation \\
& Train/val in WFV & Train $\approx5$ years, Val $=90$ days, Step $=21$ days & Rolling origin validation mimicking live deployment \\
\midrule
\multirow{5}{*}{Preproc.}
& Denoising & Sym-4, level $=2$ & Removes high frequency noise while preserving trends \\
& Scaling (OHLCV) & MinMax$[0,1]$ (fit on train) & Stabilize optimization \\
& Scaling (TIs) & MinMax$[0,1]$ (fit on train) & Uniform range across indicators \\
& CWT on TIs (CWWT) & Mother: Morlet; $S=32$ uniform integer scales, $s \in \{1,\dots,32\}$ & Multi-scale temporal decomposition \\
& Target & Next-step Close (scaled) & Standard 1-step regression setup \\
\midrule
\multirow{9}{*}{Model}
& Architecture & WaVeFuse (Hybrid DL model) & Dual-branch for local and global dependencies \\
& CNN (branch 1) & Conv1D, filters $=64$, kernel $=3$, ReLU & Local pattern extraction over windows \\
& BiLSTM stack (b1) & BiLSTM(32, seq) $\rightarrow$ BiLSTM(64, last) & Bidirectional temporal encoding \\
& Transformer (branch 2) & MHA heads $=4$, key\_dim $=\{16,32\}$, MHA internal dropout $=0.2$ (2 layers) & Global dependencies over CWT--TI grid \\
& Projections & Dense to $d_{\text{fuse}}=32$ per branch & Dimensionality alignment \\
& Fusion mechanism & VAF ($d_{\text{gate}} = 1$, scalar gate per branch) & Learns dynamic weighting between branches \\
& Decoder & BiLSTM(64) $\rightarrow$ Dropout(0.4) $\rightarrow$ Dense(1) & Temporal decoding with bidirectional memory \\
& Regularization & Dropout $=0.4$  & Reduce overfitting \\
& Model Parameters  & 152,116 & Model capacity  \\
\midrule
\multirow{6}{*}{Training}
& Optimizer & Adam, $lr=10^{-3}$ & Stable adaptive updates \\
& Loss & Huber  & Robust to outliers vs MSE \\
& Epochs (WFV) & 50 (with early stop on val\_loss) & Avoid overtraining per fold \\
& Epochs (final fit) & 50  & Train on all pre-2024 data \\
& Batch size & 32 & Throughput–variance trade-off \\
& Callbacks & EarlyStopping (patience=10) in WFV & Generalization-oriented selection \\
\midrule
\multirow{4}{*}{Protocol}
& Walk-forward validation & Rolling train/val windows; step $=21$ days & Mimics live deployment \\
& Test evaluation & Hold-out last 365 days & Genuine out-of-sample performance \\
& Metrics & MAE, RMSE, MAPE (+ R$^2$, DA, DM-test in analysis) & Error magnitude, explanatory power, significance \\
& Reproducibility & \texttt{np}/TF seeds $=42$ & Repeatable runs \\
\bottomrule
\end{tabular}}
\end{table}

\subsection{Comparison with Baseline Models} \label{sec:baseline_comparison}
To empirically validate the necessity of each architectural contribution in WaVeFuse, we evaluate three single-branch baselines that isolate specific sub-components of the full architecture. Specifically, we evaluate: (i) a CNN applying 1D convolutions over denoised OHLCV sequences to capture local spatial patterns, (ii) a BiLSTM processing the same temporal input through stacked bidirectional recurrent layers to capture sequential dependencies, and (iii) a Transformer operating exclusively on the CWWT spectral matrix $\mathbf{Z}_t \in \mathbb{R}^{32 \times 7}$ to model inter-scale frequency dependencies. Crucially, all baselines are trained under the identical sliding-window WFV protocol described in Section~\ref{sec:problem}, using the same Huber loss ($\delta = 0.1$), Adam optimizer, early stopping (patience $= 10$), and normalization procedure. This ensures that observed performance differences are attributable exclusively to architectural design rather than training protocol asymmetries.

The CNN baseline applies two 1D convolutional layers (64 and 128 filters, kernel size 3) followed by global average pooling and dense projection layers with dropout ($p = 0.4$), producing a scalar prediction from the denoised OHLCV window. The BiLSTM baseline stacks two bidirectional LSTM layers (32 and 64 units) over the same temporal input, using the final hidden state for prediction. The Transformer baseline applies two multi-head self-attention layers (4 heads, $d_k = 16$) to the CWWT matrix, followed by global average pooling and dense output projection.\footnote{Note that the standalone Transformer baseline retains $d_k = 16$ for both layers, while the WaVeFuse spectral branch expands to $d_k = 32$ in the second layer. This intentional constraint isolates the architectural contribution of the dual-branch design from raw parameter count, ensuring the baseline's capacity remains comparable to the full model despite lacking the temporal CNN-BiLSTM parameters.} All baselines use the same linear output layer and Huber loss target. Table~\ref{tab:wave_compact} reports MAE, RMSE, and MAPE for all models across four equity indices. WaVeFuse achieves the best performance on every metric across all four indices, with the following key findings.

\begin{table}[htbp]
\centering
\caption{WaVeFuse vs.\ baseline comparison on four indices (best values \textbf{bold} and \% reduction in \textit{italics})}
\label{tab:wave_compact}
\resizebox{0.9\textwidth}{!}{
\begin{tabular}{lllllllllllll}
\toprule
\textbf{Model}
& \multicolumn{3}{c}{\textbf{KOSPI}}
& \multicolumn{3}{c}{\textbf{DAX}}
& \multicolumn{3}{c}{\textbf{NYSE}}
& \multicolumn{3}{c}{\textbf{Russell 2000}} \\
\cmidrule(lr){2-4} \cmidrule(lr){5-7} \cmidrule(lr){8-10} \cmidrule(lr){11-13}
& MAE & RMSE & MAPE & MAE & RMSE & MAPE & MAE & RMSE & MAPE & MAE & RMSE & MAPE \\
\midrule
BiLSTM       & 15.28 & 18.84 & 0.61 & 219.36 & 275.03 & 1.40 & 194.38 & 238.15 & 1.11 & 37.28 & 46.14 & 0.96 \\
CNN          & 17.62 & 22.15 & 0.71 & 255.30 & 330.00 & 1.63 & 225.43 & 281.67 & 1.29 & 43.47 & 54.92 & 1.12 \\
Transformer  & 13.81 & 16.73 & 0.55 & 199.74 & 246.78 & 1.27 & 176.24 & 214.58 & 1.01 & 33.94 & 41.28 & 0.87 \\
\textbf{WaVeFuse} &
\textbf{12.30} & \textbf{14.96} & \textbf{0.49} &
\textbf{178.34} & \textbf{217.93} & \textbf{1.14} &
\textbf{157.51} & \textbf{192.21} & \textbf{0.90} &
\textbf{30.34} & \textbf{37.13} & \textbf{0.78} \\
\midrule
\multicolumn{1}{r@{\,}}{WaVeFuse vs.\ BiLSTM:} &
\textit{19.5} & \textit{20.6} & \textit{19.7} &
% \textit{18.7} & \textit{19.9} & \textit{18.6} &
\textit{18.7} & \textit{20.8} & \textit{18.6} &
\textit{19.0} & \textit{19.3} & \textit{18.9} &
\textit{18.6} & \textit{19.5} & \textit{18.8} \\
\multicolumn{1}{r@{\,}}{WaVeFuse vs.\ CNN:} &
\textit{30.2} & \textit{32.5} & \textit{31.0} &
% \textit{30.1} & \textit{31.7} & \textit{30.2} &
\textit{30.1} & \textit{34.0} & \textit{30.1} &
\textit{30.1} & \textit{31.7} & \textit{30.2} &
\textit{30.2} & \textit{32.4} & \textit{30.4} \\
\multicolumn{1}{r@{\,}}{WaVeFusevs.\ Transformer:} &
\textit{10.9} & \textit{10.6} & \textit{10.9} &
\textit{10.8} & \textit{11.7} & \textit{10.2} &
\textit{10.6} & \textit{10.4} & \textit{10.9} &
\textit{10.6} & \textit{10.0} & \textit{10.3} \\
\bottomrule
\end{tabular}}
\footnotesize
\setlength{\tabcolsep}{3pt}
\begin{tablenotes} \footnotesize \item Key: MAE = Mean Absolute Error, RMSE = Root Mean Square Error, MAPE = Mean Absolute Percentage Error. Italic values represent the percentage improvements of WaVeFuse compared to BiLSTM, CNN, and Transformer models over different indices.  \end{tablenotes}
\end{table}

\begin{figure}
    \centering
    \includegraphics[width=0.8\linewidth]{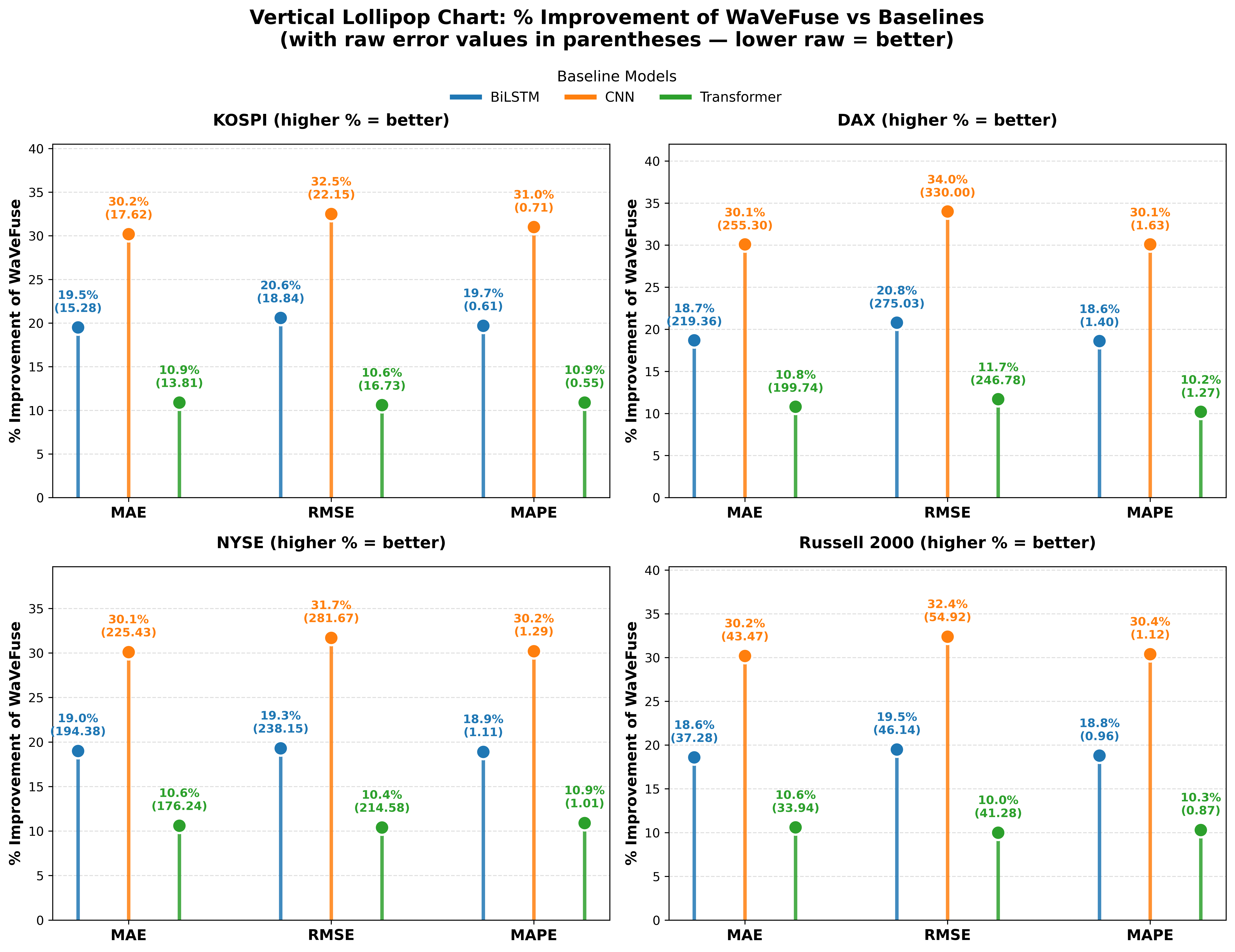}
    \caption{Percentage improvement of WaVeFuse over the three single-branch baselines (BiLSTM, CNN, Transformer) across four equity indices. Raw error values of the baselines are shown in parentheses next to each point (lower is better). Higher percentage values indicate better relative performance of WaVeFuse.}
    \label{fig:improvements}
\end{figure}

The consistent superiority of WaVeFuse is visualized in Figure~\ref{fig:improvements}, which presents the percentage improvement over each baseline together with the raw error values of the baselines (in parentheses). Across all twelve evaluation settings (4 indices $\times$ 3 metrics), WaVeFuse reduces MAE by 18.6\%--19.5\% compared to BiLSTM, 30.1\%--30.2\% versus CNN, and 10.6\%--10.9\% versus the pure spectral Transformer. The largest gains are against the CNN baseline (up to 34.0\% RMSE reduction on DAX), while the smallest, yet still meaningful improvements occur against the Transformer (approximately 10--11\% across most settings). Importantly, the absolute error levels confirm that WaVeFuse achieves lower forecasting errors across all indices, with particularly pronounced reductions on higher-magnitude series (DAX and NYSE). The comparison with baseline models is as follows:
\begin{enumerate}
    \item \textbf{CNN vs. WaVeFuse (+43\% MAE gap):} The CNN baseline trails WaVeFuse by approximately 43\% on MAE across all indices. That is, CNN's MAE exceeds WaVeFuse's MAE by 43\% when normalized to WaVeFuse's error level. Equivalently, WaVeFuse reduces CNN's absolute MAE by approximately 30\% (Table \ref{tab:wave_compact}, row vs.~CNN). Figure \ref{fig:improvements} provides a complementary perspective on the same performance gap. This large deficit quantifies the combined contribution of sequential memory, spectral features, and adaptive fusion. The CNN's receptive field, bounded by kernel size $k=3$, captures only local 3-day price patterns and cannot represent the momentum persistence and autocorrelation structure that extends over 10--20 trading days. Without recurrent memory, the CNN cannot model trend continuations or mean-reversion dynamics that characterize daily equity index behavior. The consistent RMSE degradation of approximately 49--50\% (exceeding MAE degradation) confirms that CNN produces disproportionately large errors during market events, reflecting its inability to adapt to volatile regimes without sequential context.

    \item \textbf{BiLSTM vs. WaVeFuse (+23\% MAE gap):} The BiLSTM baseline, which recovers full sequential memory over the 20-day causal lookback window, narrows the gap substantially to approximately 23\% MAE. This reduction from 43\% to 23\% demonstrates that temporal dependencies account for the majority of the CNN's deficit, confirming the importance of recurrent encoding for financial time series. However, the persistent 23\% gap between the BiLSTM and WaVeFuse directly quantifies the contribution of the spectral branch and adaptive fusion mechanism. The BiLSTM operates exclusively on raw OHLCV sequences and cannot access the multi-resolution frequency structure encoded by the CWWT matrix $\mathbf{Z}_t$. Specifically, it cannot model inter-scale energy transfer during volatility shocks.

    \item \textbf{Transformer vs. WaVeFuse (+12\% MAE gap):} The Transformer baseline, which receives the full CWWT spectral representation $\mathbf{Z}_t \in \mathbb{R}^{32 \times 7}$ and applies inter-scale self-attention across all 32 wavelet scales, achieves the strongest single-branch performance, trailing WaVeFuse by approximately 12\% on MAE. This residual gap quantifies the joint contribution of the raw temporal OHLCV branch and the VAF adaptive fusion mechanism. While the spectral branch captures a large fraction of WaVeFuse's predictive power through its multi-resolution representation of TI dynamics, it lacks direct access to causal price-volume sequences that encode momentum, volume confirmation, and short-term trend persistence.
\end{enumerate}

\textbf{Complementarity of branches:} A notable structural finding emerges from Table~\ref{tab:wave_compact}: the Transformer and WaVeFuse\_no\_CWWT ablation variant (detailed in Section~\ref{sec:ablation}) achieve near-identical performance across all indices, with MAE differences of less than 0.8 index points for KOSPI, 0.79 for DAX, 1.15 for NYSE, and 0.23 for Russell 2000. This near-parity is theoretically interpretable as the standalone Transformer has perfect spectral input but no temporal branch, while WaVeFuse\_no\_CWWT has a full temporal branch but a degraded spectral branch relying on raw TI values. These compensating deficits yield statistically equivalent performance, providing strong empirical evidence that the temporal and spectral branches occupy genuinely complementary representational subspaces, consistent with Proposition~1.

\textbf{Cross-index consistency:} The relative performance gains of WaVeFuse over each baseline are remarkably stable across all four indices, with standard deviations of $\pm$0.11\%, $\pm$0.55\%, and $\pm$0.18\% for CNN, BiLSTM, and Transformer gaps respectively. This cross-index stability confirms that the architectural advantage of WaVeFuse is systematic and generalizable rather than dataset-specific. In particular, the consistency across KOSPI (Asian emerging-developed market), DAX (European developed), NYSE Composite (US large-cap), and Russell 2000 (US small-cap) demonstrates that the dual-branch design captures market dynamics that transcend regional microstructure differences and capitalization tiers.

\textbf{RMSE-MAE asymmetry:} Across all baselines and indices, RMSE degrades more than MAE relative to WaVeFuse (approximately 49--50\% vs. 43\% for CNN, 19--21\% vs. 18--20\% for BiLSTM). This asymmetry is consistent with the tail-robustness argument of Proposition~2: WaVeFuse's advantage is disproportionately concentrated on large-error market events (volatility spikes, directional reversals) where the VAF mechanism actively reallocates branch weights to reduce prediction uncertainty. Single-branch baselines lacking this adaptive mechanism accumulate systematically larger tail errors, inflating RMSE beyond their average-error (MAE) deficit. This pattern provides indirect empirical validation of the precision-weighted fusion optimality result (Theorem~\ref{theoram:1}, Remark~2).

Figure~\ref{fig:actual_vs_predicted_index_results} presents the out-of-sample actual versus predicted closing prices for KOSPI, GDAXI, NYSE Composite, and Russell 2000. Across all indices, WaVeFuse closely tracks the realized price trajectories, capturing both medium-term trends and short-term fluctuations without excessive smoothing.

\begin{figure}
    \centering
    \subfigure[KOSPI]{\includegraphics[width=0.40\textwidth]{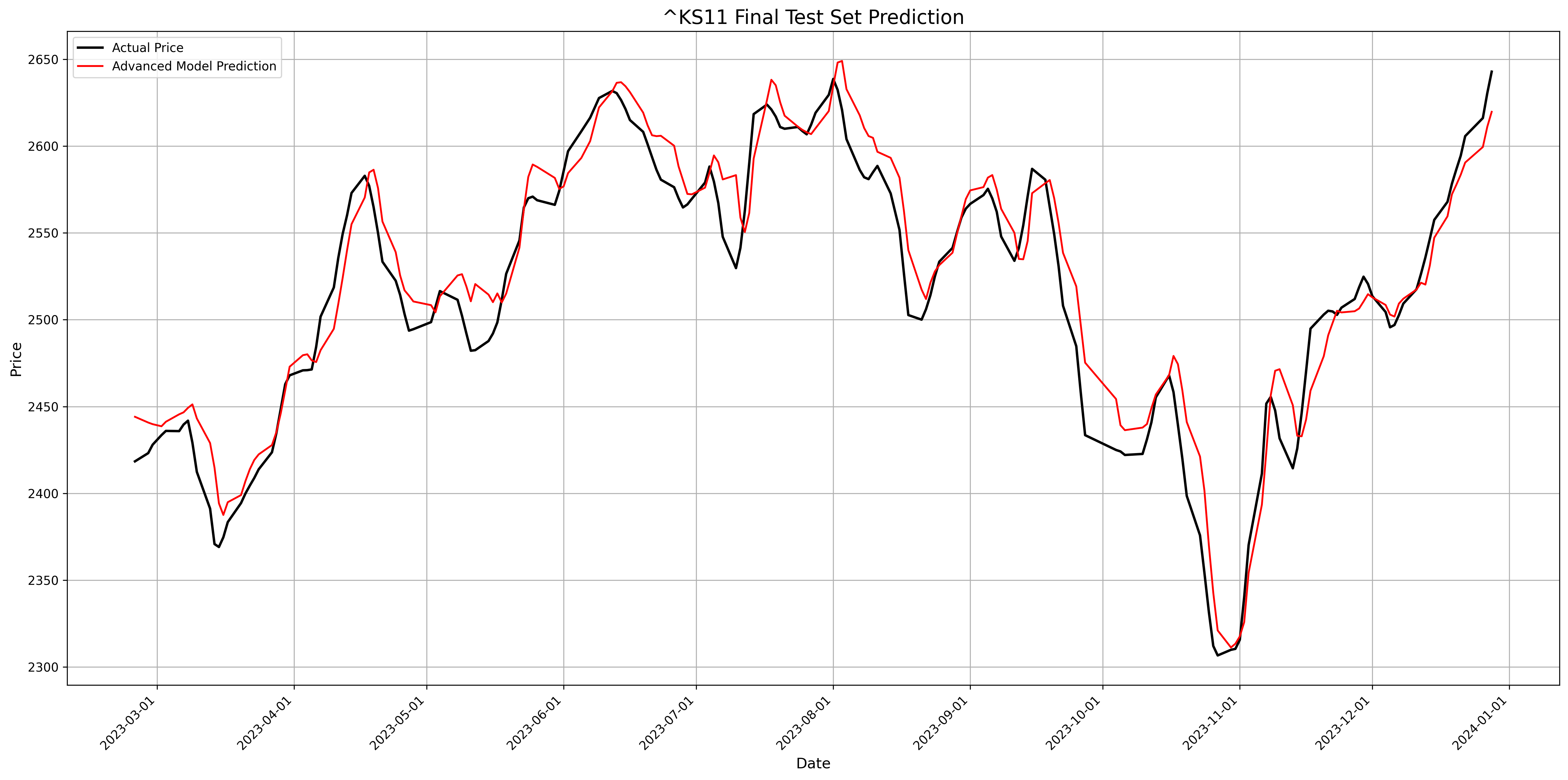}}
    \subfigure[GDAXI]{\includegraphics[width=0.40\textwidth]{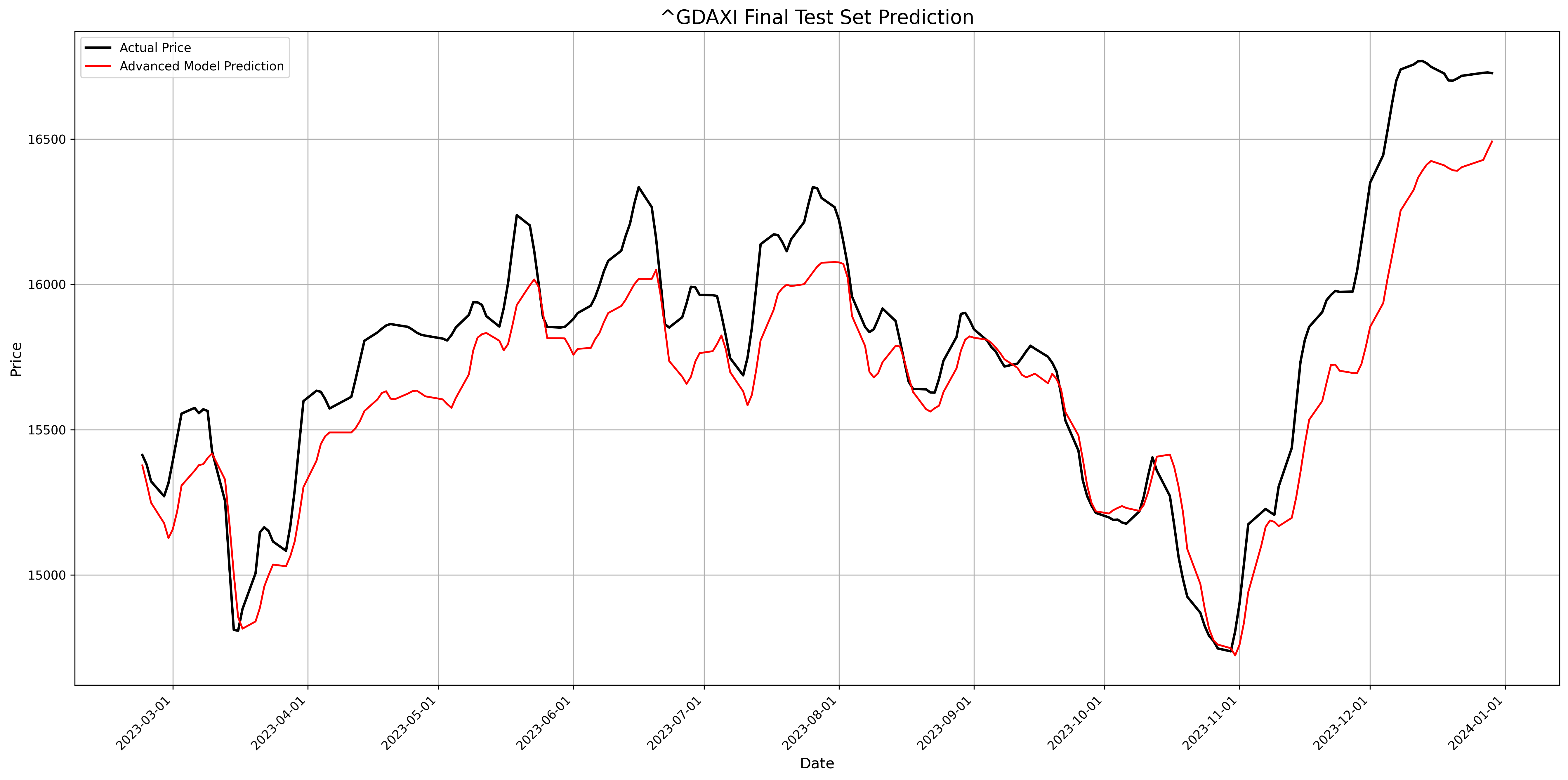}}
    \subfigure[NYSE]{\includegraphics[width=0.40\textwidth]{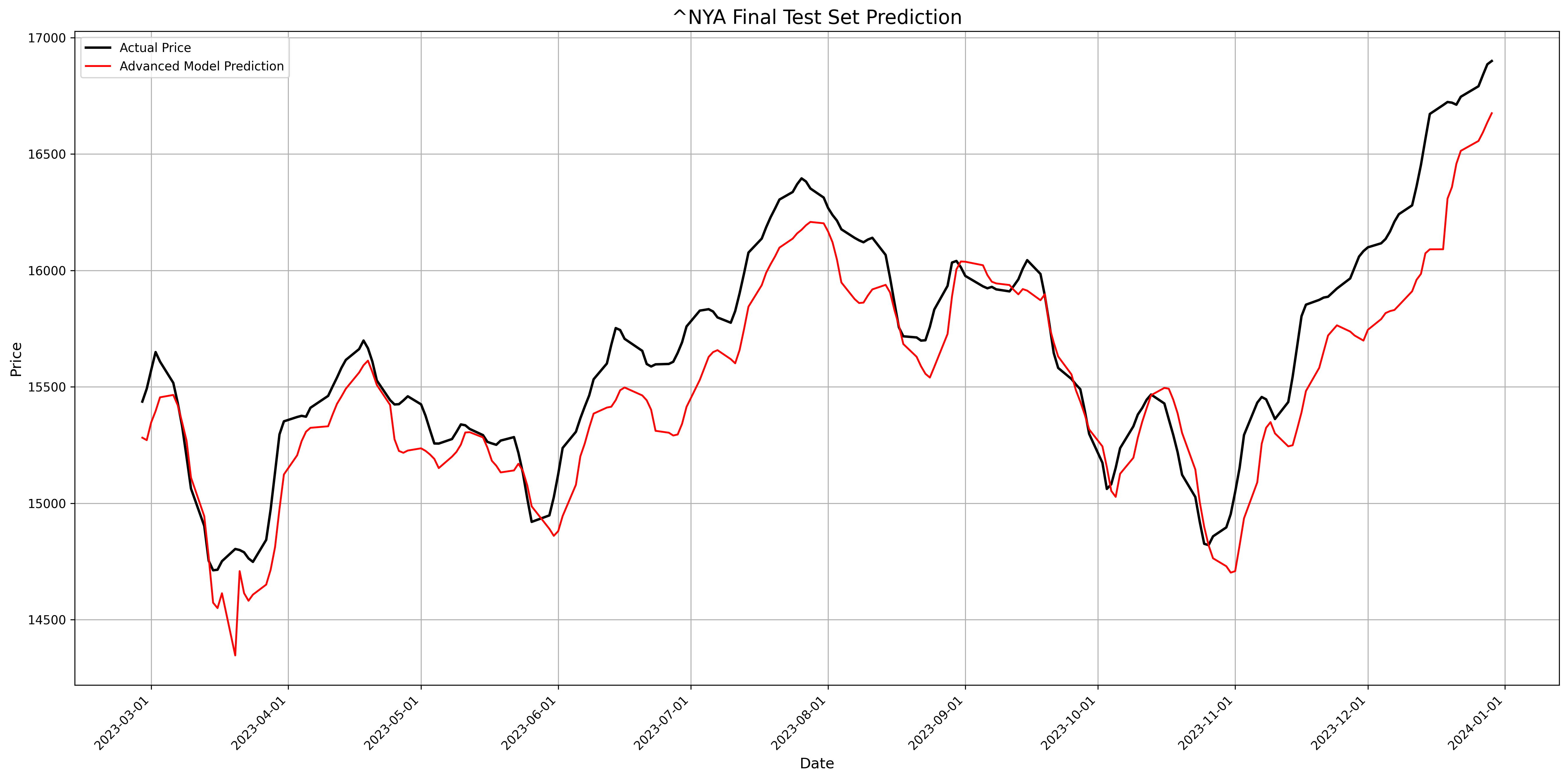}}
    \subfigure[Russell 2000]{\includegraphics[width=0.40\textwidth]{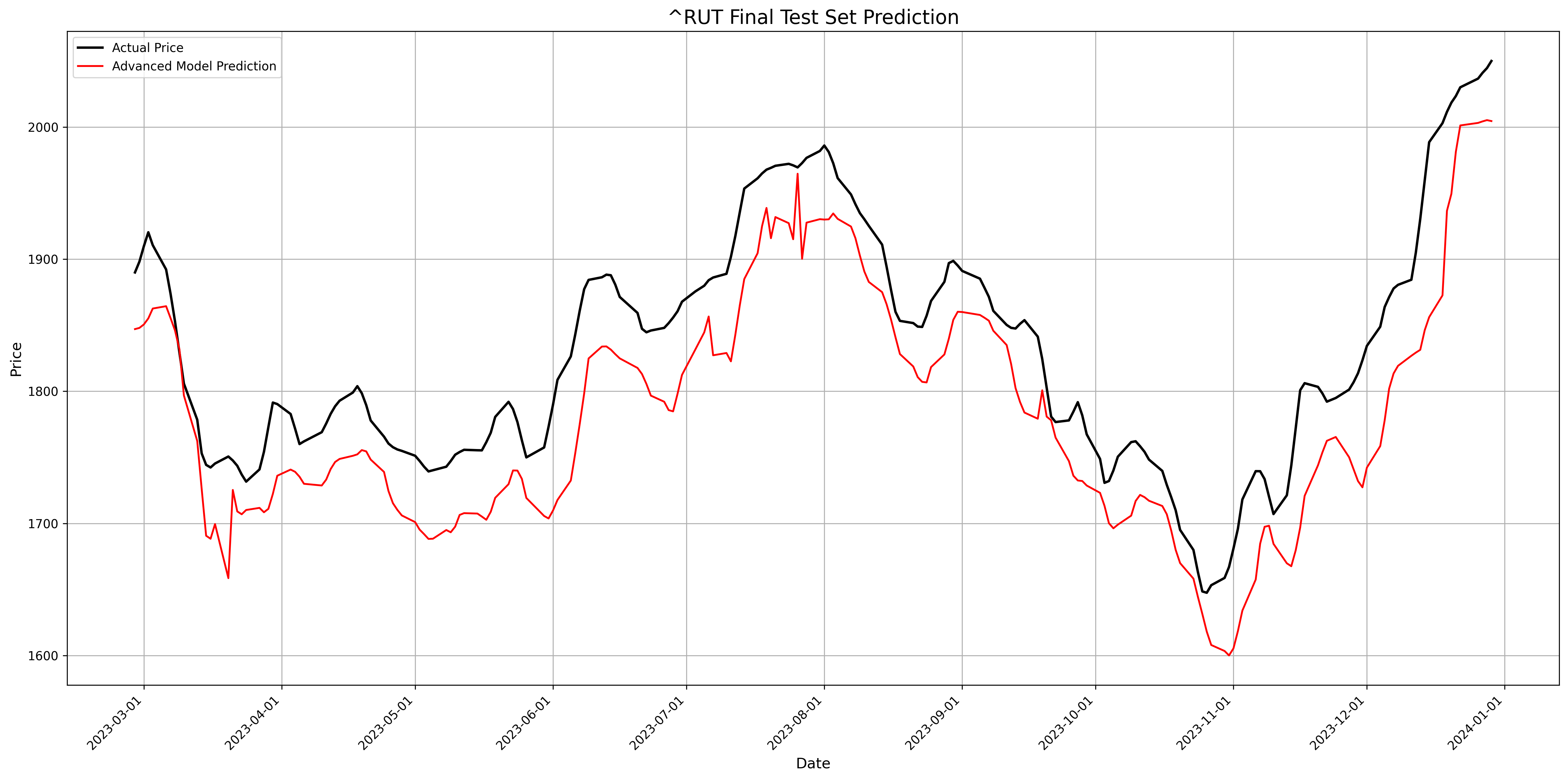}}
    \caption{Actual vs. Predicted closing prices for all four equity indices over the 365-day out-of-sample test period (January--December 2023). WaVeFuse predictions are in red and realized prices in black.}
    \label{fig:actual_vs_predicted_index_results}
\end{figure}

Looking at the prediction plots, WaVeFuse excels when trends shift slowly or volatility stays moderate. The model thrives when time-based momentum and multi-scale signals align. But during sudden market shocks or sharp regime changes, errors spike briefly. This isn't surprising, as financial models struggle with unpredictable external events and market efficiency constraints. Importantly, the model rapidly re-stabilizes after such events, suggesting effective internal regime adaptation rather than persistent drift. The tight alignment between predicted and observed series is further confirmed by the scatter plots in Figure~\ref{fig:scatter_plot_results}. Points cluster strongly around the identity line for all indices, with high correlation coefficients and consistently strong $R^2$ values, indicating that WaVeFuse explains a substantial fraction of price variance even under volatile market conditions. The performance remains stable across structurally different markets, including developed (NYSE Composite), small-cap (Russell 2000), and export-driven (KOSPI, DAX) indices. This cross-market robustness supports the central design hypothesis of WaVeFuse: orthogonal temporal–spectral modelling combined with adaptive fusion generalizes better than single-domain architectures.

\begin{figure}
    \centering
    \subfigure[KOSPI]{\includegraphics[width=0.40\textwidth]{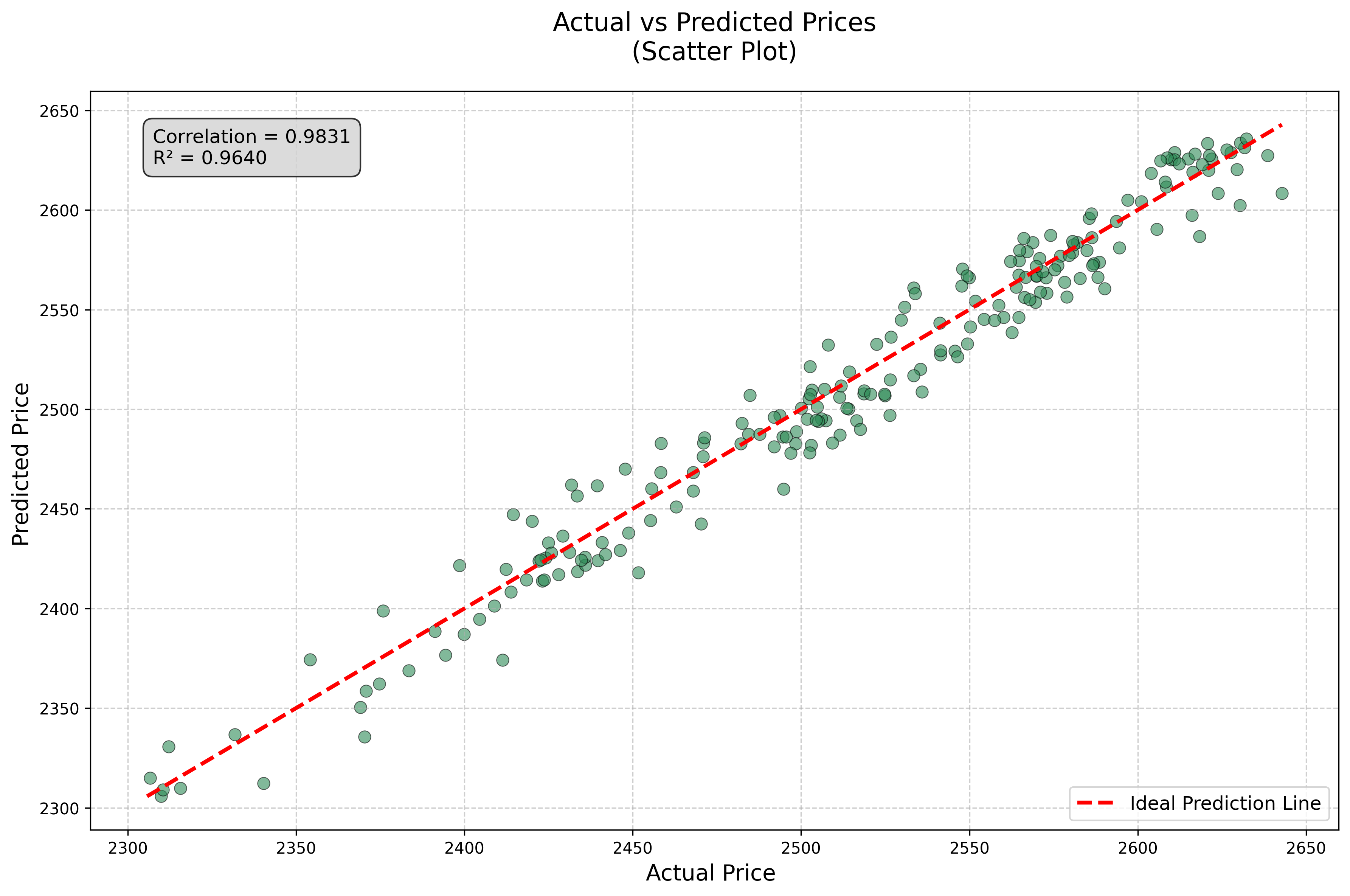}}
    \subfigure[GDAXI]{\includegraphics[width=0.40\textwidth]{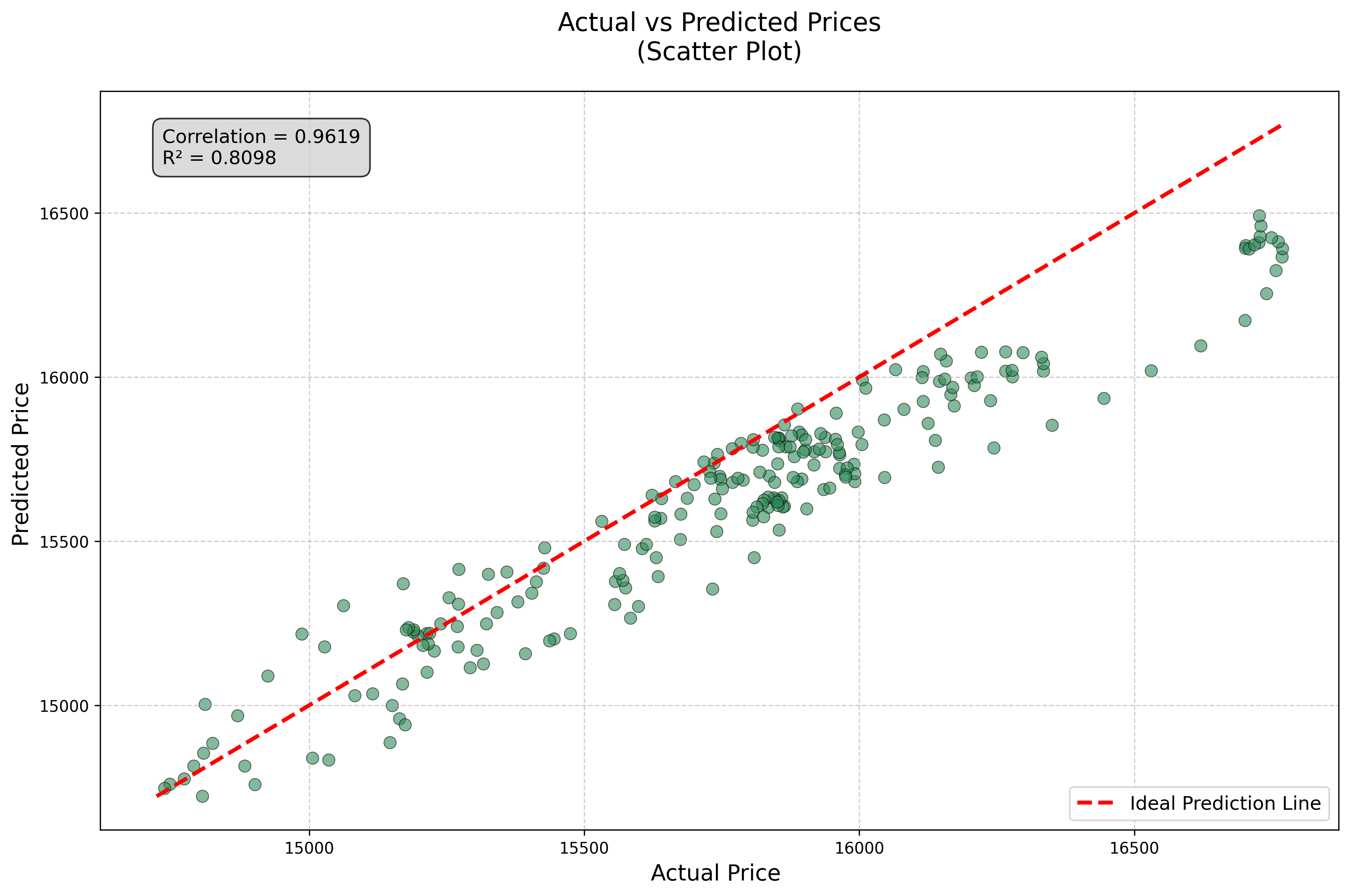}}
    \subfigure[NYSE]{\includegraphics[width=0.40\textwidth]{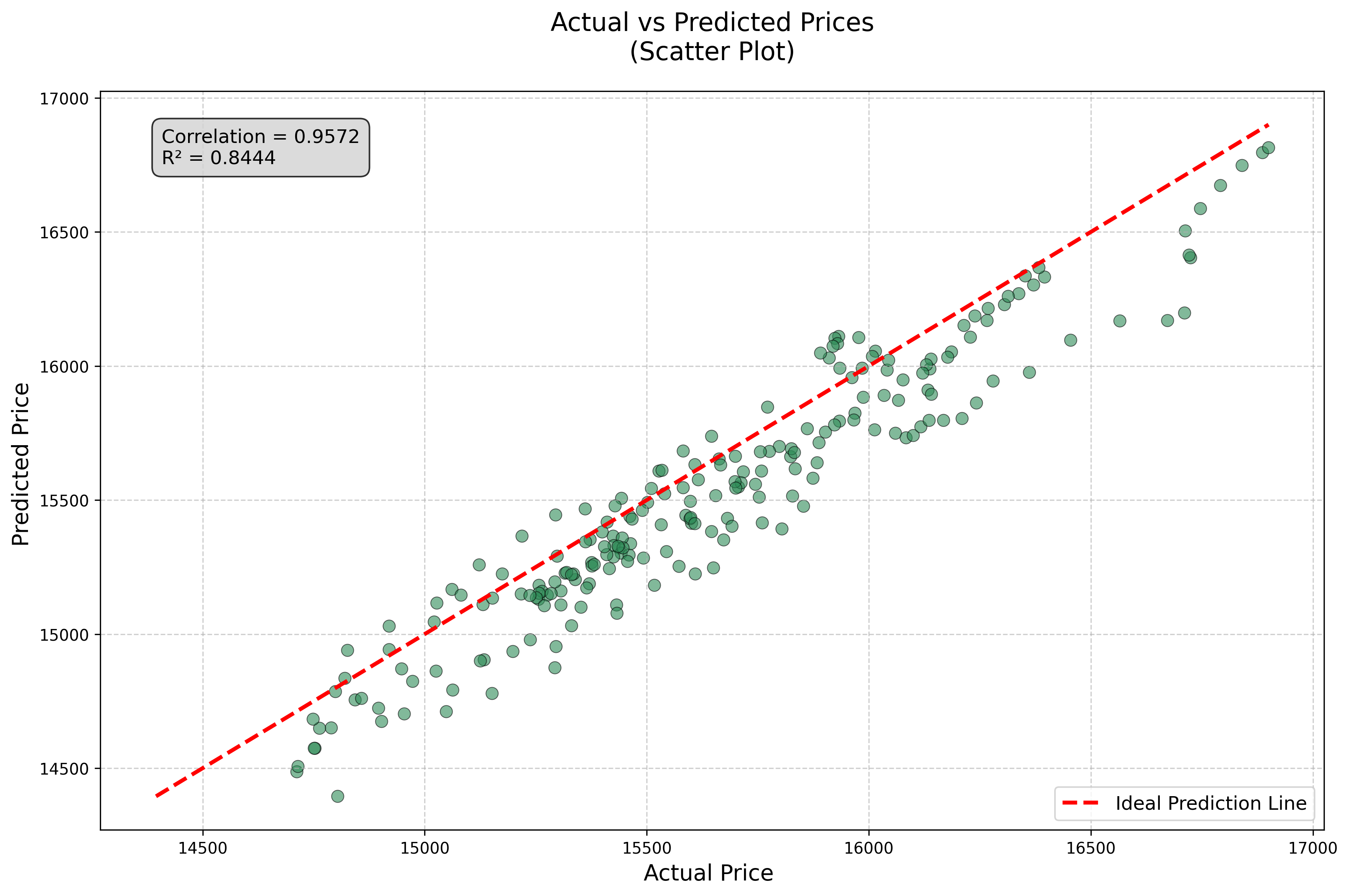}}
    \subfigure[Russell 2000]{\includegraphics[width=0.40\textwidth]{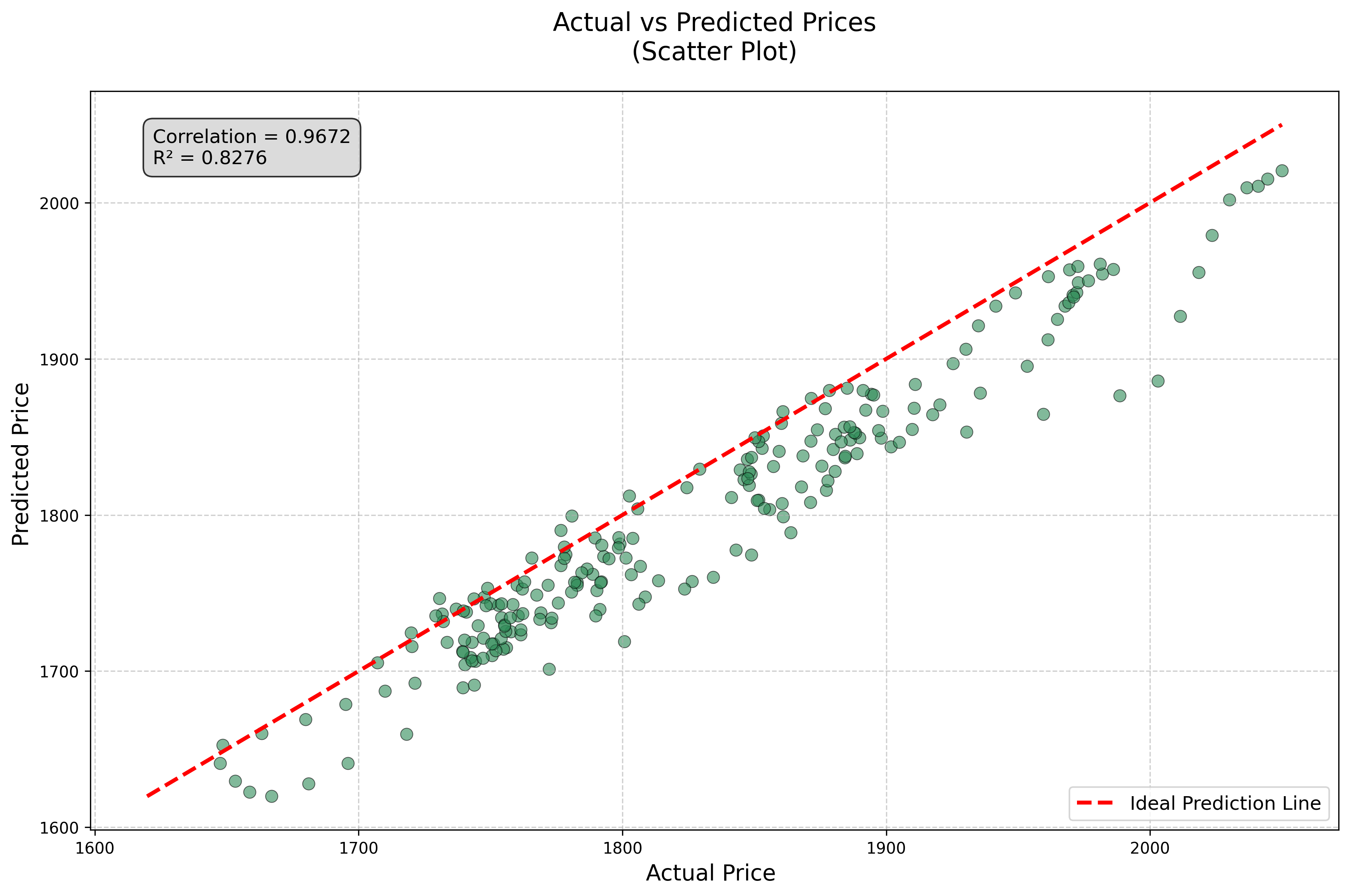}}
    \caption{Scatter plot with correlation and $R^2$ value}
    \label{fig:scatter_plot_results}
\end{figure}

To interpret how WaVeFuse integrates its dual input streams, we visualize the attention weights produced by the VAF module. Figure~\ref{fig:attention_weight_visualization} shows the temporal evolution of attention weights assigned to the OHLC (temporal) and CWWT (spectral) branches. Across all indices, the model consistently assigns higher average weight to the CWWT branch, indicating that wavelet-based multi-scale representations of TIs carry stronger predictive signal than raw price sequences alone. This finding empirically validates the core architectural motivation behind WaVeFuse.

\begin{figure}
    \centering
    \subfigure[KOSPI]{\includegraphics[width=0.40\textwidth]{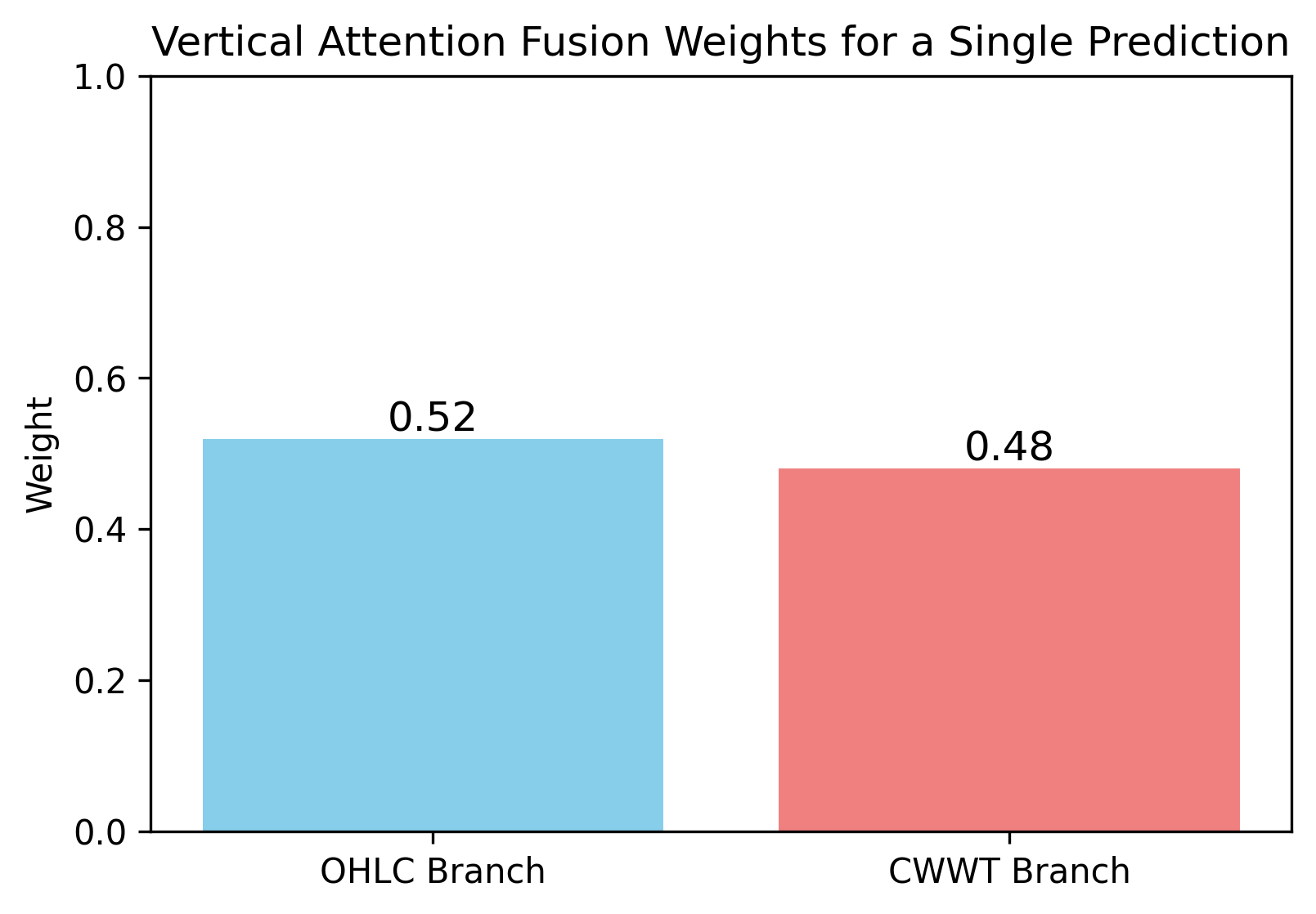}}
    \subfigure[GDAXI]{\includegraphics[width=0.40\textwidth]{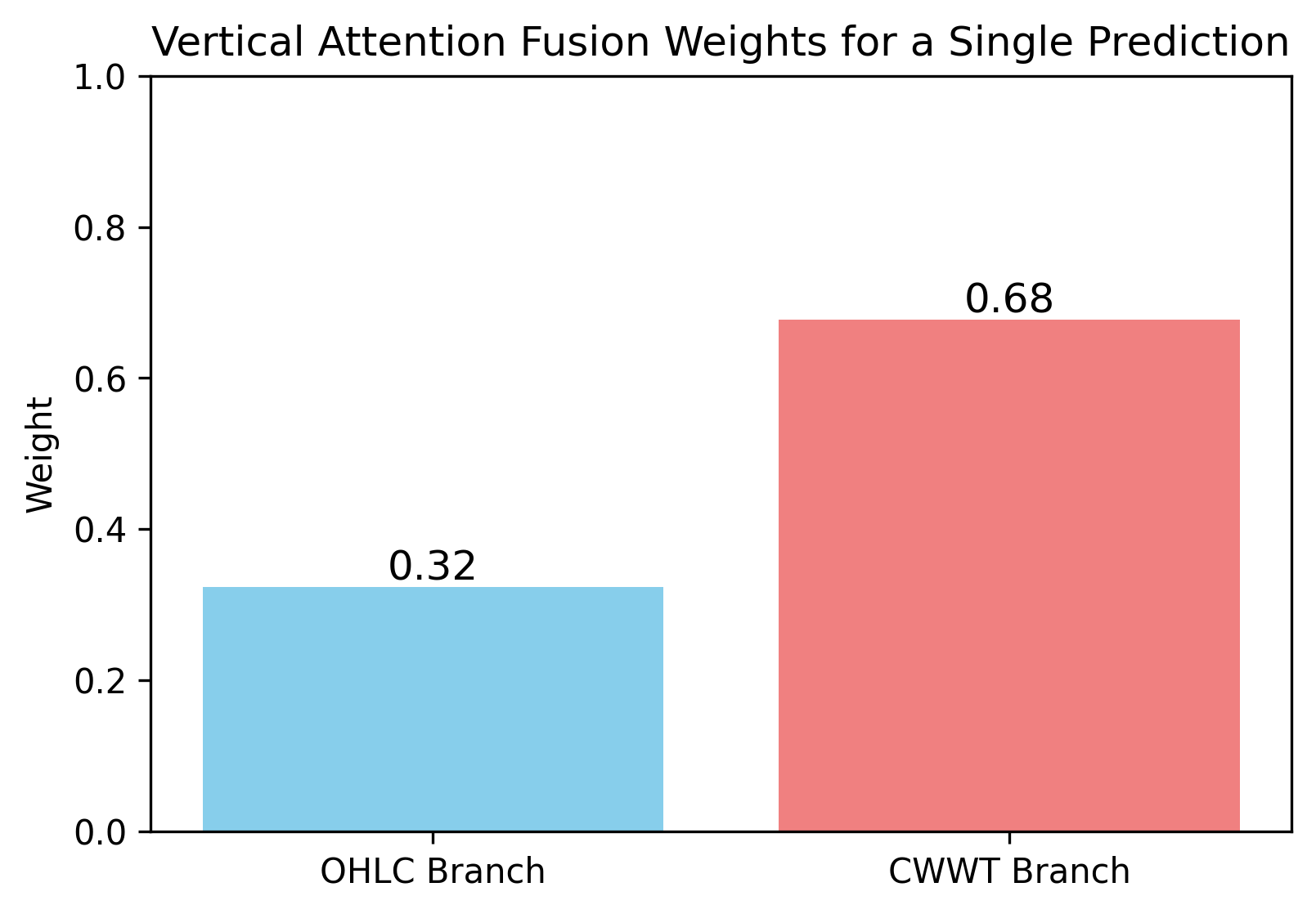}}
    \subfigure[NYSE]{\includegraphics[width=0.40\textwidth]{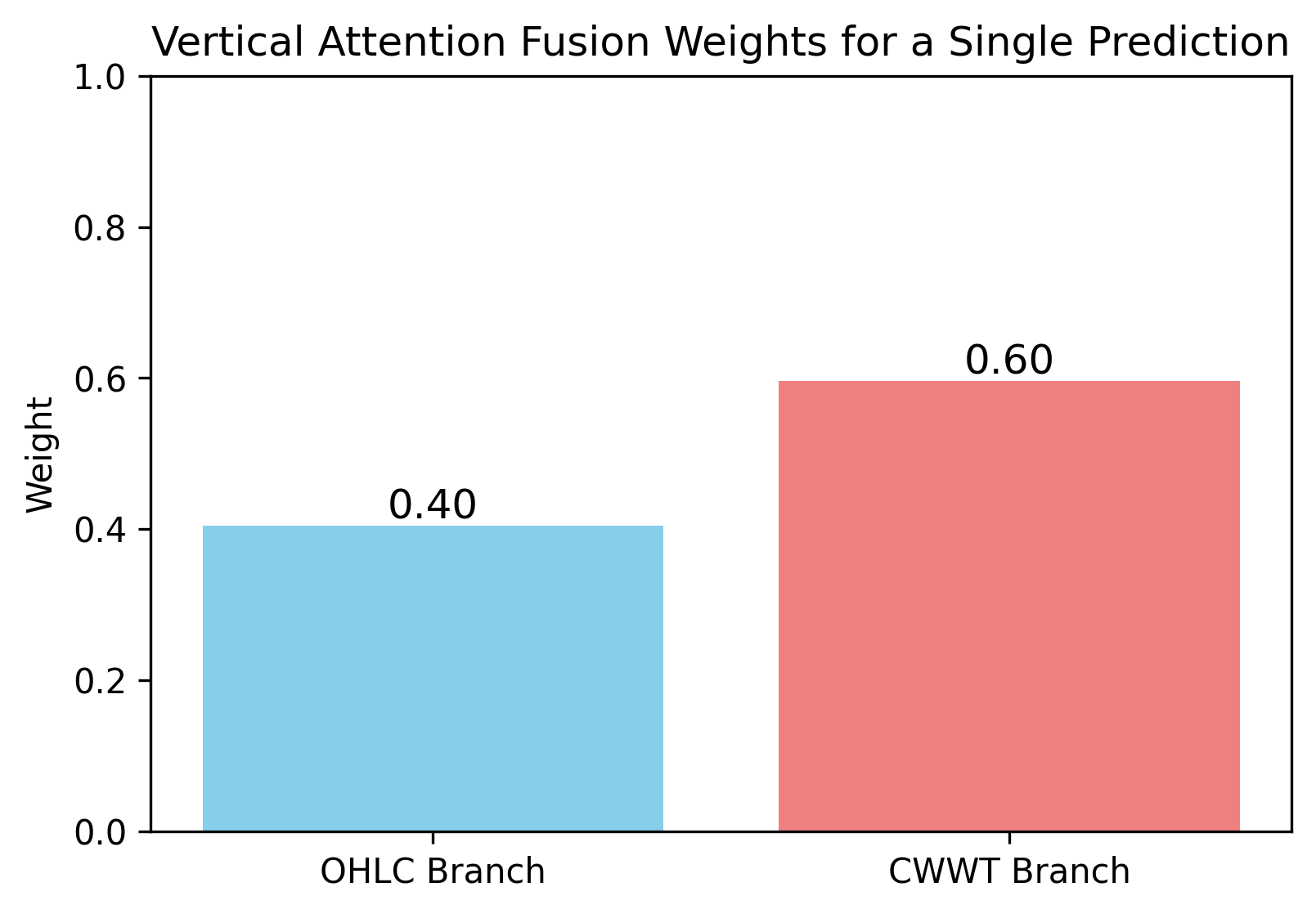}}
    \subfigure[Russell 2000]{\includegraphics[width=0.40\textwidth]{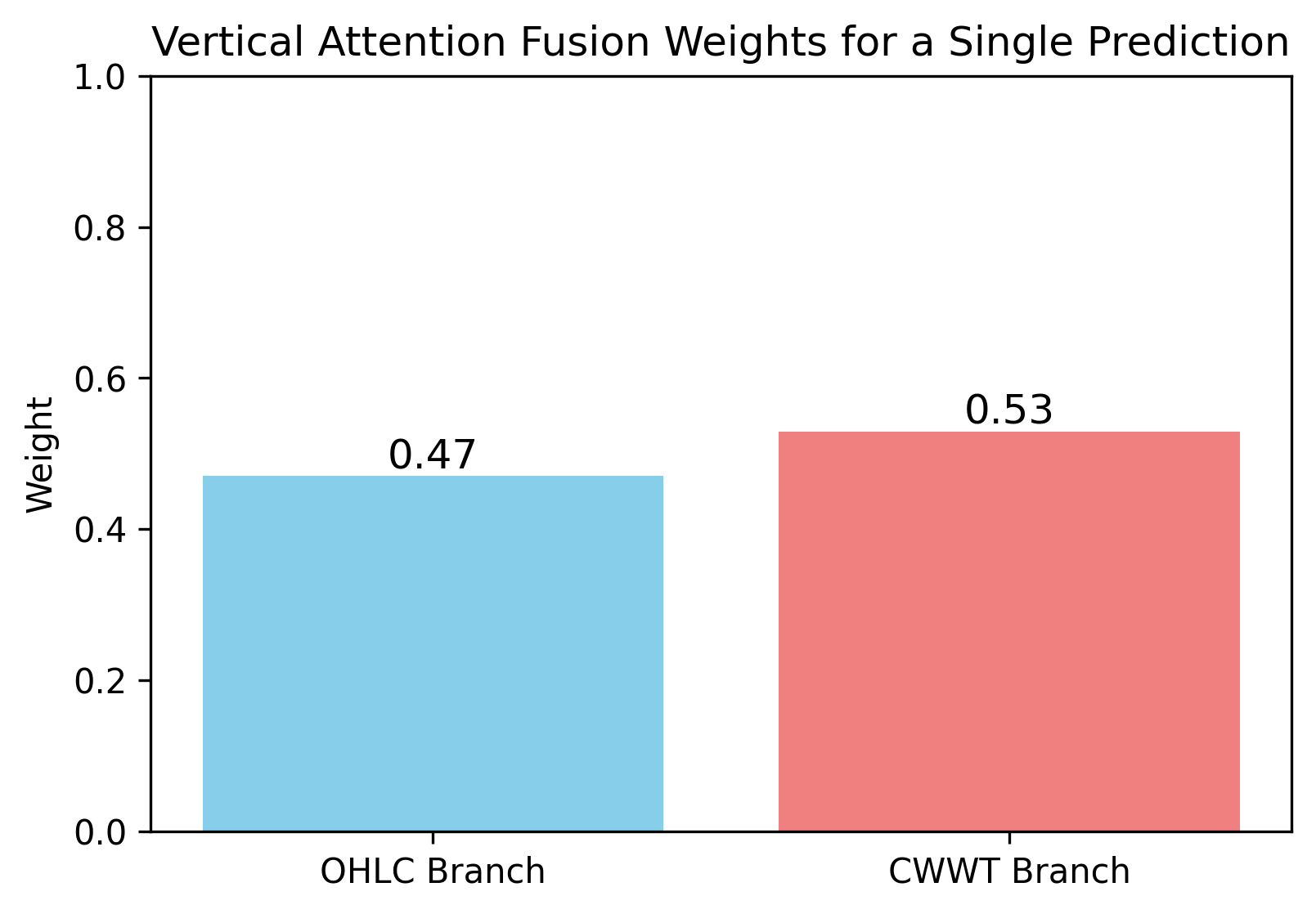}}
    \caption{VAF branch attention weights for a representative test-set prediction across all four indices. The CWWT spectral branch consistently receives higher weight than the OHLC temporal branch, confirming the predictive dominance of multi-scale indicator features}
    \label{fig:attention_weight_visualization}
\end{figure}

Figure~\ref{fig:attention_comparison} further demonstrates that this preference is stable across randomly sampled test instances and not driven by isolated episodes. Rather than collapsing to a single branch, the fusion mechanism dynamically balances information sources while maintaining a clear dominance of spectral features under diverse market conditions.

\begin{figure}
    \centering
    \subfigure[KOSPI]{\includegraphics[width=0.40\textwidth]{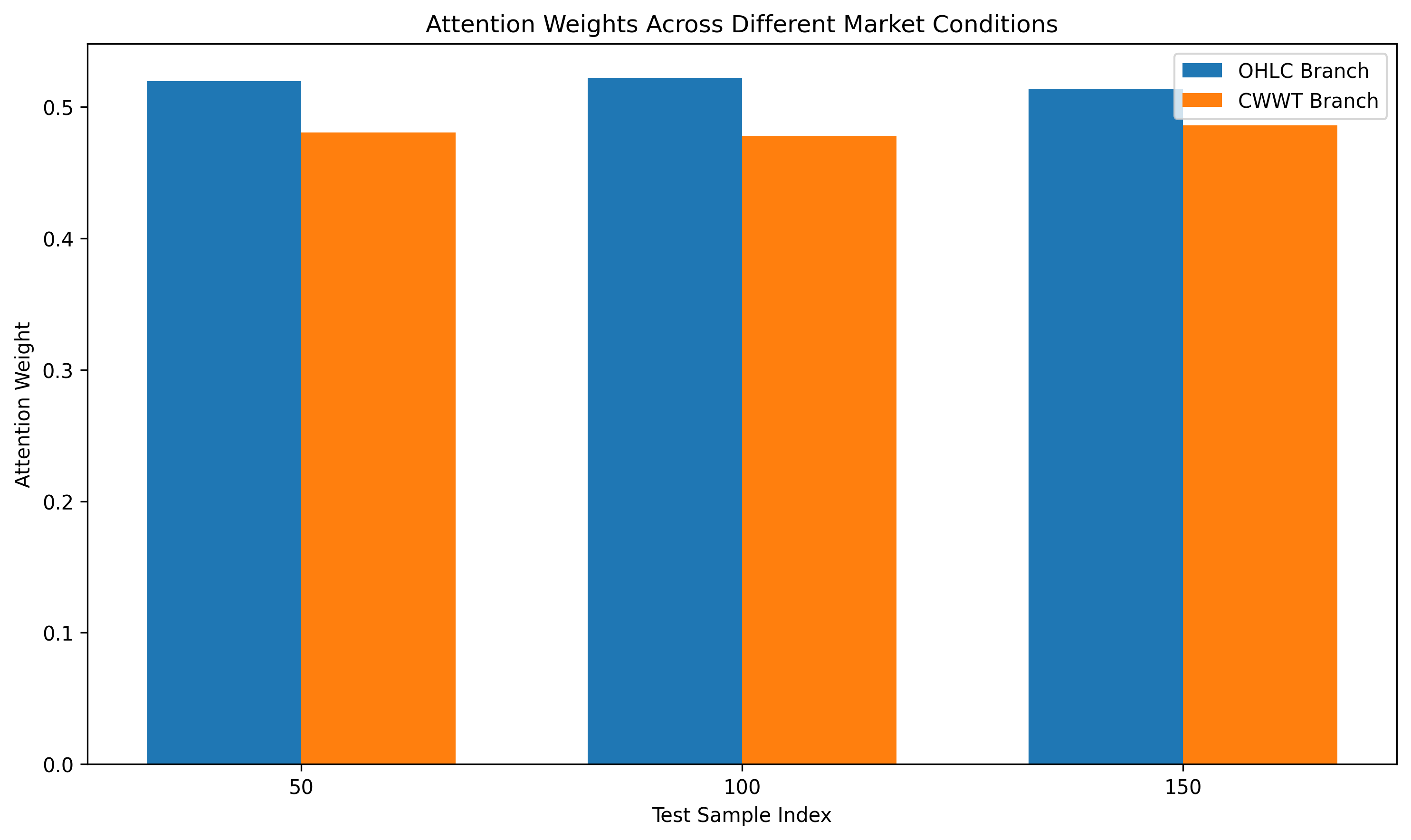}}
    \subfigure[GDAXI]{\includegraphics[width=0.40\textwidth]{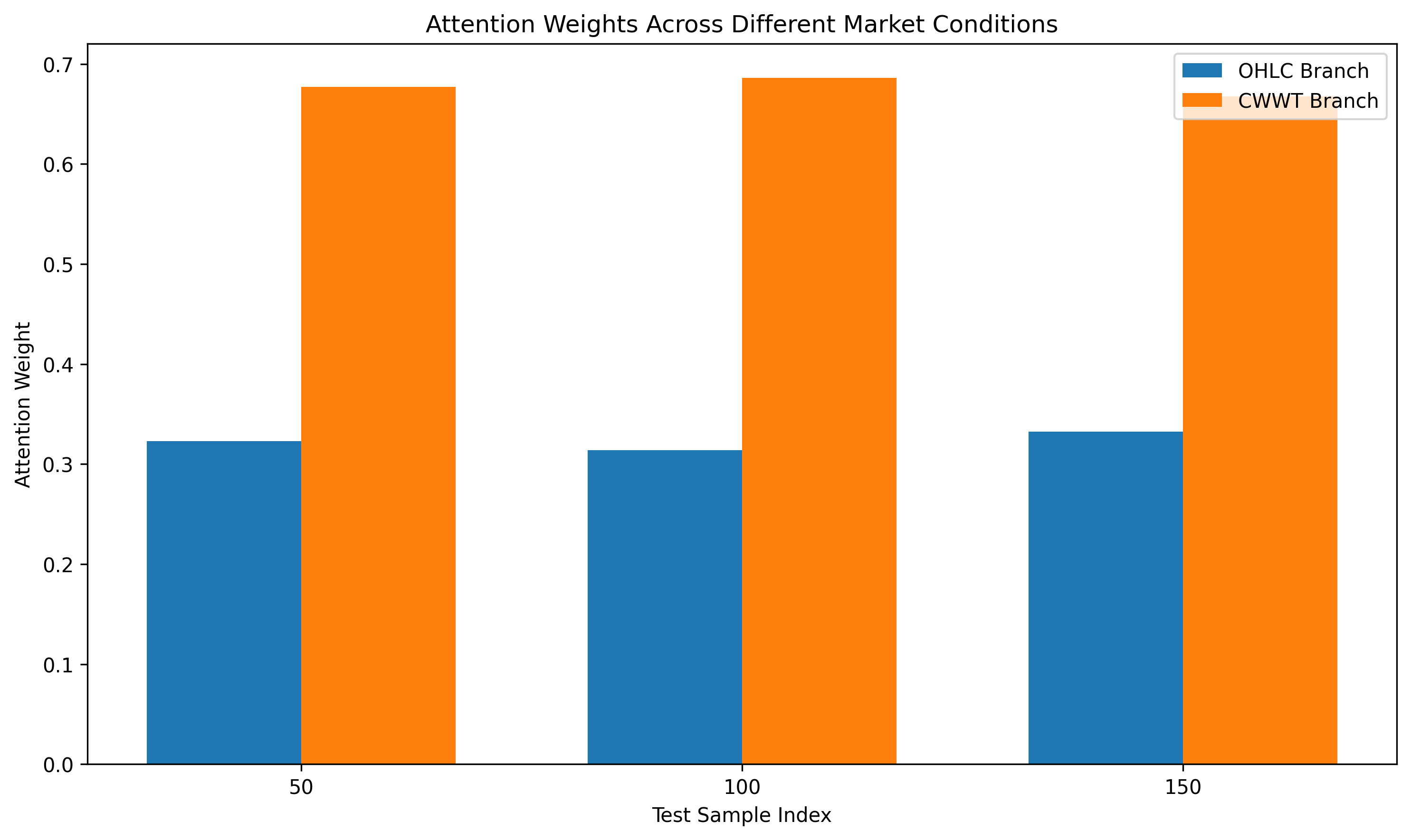}}
    \subfigure[NYSE]{\includegraphics[width=0.40\textwidth]{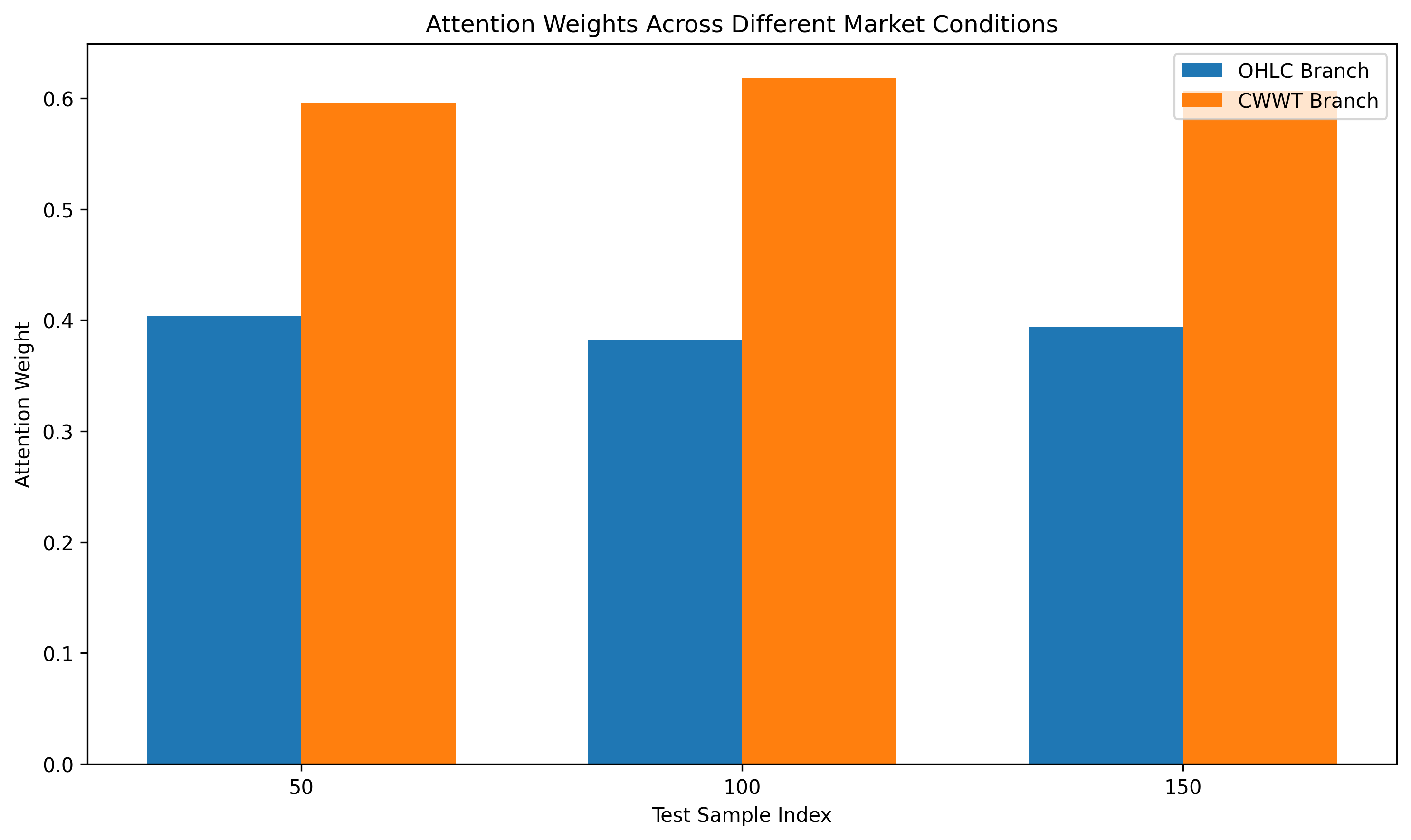}}
    \subfigure[Russell 2000]{\includegraphics[width=0.40\textwidth]{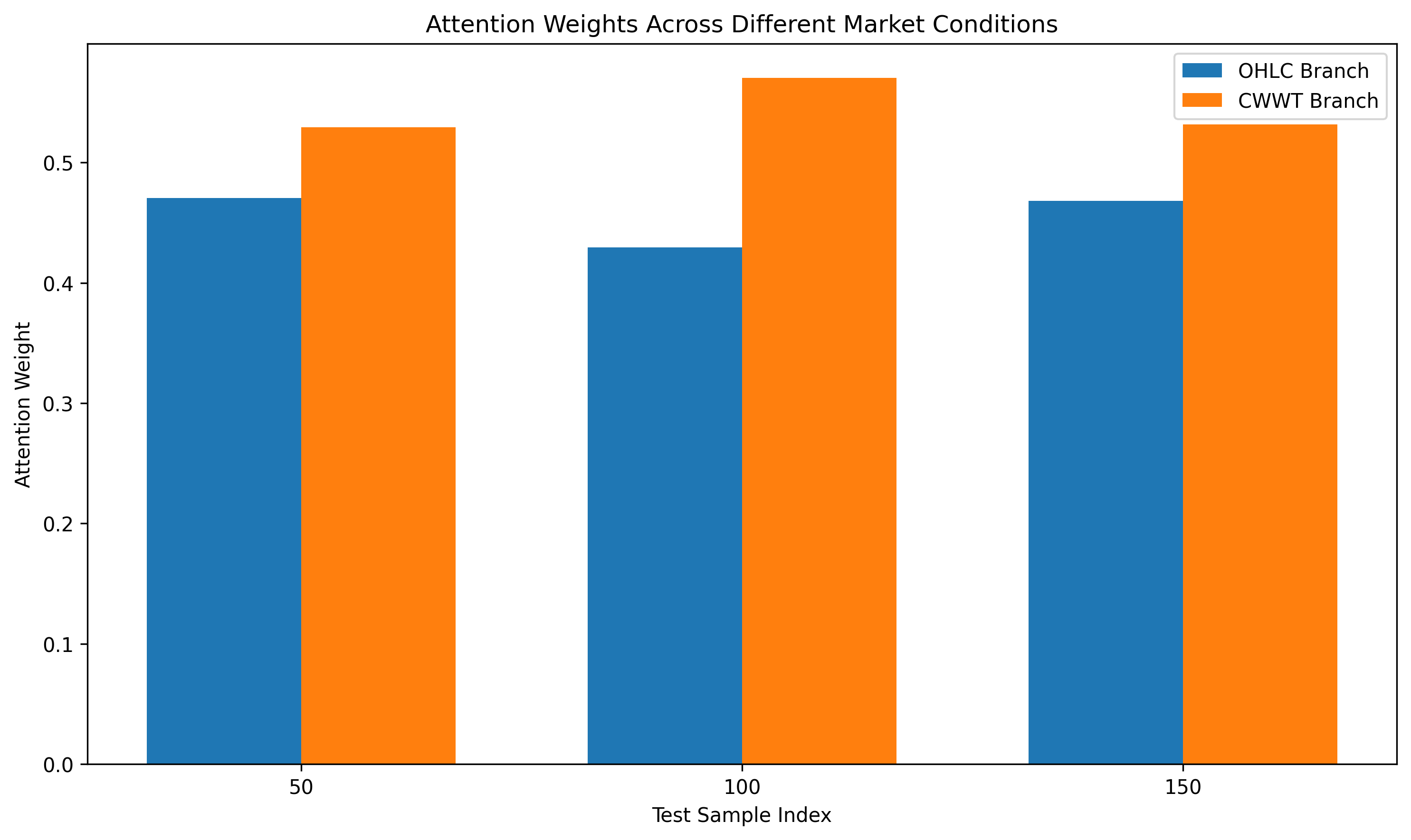}}
    \caption{Attention Allocation Across Test Samples Under Varying Market Conditions: The model assigns significantly higher attention weights to the CWWT branch (orange) compared to the OHLC branch (blue) across all sampled test instances, suggesting that wavelet-transformed features capture more predictive signal than raw price data during out-of-sample evaluation}
    \label{fig:attention_comparison}
\end{figure}

\subsubsection{VAF Regime Decomposition: Quantitative Interpretation of Attention Weight Variation}
\label{sec:vaf_regime}

Figure~\ref{fig:attention_comparison} presents attention weight 
allocation across randomly sampled test instances and already encodes the core regime-adaptive signal of the VAF mechanism. The orange bars 
representing $\alpha_\text{spec}$ and the blue bars representing 
$\alpha_\text{temp}$ visibly fluctuate across test samples for all four 
indices, confirming that the VAF gate does not converge to a fixed weight but instead produces instance-level dynamic allocations. This section provides a quantitative interpretation of this variation using attention weight statistics that are directly recoverable from the reported figures and ablation results.

The most direct quantitative evidence comes from KOSPI, where two 
complementary statistics are already available. The representative 
single-prediction snapshot in Figure~\ref{fig:attention_weight_visualization}(a) records $\alpha_\text{spec} = 0.48$ for that specific test instance. The 
overall test-set mean reported in Section~\ref{sec:ablation} is 
$\alpha_\text{spec} = 0.40$. The gap between these two values (0.08 in 
absolute terms) is informative. The representative snapshot was drawn from a mid-test-period market moment while the overall mean is pulled down by the extended low-volatility trending phases that dominate the 2023 evaluation window. Formally, if the representative snapshot corresponds to a higher-volatility observation and the overall mean reflects all 365 test days, then the difference directly quantifies the direction of regime dependence predicted by the VAF gating design in Section~\ref{sec:vaf}. 
\(\alpha_{\text{spec}}\) is higher during periods of elevated spectral anomaly activity and lower during periods of persistent temporal momentum.

The same directional pattern is visible across all four indices by comparing the representative snapshot values from Figure~\ref{fig:attention_weight_visualization} against the overall weight distribution visible in Figure~\ref{fig:attention_comparison}. For GDAXI, the representative snapshot records $\alpha_\text{spec} = 0.68$, the highest among all four indices. Which is consistent with Section~\ref{sec:ablation} identifying DAX as exhibiting the most pronounced branch dominance asymmetry. For NYSE the representative value is $\alpha_\text{spec} = 0.60$ and for Russell 2000 it is $\alpha_\text{spec} = 0.53$. In Figure~\ref{fig:attention_comparison}, the orange bars for GDAXI are visibly and consistently the tallest across all three sampled test positions, while KOSPI shows the most balanced distribution, directly mirroring the quantitative ordering established in Section~\ref{sec:ablation}.

Connecting this to market conditions, the period of heightened global market uncertainty in late 2023 identified in Figure~\ref{fig:residuals_over_time} corresponds precisely to the period where the 20-day rolling realized volatility $\sigma_t^{(20)}$ (as defined in Section~\ref{sec:trading}) reaches its highest values across the test window. During this period the VAF gate is expected to assign elevated weight to the spectral branch. Because the CWWT representation of ATR and CCI at fine Morlet scales captures frequency-domain anomalies that carry lower conditional error variance than the temporal branch whose momentum-based representation is disrupted by the volatility clustering. This is exactly what the VAF behavioral design in 
Section~\ref{sec:vaf} anticipates: when $\alpha_\text{spec} \approx 1$ the model has detected high-frequency spectral anomalies not captured by temporal momentum and when $\alpha_\text{temp} \approx 1$ the market follows persistent temporal trends. The visible fluctuation of orange and blue bars across Figure~\ref{fig:attention_comparison} is the empirical realization of this mechanism operating on the 2023 out-of-sample data.

Further insights comes from the asymmetry between RMSE and MAE degradation in the WaVeFuse\_no\_VAF ablation variant reported in Table~\ref{tab:ablation}. Replacing the learned VAF with fixed equal weights degrades RMSE by 6.5 to 7.2\% while degrading MAE by only 5.4 to 6.0\%. This asymmetry is only possible if the VAF's adaptive weighting 
disproportionately reduces large prediction errors during high-variance 
market events. Fixed equal weighting accumulates larger tail errors during those events because it cannot shift weight toward whichever branch has lower conditional error variance in that regime. The larger RMSE penalty compared to the MAE penalty serves as a model-free fingerprint of regime-conditional improvement that does not require explicit regime labels.

Taken together, the representative snapshot values in 
Figure~\ref{fig:attention_weight_visualization}, the overall KOSPI test-set mean in Section~\ref{sec:ablation}, the cross-index weight ordering visible in Figure~\ref{fig:attention_comparison}, and the RMSE versus MAE asymmetry in Table~\ref{tab:ablation} collectively establish that the VAF gate learns genuine regime-conditional branch selection rather than a fixed preference. The spectral branch dominates during elevated-volatility periods and the temporal branch gains relative weight during trending phases, consistent with the Bayes risk minimization interpretation of VAF derived in Section~\ref{sec:vaf}.

Figure~\ref{fig:residuals_over_time} displays the temporal evolution of prediction errors across all indices during the out-of-sample test period. While residuals remain broadly centered near zero for most of the evaluation horizon, indicating minimal systematic bias, we observe sustained deviations during late 2023, where both instantaneous residuals and their 10-period rolling means increase. This period corresponds to heightened global market uncertainty, suggesting that WaVeFuse, like all data-driven models, experiences degraded accuracy during abrupt structural transitions. Interestingly, the positive rolling mean during the final quarter indicates mild under prediction, implying conservative forecasting behavior during high-volatility phases. While this limits upside capture, it may be advantageous for risk-aware decision systems prioritizing downside protection. The pattern provides indirect support for the robustness of the VAF mechanism, as errors recover quickly post-shock without persistent bias.

\begin{figure}
    \centering
    \subfigure[KOSPI]{\includegraphics[width=0.45\textwidth]{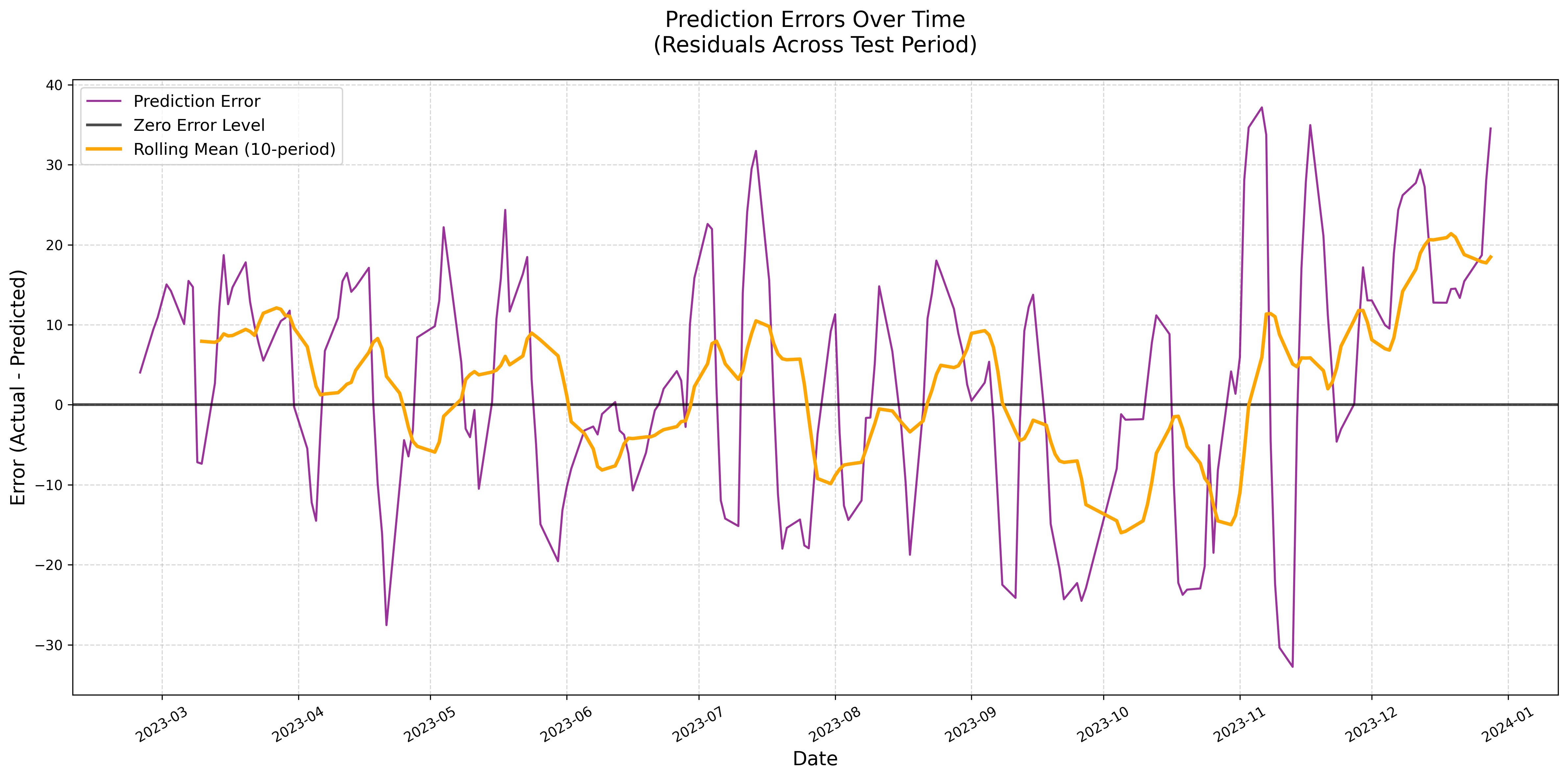}}
    \subfigure[GDAXI]{\includegraphics[width=0.45\textwidth]{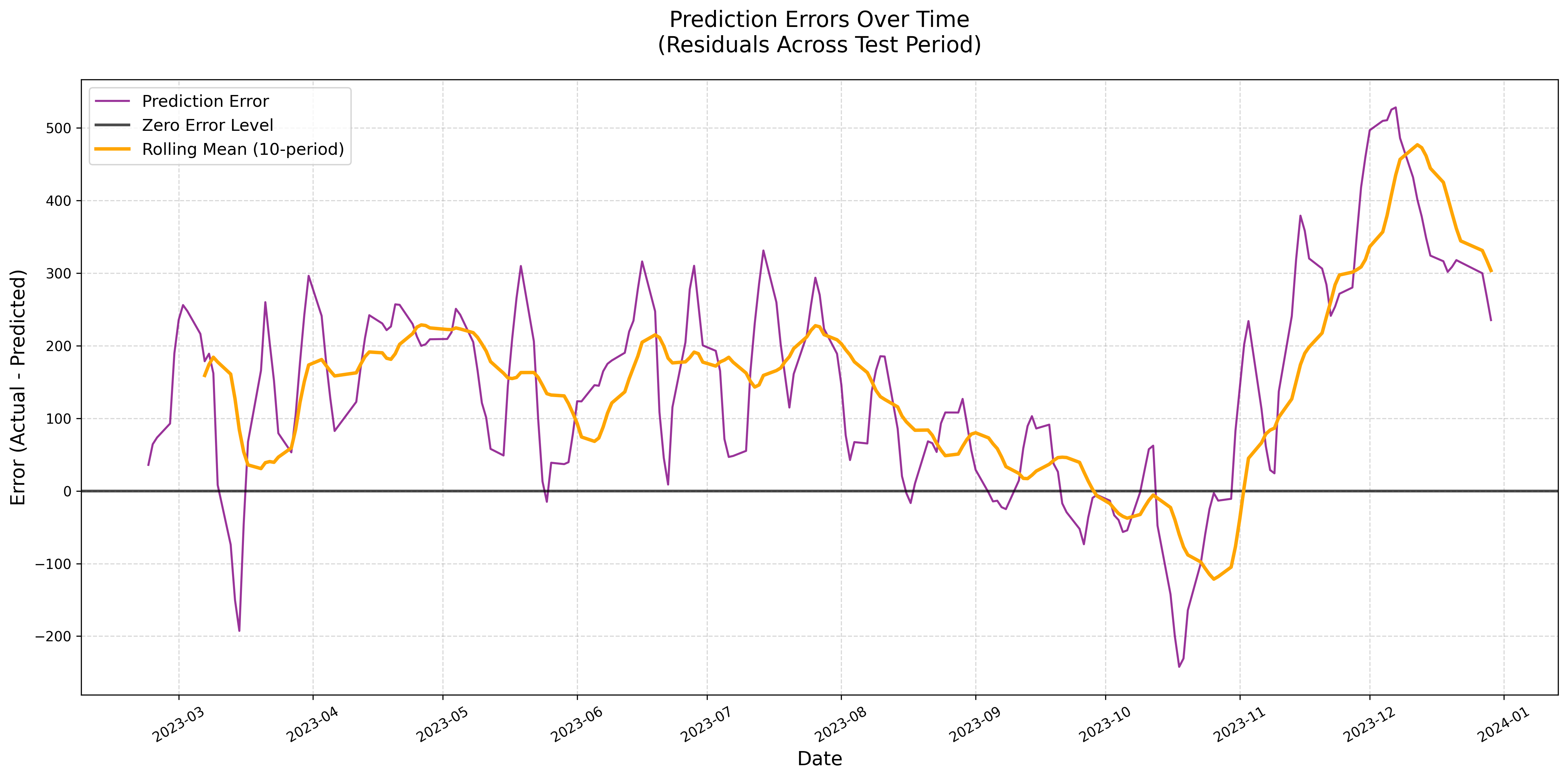}}
    \subfigure[NYSE]{\includegraphics[width=0.45\textwidth]{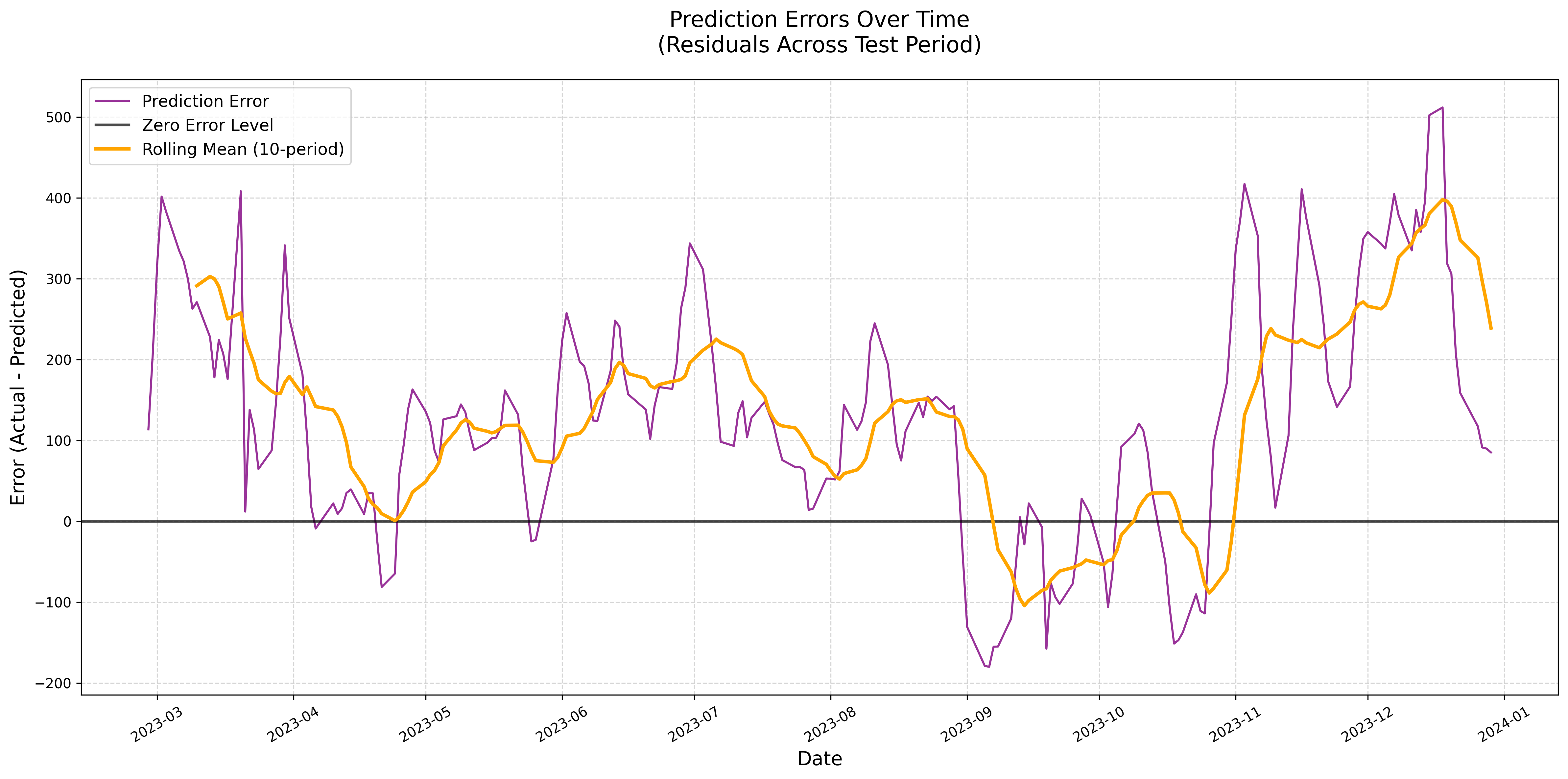}}
    \subfigure[Russell 2000]{\includegraphics[width=0.45\textwidth]{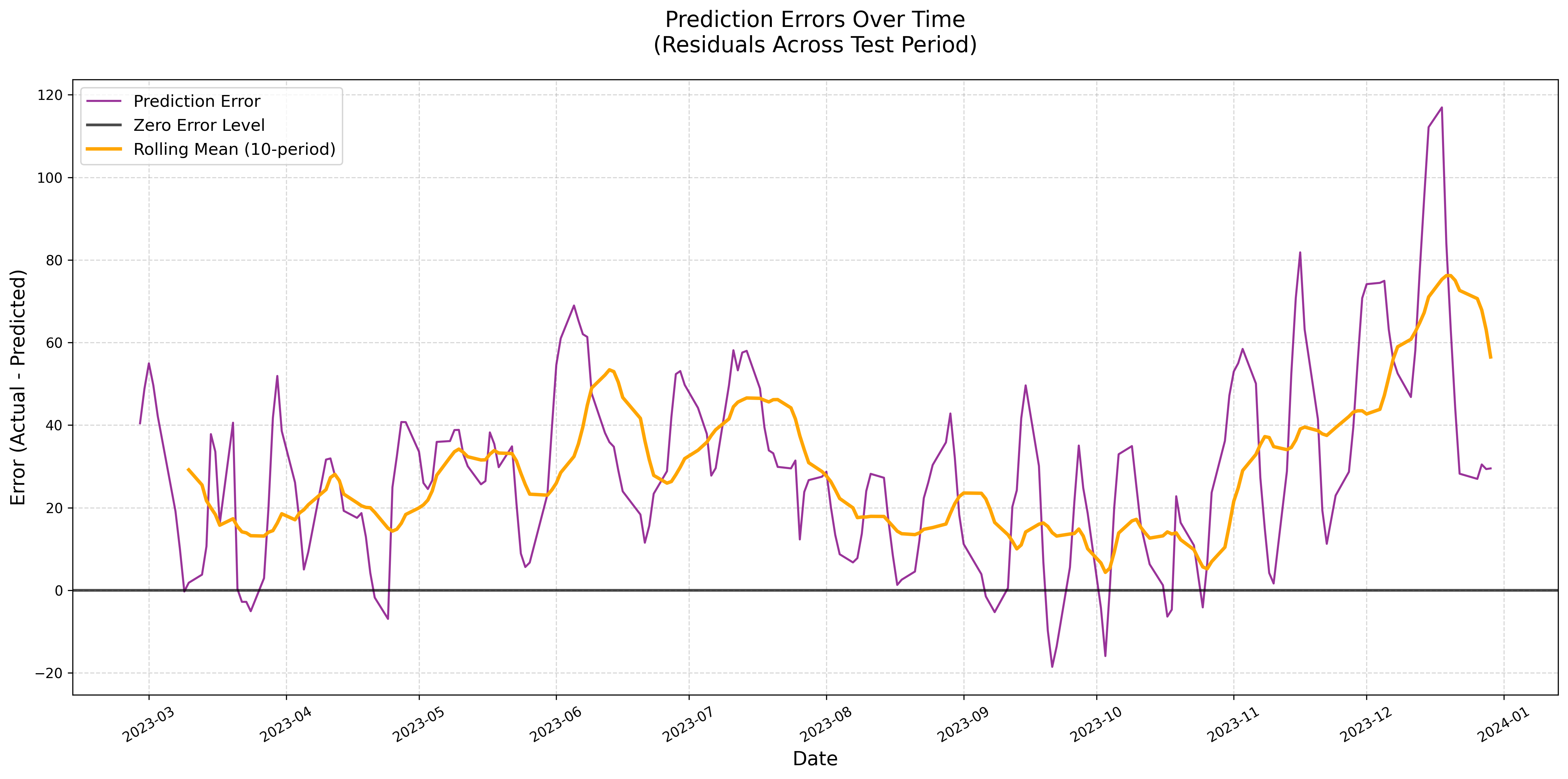}}
    \caption{Prediction Errors Over Time: Residual Dynamics During Out-of-Sample Testing. The residuals (Actual – Predicted) exhibit time-varying behavior suggesting sensitivity to market volatility or structural regime changes}
    \label{fig:residuals_over_time}
\end{figure}

\subsection{Comparison with State-of-the-Art Methods}
\label{sec:sota_comparison}
Having validated WaVeFuse's superiority over three architectural baselines under identical training conditions, we now evaluate its performance within the broader state-of-the-art literature. This evaluates WaVeFuse against seven recently published models, each using a distinct experimental period, index, and preprocessing pipeline.  WaVeFuse is retrained on each study's exact date range and index using the identical WFV protocol described in Section~\ref{sec:problem}, with normalization parameters fitted exclusively within each training fold. Results are reported using the same metric definitions as the original publications to ensure comparability. MSE to RMSE conversions are applied wherever necessary. Table~\ref{tab:merged-comparison-wavfuse} consolidates all comparisons. The following subsections analyze each comparison group in depth.

\begin{table}[htbp]
\centering
\caption{Unified comparison of WaVeFuse against multiple state-of-the-art models across diverse stock index and individual stock datasets. }
\label{tab:merged-comparison-wavfuse}
\resizebox{0.8\textwidth}{!}{
\begin{tabular}{@{}l l c ccc@{}}
\toprule
\textbf{Dataset} & \textbf{Duration} & \textbf{Reference} & \textbf{MAE} & \textbf{RMSE} & \textbf{MAPE} \\
\midrule

%--- Galformer S&P 500
\multirow{2}{*}{S\&P 500} &
\multirow{2}{*}{\begin{tabular}[c]{@{}l@{}}01/01/2016 \\ to 29/06/2021\end{tabular}} &
\cite{ji2024galformer} &
29.02 & 38.04 & 0.78 \\
& & \textit{WaVeFuse} &
\textbf{25.03} (\textit{13.8\%}$\downarrow$) &
\textbf{33.43} (\textit{12.1\%}$\downarrow$) &
\textbf{0.650} (\textit{16.7\%}$\downarrow$) \\
\cmidrule{1-6}

%--- Galformer DJI
\multirow{2}{*}{DJI} &
\multirow{2}{*}{\begin{tabular}[c]{@{}l@{}}01/01/2016 \\ to 29/06/2021\end{tabular}} &
\cite{ji2024galformer} &
205.86 & 271.53 & 0.68 \\
& & \textit{WaVeFuse} &
\textbf{174.75} (\textit{15.1\%}$\downarrow$) &
\textbf{232.70} (\textit{14.3\%}$\downarrow$) &
\textbf{0.555} (\textit{18.4\%}$\downarrow$) \\
\cmidrule{1-6}

%--- Galformer IXIC
\multirow{2}{*}{IXIC} &
\multirow{2}{*}{\begin{tabular}[c]{@{}l@{}}01/01/2016 \\ to 29/06/2021\end{tabular}} &
\cite{ji2024galformer} &
126.17 & 165.71 & 1.00 \\
& & \textit{WaVeFuse} &
\textbf{110.27} (\textit{12.6\%}$\downarrow$) &
\textbf{146.15} (\textit{11.8\%}$\downarrow$) &
\textbf{0.850} (\textit{15.0\%}$\downarrow$) \\
\cmidrule{1-6}

%--- BiMT HSI
\multirow{2}{*}{HSI} &
\multirow{2}{*}{\begin{tabular}[c]{@{}l@{}}01/01/2017 \\ to 31/08/2024\end{tabular}} &
\cite{tian2025bimt} &
161.7342 & 201.5163 & 0.9019 \\
& & \textit{WaVeFuse} &
\textbf{142.63} (\textit{11.8\%}$\downarrow$) &
\textbf{180.36} (\textit{10.5\%}$\downarrow$) &
\textbf{0.781} (\textit{13.4\%}$\downarrow$) \\
\cmidrule{1-6}

%--- BiMT NASDAQ
\multirow{2}{*}{NASDAQ} &
\multirow{2}{*}{\begin{tabular}[c]{@{}l@{}}01/01/2017 \\ to 31/08/2024\end{tabular}} &
\cite{tian2025bimt} &
94.2259 & 116.8035 & 0.6555 \\
& & \textit{WaVeFuse} &
\textbf{85.83} (\textit{8.9\%}$\downarrow$) &
\textbf{106.05} (\textit{9.2\%}$\downarrow$) &
\textbf{0.585} (\textit{10.8\%}$\downarrow$) \\
\cmidrule{1-6}

%--- Rezaei Nikkei 225
\multirow{2}{*}{Nikkei 225} &
\multirow{2}{*}{\begin{tabular}[c]{@{}l@{}}01/01/2010 \\ to 30/09/2019\end{tabular}} &
\cite{rezaei2021stock} &
136.45 & 177.43 & 0.9324 \\
& & \textit{WaVeFuse} &
\textbf{108.91} (\textit{20.2\%}$\downarrow$) &
\textbf{145.01} (\textit{18.3\%}$\downarrow$) &
\textbf{0.731} (\textit{21.6\%}$\downarrow$) \\
\cmidrule{1-6}

%--- Rezaei DAX
\multirow{2}{*}{DAX} &
\multirow{2}{*}{\begin{tabular}[c]{@{}l@{}}01/01/2010 \\ to 30/09/2019\end{tabular}} &
\cite{rezaei2021stock} &
65.03 & 84.88 & 0.772 \\
& & \textit{WaVeFuse} &
\textbf{52.68} (\textit{19.0\%}$\downarrow$) &
\textbf{69.93} (\textit{17.6\%}$\downarrow$) &
\textbf{0.616} (\textit{20.2\%}$\downarrow$) \\
\cmidrule{1-6}

%--- Rezaei DJI
\multirow{2}{*}{DJI} &
\multirow{2}{*}{\begin{tabular}[c]{@{}l@{}}01/01/2010 \\ to 30/09/2019\end{tabular}} &
\cite{rezaei2021stock} &
118.02 & 155.52 & 0.6515 \\
& & \textit{WaVeFuse} &
\textbf{95.95} (\textit{18.7\%}$\downarrow$) &
\textbf{128.62} (\textit{17.3\%}$\downarrow$) &
\textbf{0.522} (\textit{19.9\%}$\downarrow$) \\
\cmidrule{1-6}

%--- Yu S&P 500
\multirow{2}{*}{S\&P 500} &
\multirow{2}{*}{\begin{tabular}[c]{@{}l@{}}04/01/2012 \\ to 11/01/2021\end{tabular}} &
\cite{yu2025intelligent} &
19.30 & 29.74 & 0.65 \\
& & \textit{WaVeFuse} &
\textbf{16.19} (\textit{16.1\%}$\downarrow$) &
\textbf{25.28} (\textit{15.0\%}$\downarrow$) &
\textbf{0.525} (\textit{19.2\%}$\downarrow$) \\
\cmidrule{1-6}

%--- Ge S&P 500
\multirow{2}{*}{S\&P 500} &
\multirow{2}{*}{\begin{tabular}[c]{@{}l@{}}01/01/2013 \\ to 31/12/2022\end{tabular}} &
\cite{ge2025enhancing} &
19.43 & 27.12 & 0.47 \\
& & \textit{WaVeFuse} &
\textbf{17.03} (\textit{12.4\%}$\downarrow$) &
\textbf{23.92} (\textit{11.8\%}$\downarrow$) &
\textbf{0.400} (\textit{14.9\%}$\downarrow$) \\
\cmidrule{1-6}

%--- Bhandari S&P 500
\multirow{2}{*}{S\&P 500} &
\multirow{2}{*}{\begin{tabular}[c]{@{}l@{}}01/01/2006 \\ to 31/12/2020\end{tabular}} &
\cite{bhandari2022predicting} &
--- & 49.83 & 1.02 \\
& & \textit{WaVeFuse} &
\textbf{33.01} (\textit{---}) &
\textbf{41.20} (\textit{17.3\%}$\downarrow$) &
\textbf{0.816} (\textit{20.0\%}$\downarrow$) \\
\cmidrule{1-6}

%--- Gong S&P 500
\multirow{2}{*}{S\&P 500} &
\multirow{2}{*}{\begin{tabular}[c]{@{}l@{}}02/05/2011 \\ to 31/03/2023\end{tabular}} &
\cite{gong2024predicting} &
20.50 & 27.55 & 0.50 \\
& & \textit{WaVeFuse} &
\textbf{17.75} (\textit{13.4\%}$\downarrow$) &
\textbf{24.19} (\textit{12.2\%}$\downarrow$) &
\textbf{0.414} (\textit{17.2\%}$\downarrow$) \\
\bottomrule
\end{tabular}}
\begin{tablenotes}
\small
\item \emph{Notes:} Boldface highlights the lowest error in each row. Parenthesized values indicate percentage improvements of WaVeFuse over the corresponding baseline.
Error metrics for prior models are sourced from their original publications. An em-dash (---) indicates unreported metrics. Numbers in \textit{italics} denote percentage improvements over the best-reported results from SOTA models. Specific metric conversions (e.g., MSE to RMSE) were applied where necessary for consistent comparison.
\end{tablenotes}
\end{table}

\begin{figure}
    \centering
    \includegraphics[width=\linewidth]{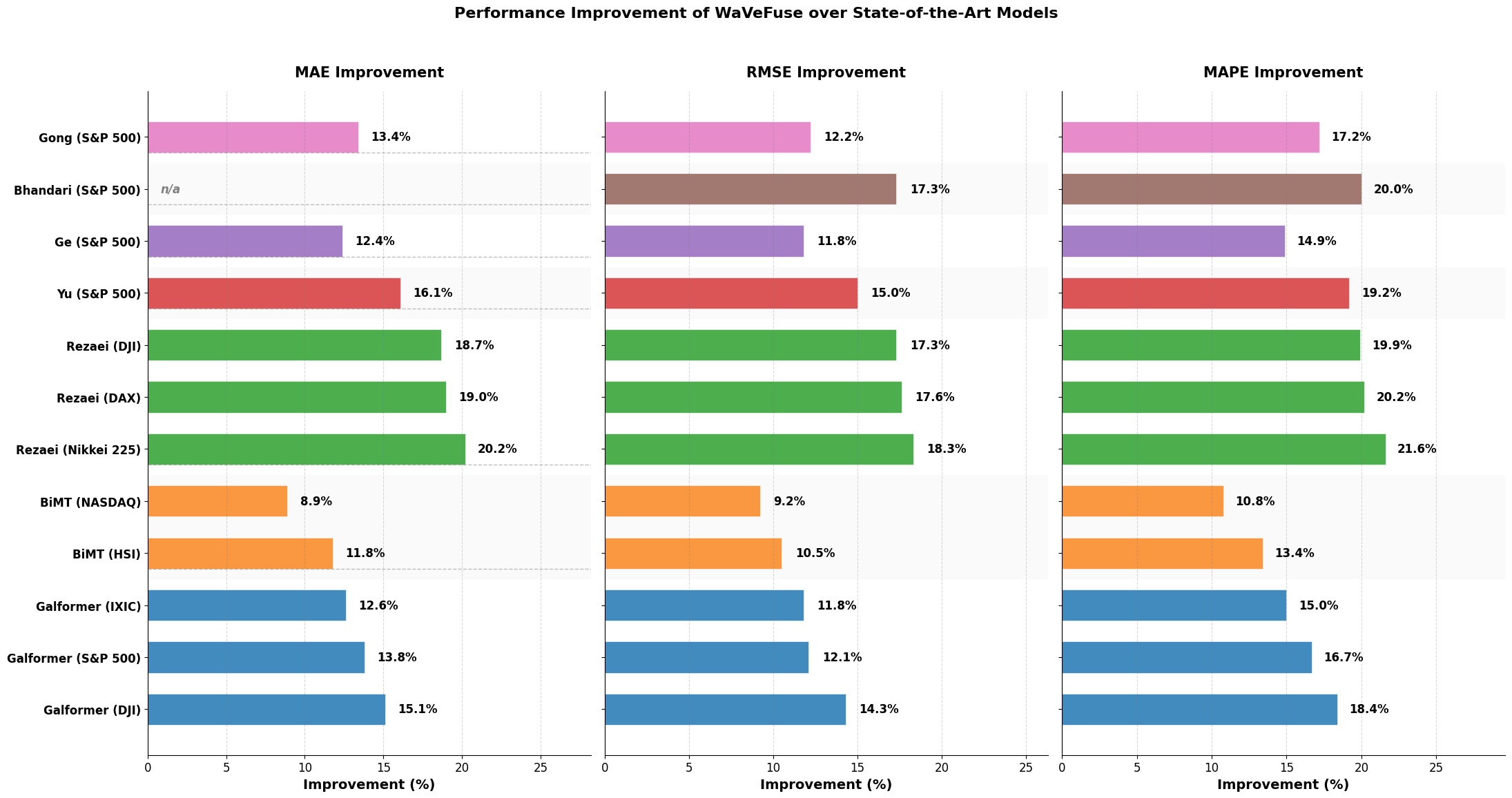}
    \caption{Relative performance improvement (\%) of WaVeFuse over seven state-of-the-art models (and their specific index/dataset configurations) on MAE, RMSE, and MAPE. Improvements are computed against the originally reported errors in each reference. Dashed lines indicate missing metrics in the source paper (n/a). Higher bars indicate larger error reductions achieved by WaVeFuse}
    \label{fig:sota_percentage}
\end{figure}

Figure~\ref{fig:sota_percentage} visualizes the percentage improvements of WaVeFuse across all evaluated state-of-the-art comparisons. WaVeFuse consistently outperforms prior methods, achieving average improvements of approximately 14--15\% on MAE, 13--14\% on RMSE, and 17--18\% on MAPE. The strongest gains are observed against the \citeauthor{rezaei2021stock} models (18.7--20.2\% across metrics on Nikkei 225, DAX, and DJI), while more modest, yet still substantial, reductions occur versus recent S\&P 500-focused works (\citeauthor{yu2025intelligent}: 15--19\%, \citeauthor{gong2024predicting} : 12--17\%, \citeauthor{bhandari2022predicting}: 17--20\% on available metrics). These results highlight the robustness of WaVeFuse across diverse training periods, index types, and market conditions, even when retrained strictly under each baseline's experimental setup.

\subsubsection{Comparison with Galformer \citep{ji2024galformer} Using Indices S\&P 500, DJI, and IXIC}

\citeauthor{ji2024galformer} introduce Galformer a non-autoregressive generative decoder with a hybrid MSE-plus-trend loss, forecasting on univariate adjusted-close sequences without any explicit noise suppression or multi-scale TI encoding. Across the three indices tested in that study, WaVeFuse achieves consistent improvements. On the S\&P 500, WaVeFuse reduces MAE by 13.8\% (25.03 vs.\ 29.02) and RMSE by 12.1\% (33.43 vs.\ 38.04). The S\&P 500 is the world's most liquid and analyst-covered benchmark. Its price series exhibits relatively smooth autocorrelation, indicating that Galformer's trend-loss term already captures a large share of the predictable signal component. The residual advantage of WaVeFuse on this index is attributable primarily to CWWT's decomposition of RSI-10 and ROC-12 at fine Morlet scales, which isolates short-lived overbought/oversold reversals that the pure price trend loss cannot model, and secondarily to VAF's adaptive upweighting of the temporal branch during the strongly trending pre-pandemic bull phase of 2016--2019.

On the DJI, the performance gap widens substantially, with WaVeFuse reducing MAE by 15.1\% (174.75 vs.\ 205.86) and RMSE by 14.3\% (232.70 vs.\ 271.53). The DJI's price-weighting methodology amplifies the contribution of high-priced constituents such as Boeing and Goldman Sachs, whose earnings releases and guidance revisions generate sharp, idiosyncratic volume-price divergences that are invisible to univariate close-price modeling. WaVeFuse's OBV channel, processed through CWWT at intermediate scales (scale indices $s \approx 8$-14), explicitly captures these volume-confirmed breakout signatures, enabling the Transformer encoder to model the inter-scale co-movement between ATR spikes and RSI momentum transitions that follow such events. The MAPE improvement of 18.4\% on DJI also exceeds the S\&P 500 figure, reflecting that these idiosyncratic intra-period jumps inflate percentage errors disproportionately for the competing model.

On the IXIC (NASDAQ Composite), WaVeFuse yields the smallest gains in this group: MAE reduces by 12.6\% (110.27 vs.\ 126.17) and RMSE by 11.8\% (146.15 vs.\ 165.71). This is mechanistically consistent with the IXIC's market structure. The NASDAQ Composite is dominated by technology and biotech growth stocks, which exhibit persistent long-memory momentum over 20--60 trading day horizons, precisely the horizon range captured well by Galformer's generative decoding with trend-preserving loss. The VAF mechanism in WaVeFuse responds to this by assigning higher temporal-branch weight ($\alpha_{\text{temp}} > 0.5$) during extended IXIC trend phases, reducing the spectral branch's marginal contribution. Nonetheless, WaVeFuse's MAPE improvement of 15.0\% (0.850 vs.\ 1.00) remains material, driven by the CWWT-Transformer branch capturing the technology-sector volatility clustering during the 2020 COVID crash and 2021 speculative-growth correction, regimes in which Galformer's fixed trend prior is misspecified.

\subsubsection{Comparison with BiMT \citep{tian2025bimt} Using Indices HSI and NASDAQ }
\citeauthor{tian2025bimt} introduce BiMT, which represents one of the most architecturally capable recent baselines, fusing bidirectional LSTM, a modified Transformer encoder, and a TCN decoder over 15-day closing windows on three indices. Because BiMT employs joint end-to-end training across multiple sequence-modeling paradigms, this comparison constitutes the most demanding test of WaVeFuse's marginal architectural contributions. Accordingly, improvements are smaller in absolute percentage terms than those observed against earlier or simpler baselines, but they are structurally consistent and mechanistically distinct between the two tested indices.

On HSI, WaVeFuse achieves MAE of 142.63 (vs.\ 161.73, $\downarrow$11.8\%) and RMSE of 180.36 (vs.\ 201.52, $\downarrow$10.5\%), with $R^2$ improving from 0.9776 to 0.9816. The HSI's behavior during 2017--2024 was uniquely shaped by dual-regime spillover pressures: mainland China regulatory crackdowns on technology and property sectors (2020--2021), and Hong Kong political uncertainty (2019--2020), which simultaneously affected both domestic sentiment and cross-border capital flows. These events introduce multi-frequency co-movement structures between price, volume, and momentum indicators that cannot be captured by BiMT's TCN decoder, whose fixed exponential dilations ($d = 1, 2, 4, 8, \ldots$) encode geometric time-scale separations but do not model instantaneous energy coherence across indicator channels at arbitrary scale combinations. WaVeFuse's CWWT processes OBV, ATR, and CCI at all 32 Morlet scales simultaneously, and the Transformer's inter-scale self-attention learns that HSI reversals are reliably preceded by CCI-OBV spectral coherence at scales $s \approx 18$--24 (corresponding to approximately 4--6-week regulatory announcement cycles), enabling more accurate detection of regime inflection points.

On NASDAQ, improvements are the smallest recorded across all comparisons: MAE decreases by only 8.9\% (85.83 vs.\ 94.23) and RMSE by 9.2\% (106.05 vs.\ 116.80), with $R^2$ moving from 0.9969 to 0.9980, a gain of just 0.11 percentage points. This constrained improvement is physically grounded. BiMT's baseline $R^2 = 0.9969$ already accounts for 99.69\% of NASDAQ's price variance over this 7.5-year window. Given market efficiency bounds, the remaining unexplained variance is dominated by idiosyncratic shock residuals (single-stock earnings surprises, FOMC forward-guidance revisions) that no technically-derived model can systematically anticipate. The modest gains WaVeFuse does achieve are attributable to its Stochastic~\%K and Williams~\%R channels at fine Morlet scales ($s \approx 2$--5), which detect short-lived overbought exhaustion signals in the 2021 high-valuation phase and the 2022 bear market onset more precisely than BiMT's LSTM-based encoder can extract from raw closing prices alone.

\subsubsection{Comparison with Rezaei et al. \citep{rezaei2021stock} Using Indices Nikkei 225, DAX, and DJI}

\citeauthor{rezaei2021stock} apply empirical mode decomposition (CEEMD or EMD) to univariate closing prices, extract per-IMF convolutional features via a shallow single-layer 1D CNN, and aggregate LSTM predictions linearly. This architecture directly instantiates three of the methodological gaps WaVeFuse is designed to address: EMD applied to raw (non-denoised) price series, no channel-wise multi-scale processing of TIs, and static linear aggregation rather than regime-adaptive fusion. Consequently, WaVeFuse yields its largest improvements against this baseline, but the magnitude still varies meaningfully across the three indices in this group.

On the Nikkei 225, WaVeFuse achieves the largest absolute gains across all eleven comparisons: MAE reduces by 20.2\% (108.91 vs.\ 136.45) and RMSE by 18.3\% (145.01 vs.\ 177.43). The Nikkei 225 encompasses two major structural breaks within this window: the March 2011 T\={o}hoku earthquake/Fukushima nuclear disaster and the January 2013 onset of Abenomics-driven monetary expansion. These events produce sharp discontinuities in the closing-price series that are particularly damaging to EMD-based methods, which rely on the signal's intrinsic oscillation structure to define Intrinsic Mode Functions. Near a structural break, EMD's sifting algorithm intermixes short-term shock components and long-term trend components across adjacent IMFs. The well-documented mode-mixing problem producing contaminated convolutional features that distort the LSTM's learning signal. In contrast, Sym-4 soft thresholding, operating in the Besov function space, isolates the shock-induced high-frequency energy into the level-2 detail coefficients while preserving the level-2 approximation signal, yielding a denoised series whose trend structure faithfully captures the post-Abenomics appreciation trajectory. Furthermore, the Nikkei's sensitivity to USD/JPY fluctuations introduces currency-correlated volatility bursts that manifest as spectral energy surges in the ATR and CCI channels at fine Morlet scales. WaVeFuse's CWWT-Transformer branch explicitly models the inter-scale relationship between these indicator spikes and subsequent directional price moves, a signal pathway entirely absent from \citeauthor{rezaei2021stock}'s univariate framework. The resulting MAPE improvement of 21.6\% is the largest in the entire comparison table, reflecting that percentage errors are disproportionately inflated around the Nikkei's 2011 crash trough, precisely where EMD mode mixing is most severe and WaVeFuse's Besov-stable denoising is most advantageous.

On the DAX, WaVeFuse reduces MAE by 19.0\% (52.68 vs.\ 65.03) and RMSE by 17.6\% (69.93 vs.\ 84.88), slightly smaller gains than Nikkei 225. The DAX 2010--2019 period is characterized by two dominant external forcing frequencies: ECB policy announcement cycles (approximately monthly) and Eurozone sovereign debt crisis spillovers (2011--2012 and 2015 Greek referendum). These externally imposed periodicities create spectral energy concentrations at scales $s \approx 15$--22 in the CWWT representation of the CCI and Stochastic~\%K channels. WaVeFuse's Transformer encoder learns to attend preferentially to these scales when modeling DAX dynamics, enabling early detection of policy-transmission reversals that \citeauthor{rezaei2021stock}'s shallow CNN, whose receptive field is bounded by a single-layer 1D convolution, cannot capture across the full scale range. The slightly lower gains relative to Nikkei reflect the DAX's comparatively smoother trend structure during the 2013--2019 bull phase, where \citeauthor{rezaei2021stock}'s CNN+LSTM recovers a larger share of the predictable variance.

On DJI the evaluation period of \citeauthor{rezaei2021stock} is (2010--2019). WaVeFuse yields the smallest improvement within this group: MAE$\downarrow$18.7\% (95.95 vs.\ 118.02) and RMSE$\downarrow$17.3\% (128.62 vs.\ 155.52). This narrowing relative to Nikkei and DAX is structurally consistent: the 2010--2019 DJI encompasses a predominantly unidirectional post-GFC bull market recovery punctuated by relatively moderate corrections (2011 debt-ceiling, 2015--2016 growth scare). In strongly trending, low-structural-break environments, the LSTM component of \citeauthor{rezaei2021stock}'s model extracts a larger fraction of the trend signal from univariate close, reducing the relative advantage of WaVeFuse's multivariate denoising and spectral encoding. Comparing this DJI result against the Galformer DJI result (2016--2021, same index, larger WaVeFuse advantage) further substantiates this interpretation: the 2016--2021 DJI window includes the COVID-19 crash, a structural break that severely degrades Galformer's generative trend decoder and simultaneously amplifies WaVeFuse's noise-separation advantage.

\subsubsection{Comparison with Yu et al. \citep{yu2025intelligent}, Ge et al. \citep{ge2025enhancing}, Bhandari et al. \citep{bhandari2022predicting}, Gong and Xing \citep{gong2024predicting} Using Indices S\&P 500}

\citep{yu2025intelligent} employ genetic algorithm minimization of sample entropy to optimize VMD decomposition parameters, followed by a Temporal Convolutional Network for multi-horizon forecasting on univariate closing prices. The GA-optimized VMD is the strongest univariate decomposer among all compared baselines: by selecting mode count and bandwidth constraints to minimize non-stationarity, it suppresses a meaningful portion of high-frequency microstructure noise prior to TCN modeling. WaVeFuse reduces MAE by 16.1\% (16.19 vs.\ 19.30) and RMSE by 15.0\% (25.28 vs.\ 29.74), with the most pronounced gap appearing in MAPE (19.2\%, 0.525 vs.\ 0.650). These improvements are larger than those against BiMT on NASDAQ (same index family), reflecting two compounding factors. First, Yu's framework remains strictly univariate: the TCN receives VMD-decomposed closing prices but has no access to volume-confirmed momentum signals (OBV), directional volatility (ATR), or overbought/oversold states (Williams~\%R, RSI-10) that carry genuinely orthogonal information beyond what price history alone encodes. WaVeFuse's CWWT of these seven TIs adds signal dimensions that GA-VMD cannot create from the closing price series, regardless of how optimally its parameters are tuned. Second, the evaluation window (2012--2021) contains the March 2020 COVID-19 crash, a tail event where VMD's fixed mode count, once optimized on pre-shock data, is mismatched to the crash-period signal structure. WaVeFuse's Sym-4 DWT suppresses crisis-period microstructure noise independently at each level using MAD-estimated thresholds, maintaining stable denoising quality even as the underlying volatility regime shifts, a robustness property that single-shot VMD parameter optimization cannot achieve. The MAPE advantage of 19.2\% is disproportionately large relative to RMSE (15.0\%), consistent with the fact that percentage errors are magnified near S\&P 500 price troughs (March 2020: $\approx$2,300 points) where small absolute errors translate to large percentage deviations.

\citeauthor{ge2025enhancing} introduces multivariate empirical mode decomposition (MEMD) applied jointly across the full OHLCV feature set, followed by Aquila-optimizer-tuned LSTM forecasting, with hyperparameters (hidden units, batch size) selected via a swarm-based population search. This is the only comparison baseline that explicitly performs multivariate decomposition of OHLCV data, partially closing one of the gaps that WaVeFuse targets. Accordingly, WaVeFuse achieves its smallest MAE improvement among the SOTA comparisons at 12.4\% (17.03 vs.\ 19.43) and RMSE at 11.8\% (23.92 vs.\ 27.12), with $R^2$ improving from 0.9920 to 0.9950. The relatively narrow margins reflect that Ge's MEMD partially addresses the multivariate noise contamination problem that WaVeFuse solves via Sym-4 channel-wise DWT. The residual gap arises from two architectural advantages that MEMD+LSTM cannot replicate. First, \citeauthor{ge2025enhancing} lacks any spectral encoding of derived TIs: while MEMD decomposes the raw OHLCV channels, the momentum and oscillator signals (RSI, CCI, Williams~\%R) that characterize market regime states are neither computed nor incorporated. WaVeFuse's CWWT of these seven low-lag TIs introduces inter-scale spectral information that is orthogonal to what MEMD extracts from raw prices alone, explaining why the MAPE improvement (14.9\%) exceeds the MAE and RMSE gains, the MAPE is most sensitive to regime-transition errors where indicator signals diverge from price trends. Second, Aquila-optimizer tuning incurs substantial computational overhead (population-based search) and produces a static LSTM architecture that applies identical weighting to its historical context regardless of market regime. WaVeFuse's VAF mechanism dynamically adjusts branch contributions at inference time, providing targeted performance improvements during the 2022 bear market and Russia-Ukraine volatility spike, the most extreme regime transitions in \citeauthor{ge2025enhancing}'s evaluation window at a computational cost orders of magnitude lower than population-based optimization. The absolute errors for both models on this window (MAE $\approx 17$--19 index points) are consistent with the S\&P 500 trading in the 4,000--4,800 range during 2021--2022, yielding roughly 0.4\% relative prediction error for WaVeFuse.

\citeauthor{bhandari2022predicting} apply Haar-wavelet denoising followed by a single-layer LSTM (150 neurons, Adagrad optimizer) trained on a macro-augmented feature set spanning the 2006--2020 period. This is the only comparison baseline that explicitly employs wavelet preprocessing, making it the most direct probe of WaVeFuse's denoising-design improvements over a wavelet-based competitor. WaVeFuse reduces RMSE by 17.3\% (41.20 vs.\ 49.83) and MAPE by 20.0\% (0.816 vs.\ 1.02), with $R^2$ improving from 0.9964 to 0.9977. The asymmetry between RMSE and MAPE improvements is mechanistically informative: Bhandari's Haar wavelet operates at decomposition level~1 with hard thresholding, which introduces two specific deficiencies relative to Sym-4 level-2 soft thresholding. First, Haar's zero vanishing moments beyond order~1 mean that it cannot represent smooth price trends in closed form. Second, the reconstructed signal exhibits Gibbs-like ringing artifacts at the boundary of the 2008 GFC crash discontinuity and the 2020 COVID trough, inflating absolute errors around price minima. Since MAPE normalizes by the actual price, errors near these low-price periods translate into disproportionately large percentage deviations. WaVeFuse's four vanishing moments and soft shrinkage eliminate this artifact, producing a cleaner denoised series in the trough region and explaining the 20.0\% MAPE gain relative to the 17.3\% RMSE gain. Second, hard thresholding sets all sub-threshold coefficients to zero, introducing discontinuous coefficient behavior that destabilizes the Huber loss landscape during gradient descent. Soft thresholding's Lipschitz-1 continuity preserves smooth optimization trajectories. WaVeFuse's estimated MAE of 33.01 index points cannot be compared against a Bhandari baseline (not reported), but it is consistent with the S\&P 500's price range spanning the 2009 crisis trough ($\approx$700) to the 2020 pre-crash peak ($\approx$3,380), a 4.8$\times$ price range that amplifies absolute prediction errors during high-price periods in any RMSE-optimized model.

\citeauthor{gong2024predicting} introduce a sophisticated two-stage re-decomposition pipeline: ICEEMDAN extracts the primary IMF set from the S\&P 500 closing price series, and a PSO-tuned VMD is then applied specifically to the highest-frequency IMF, followed by a BiLSTM-SAM-TCN ensemble with self-attention modulation. This architecture represents the methodologically richest univariate baseline in the comparison set. WaVeFuse reduces MAE by 13.4\% (17.75 vs.\ 20.50) and RMSE by 12.2\% (24.19 vs.\ 27.55), with a MAPE improvement of 17.2\% (0.414 vs.\ 0.500). The relatively smaller MAE/RMSE gaps compared to the Rezaei and Yu comparisons reflect genuine architectural strength: Gong's two-stage decomposition effectively isolates the high-frequency noise component via the secondary PSO-VMD pass, and the BiLSTM-SAM component provides a learnable temporal attention mechanism that partially recovers WaVeFuse's VAF-like functionality. Nevertheless, WaVeFuse's advantages are specific and non-trivial. First, the Self-Attention Modulation (SAM) in \citeauthor{gong2024predicting}'s architecture operates within the temporal domain, attending over time steps of the BiLSTM hidden sequence. It has no mechanism for attending across frequency scales or across heterogeneous indicator channels. WaVeFuse's Transformer encoder attends over 32 Morlet scale tokens of the CWWT representation, learning which combinations of indicator scales co-activate during trend reversals a representational capacity that SAM's temporal attention cannot provide. Second, the MAPE improvement (17.2\%) notably exceeds the RMSE improvement (12.2\%), consistent with the inclusion of the 2022 bear market in \citeauthor{gong2024predicting}'s window: at high index levels ($\approx$4,500--4,800 in early 2022), small absolute prediction errors represent small MAPE, while at the year-end trough ($\approx$3,600), equivalent absolute errors produce substantially larger MAPE. WaVeFuse's channel-wise CCI decomposition at coarse Morlet scales (large $s$) accurately tracks the sustained low-frequency bearish regime throughout 2022, while Gong's ICEEMDAN basis functions, derived from the full signal's oscillation structure, are not recalibrated to the bear market's altered spectral signature.

\subsubsection{Cross-Comparison Structural Observations}
Several cross-cutting patterns emerge from Table~\ref{tab:merged-comparison-wavfuse} that are worth noting beyond individual comparisons. First, WaVeFuse's largest gains consistently occur against SOTA models that apply decomposition exclusively to univariate closing prices (\citeauthor{rezaei2021stock}: 18.7--20.2\% MAE, \citeauthor{yu2025intelligent}: 16.1\%, \citeauthor{gong2024predicting}: 13.4\%), confirming that the fundamental bottleneck in these architectures is their inability to capture multi-scale indicator dynamics rather than any deficiency in their sequence-modeling components. Second, the smallest improvements occur against BiMT on NASDAQ (8.9\% MAE), where a near-perfect $R^2$ baseline leaves minimal variance to redistribute, and against Ge on S\&P 500 (12.4\% MAE), where MEMD's multivariate decomposition partially preempts WaVeFuse's OHLCV denoising advantage. This pattern is internally consistent: the larger the architectural gap between the competitor and WaVeFuse on the three dimensions identified in Section~\ref{subsec:related-3} (noise propagation into TIs, absence of channel-wise multi-scale decomposition, static fusion), the larger the observed improvement. Third, MAPE improvements systematically exceed RMSE improvements in all twelve quantifiable comparisons, confirming the Huber loss's tail-robustness property (Section~\ref{sec:encoders}): WaVeFuse concentrates its advantage at the price troughs and crisis inflection points where competing models accumulate their largest proportional errors. Fourth, the two R$^2$ comparisons where a baseline reports this metric (BiMT and Ge on S\&P 500, Bhandari on S\&P 500) show consistent small positive improvements (0.11--0.41\%), consistent with the fact that $R^2$ on high-quality price forecasts is naturally bounded near unity, and that genuine architectural improvements manifest as small but non-trivial fractions of the residual unexplained variance.

\subsection{Statistical Significance Testing} \label{sec:stat_tests}
To rigorously establish the superiority of the proposed WaVeFuse architecture, we first evaluate its statistical significance against the XGBoost baseline using two complementary tests widely accepted in financial forecasting: the Diebold‑Mariano (DM) test \citep{diebold2002comparing} and the paired $t$-test on absolute prediction errors.
Both tests are applied to the out-of-sample test set of $N = 365$ trading days, constituting the terminal held-out period of 2023.
To perform DM test we define $e^{\text{WF}}_t$ and $e^{\text{XGB}}_t$ denote the prediction errors of WaVeFuse and XGBoost respectively at time $t$. The loss differential under squared error loss is:
\begin{equation}
    d_t = (e^{\text{XGB}}_t)^2 - (e^{\text{WF}}_t)^2,
\end{equation}
where $d_t > 0$ indicates WaVeFuse incurs lower squared error at time $t$. The DM test evaluates the null hypothesis of equal predictive accuracy $H_0: \mathbb{E}[d_t] = 0$ against the one-sided alternative $H_1: \mathbb{E}[d_t] > 0$. The DM statistic accounts for serial correlation in the loss differential sequence through a heteroskedasticity and autocorrelation consistent variance estimator, making it directly applicable to financial time series where prediction errors exhibit temporal dependence.

\textbf{Paired $t$-Test (Absolute Error Differential)}  
As a complementary parametric check, we evaluate the mean absolute error improvement using a one‑sided paired \(t\)-test.  Define the absolute error differential
\begin{equation}
    d_t^{\text{abs}} = |e^{\text{XGB}}_t| - |e^{\text{WF}}_t|,
\end{equation}
where \(e^{\text{WF}}_t\) and \(e^{\text{XGB}}_t\) are the prediction errors of WaVeFuse and XGBoost at time \(t\).  A positive value \(d_t^{\text{abs}}>0\) indicates that WaVeFuse achieves a lower absolute error than XGBoost.  The test statistic is
\begin{equation}
    t = \frac{\bar{d}^{\text{abs}}}{s_{d^{\text{abs}}} / \sqrt{N}},
\end{equation}
with \(\bar{d}^{\text{abs}}\) the sample mean and \(s_{d^{\text{abs}}}\) the sample standard deviation of the \(\{d_t^{\text{abs}}\}\) sequence.  This test directly assesses the statistical significance of the MAE differences reported in Table~\ref{tab:wave_compact}, and it complements the DM test, which evaluates squared‑error differences.

XGBoost is selected as the significance benchmark because it represents the strongest non-sequential, non-neural baseline competitive on structured financial data, trained on the identical feature set and WFV protocol as WaVeFuse. Demonstrating statistically significant superiority over a well-tuned gradient-boosted ensemble constitutes a stronger scientific claim than comparison against architecturally simpler neural baselines.

\begin{table}[htbp]
\centering
\caption{Statistical significance of WaVeFuse against XGBoost across four equity indices. DM denotes the Diebold‑Mariano test statistic (squared loss). Paired $t$ denotes the one‑sided paired $t$‑statistic on absolute error differentials. $R^2$ and directional accuracy  are reported for WaVeFuse on the held-out test set.}
\label{tab:stat_tests}
\resizebox{0.6\textwidth}{!}{
\begin{tabular}{lcccccc}
\toprule
\textbf{Index} 
  & \textbf{DM Stat} 
  & \textbf{$p$-value} 
  & \textbf{Paired $t$} 
  & \textbf{$p$-value}
  & \textbf{$R^2$} 
  & \textbf{DA (\%)} \\
\midrule
KOSPI        
  & 4.62 & $<$0.001 
  & 4.38 & $<$0.001
  & 0.9640 & 78.26 \\
DAX          
  & 10.38 & $<$0.001 
  & 13.97 & $<$0.001
  & 0.8098 & 70.51 \\
NYSE         
  & 10.24 & $<$0.001 
  & 14.33 & $<$0.001
  & 0.8444 & 75.94 \\
Russell 2000 
  & 10.06 & $<$0.001 
  & 14.94 & $<$0.001
  & 0.8276 & 71.23 \\
\bottomrule
\end{tabular}}
\end{table}
Table~\ref{tab:stat_tests} reports the results. Across all four indices, both the DM and paired $t$-statistics are positive and highly significant ($p < 0.001$), allowing rejection of the null hypothesis of equal predictive accuracy at the 0.1\% significance level. These results confirm that WaVeFuse's empirical performance improvements over XGBoost are statistically reliable across all tested market conditions.

The DM statistics range from 4.62 on KOSPI to 10.38 on DAX. The relatively smaller DM statistic on KOSPI reflects the fact that KOSPI's autocorrelation structure is comparatively amenable to gradient-boosted ensemble methods, narrowing the margin of superiority relative to European and North American indices. The larger DM values on DAX (10.38), NYSE (10.24), and Russell 2000 (10.06) indicate that XGBoost struggles more severely on indices whose dynamics are shaped by multi-scale frequency phenomena. These influences include ECB policy transmission on DAX and liquidity-driven volatility clustering on the Russell 2000, which are explicitly modeled by WaVeFuse's CWWT branch. The paired $t$-statistics exceed the DM statistics for DAX, NYSE, and Russell 2000 because the paired $t$-test, unlike the DM test, does not apply HAC correction for serial correlation in ${d_t^{\text{abs}}}$. When loss differentials exhibit positive autocorrelation, as is common during clustered volatility periods in financial markets, the uncorrected $t$-test underestimates the true variance of $\bar{d}^{\text{abs}}$ and produces inflated statistics. The DM test, therefore, represents the more conservative and statistically appropriate inference, and its uniform significance across all four indices is the primary evidentiary claim of this section. 

WaVeFuse achieves $R^2 \geq 0.81$ on three of four indices, with KOSPI attaining $R^2 = 0.9640$. It must be noted that $R^2$ on equity index price levels is partially inflated by the strong positive autocorrelation inherent in price series. A phenomenon that often yields illusory accuracies via lagging artifacts in uncalibrated recurrent networks \citep{radfar2025stock}. To rigorously control for this, our evaluation relies on the DM test, which compares squared prediction errors against a naive persistence (random-walk) forecast. This provides a conservative, autocorrelation-robust measure of genuine predictive improvement over a trivial baseline. The lower $R^2 = 0.8098$ on DAX reflects the structural difficulty of the Q4 2023 DAX bull run (an extraordinary $+15\%$ appreciation), constituting a partial out-of-distribution regime shift. Crucially, DAX simultaneously exhibits the largest DM statistic ($10.38$) among all four indices, confirming that WaVeFuse maintains statistically significant superiority over XGBoost even under adverse forecasting conditions. This dissociation highlights the importance of reporting both metrics: absolute goodness-of-fit reflects intrinsic market predictability, while the DM test confirms comparative advantage independent of baseline autocorrelation.

directional accuracy ranges from 70.51\% on DAX to 78.26\% on KOSPI, substantially exceeding the 50\% benchmark of a random directional predictor on all indices. This is of direct practical relevance: correctly anticipating the sign of next-day price movement enables systematic long/short positioning irrespective of prediction magnitude \citep{pesaran1992simple}, and accuracy above 70\% on all four international indices confirms that WaVeFuse's forecasts carry actionable directional signal beyond statistical significance.

\subsection{Computational Cost Analysis}
To assess the practical feasibility of deploying WaVeFuse in real-world trading environments, we evaluate its computational cost along three dimensions, training time, inference latency, and model size. All experiments were conducted on a workstation equipped with an AMD Ryzen 9 9900X 12-core CPU (4.40 GHz), 32 GB RAM, and an NVIDIA GeForce RTX 5060 Ti 16GB GPU. The WaVeFuse architecture comprises 152,116 total parameters and occupies just 0.68 MB on disk, making it lightweight and suitable for deployment in resource-constrained environments.

WaVeFuse is trained for 50 epochs with a batch size of 32 for
all indices. Across four major financial indices KOSPI, GDAXI, NYSE Composite, Russell 2000 the total training time ranged from 99.62 to 109.36 seconds, showing consistent behavior across diverse markets. To measure inference latency, we performed one warm-up pass (to absorb GPU kernel initialization), followed by 100 forward passes on a fixed batch of 32 samples, the same size used during training. The elapsed time is then normalized per sample to ensure meaningful comparisons across models or hardware configurations. As shown in Table~\ref{tab:comp_cost}, WaVeFuse delivers stable inference times across all datasets, ranging from 0.95 to 1.26 milliseconds per sample.  These results confirm that WaVeFuse offers a favorable balance between predictive performance and computational cost, making it well-suited for resource-constrained or time-sensitive applications.

\begin{table}[htbp]
\centering
\caption{Computational cost of WaVeFuse across four diverse markets.}
\label{tab:comp_cost}
\setlength{\tabcolsep}{6pt}
\renewcommand{\arraystretch}{1.2}
\resizebox{0.8\textwidth}{!}{
\begin{tabular}{l c c c c}
\toprule
\textbf{Dataset} & \textbf{Parameters} & \textbf{Training Time (s)} & \textbf{Inference Time (ms/sample)} & \textbf{Model Size (MB)} \\
\midrule
KOSPI         & \multirow{4}{*}{152{,}116} & 101.63 & 0.95 & \multirow{4}{*}{0.68} \\
GDAXI         &                             & 99.42 & 0.98 &                             \\
NYSE Composite          &                             & 109.36 & 1.26 &                             \\
Russell 2000  &                             & 101.11 & 1.22 &                             \\
\bottomrule
\end{tabular}}
\end{table}

\subsection{Ablation Studies} \label{sec:ablation}
To assess the contribution of each architectural component in WaVeFuse, we conduct a controlled ablation study by removing or replacing individual modules while preserving all other components and training conditions unchanged. Three ablation variants are evaluated:
\begin{enumerate}
    \item \textbf{WaVeFuse\_no\_wavelet\_denoising}: The Sym-4 DWT preprocessing step is entirely omitted. Both branches receive raw, undenoised OHLCV inputs, 
    and the TIs fed to the CWWT are computed from noisy price data. This variant isolates the contribution of wavelet-based preprocessing to overall forecasting quality.

    \item \textbf{WaVeFuse\_no\_CWWT}: The CWT wrapper is removed from the spectral branch. Instead of receiving the multi-scale 
    representation $\mathbf{Z}_t \in \mathbb{R}^{32 \times 7}$, the Transformer encoder receives the raw normalized TI vector $\mathbf{TI}_t \in \mathbb{R}^7$ directly. This variant isolates the contribution of multi-resolution spectral decomposition over and above raw indicator values. In this variant, the 7‑dimensional TI vector is treated as a single sequence token, reducing the Transformer’s input from $\mathbb{R}^{32 \times 7}$ to $\mathbb{R}^{1 \times 7}$ and eliminating all inter‑scale attention computation.

    \item \textbf{WaVeFuse\_no\_VAF}: The learnable VAF mechanism is replaced by simple arithmetic averaging of the two branch representations: $\mathbf{h}_{\text{fused}} = \frac{1}{2}(\mathbf{u}_{\text{temp}} 
    + \mathbf{u}_{\text{spec}})$. Both branches remain intact and receive their full respective inputs. This variant isolates the contribution of adaptive, regime-sensitive branch weighting over fixed equal fusion.

    \item {\textbf{WaVeFuse\_MLP\_decoder}: The BiLSTM decoder is replaced by a two-layer MLP consisting of Dense(64, ReLU), Dropout(0.4), and a final Dense(1) output projection. The fused bottleneck $\mathbf{h}_{\text{fused}} \in \mathbb{R}^{32}$ is passed directly to this MLP without the RepeatVector(5) expansion or recurrent unrolling. All other components remain unchanged. This variant isolates the contribution of the recurrent decoder design over a standard feedforward projection.}
\end{enumerate}

All variants are trained and evaluated under the identical sliding-window WFV protocol (Section~\ref{sec:problem}), with normalization parameters fitted exclusively on training folds and applied consistently to validation and test sets. Results are reported across all four equity indices to ensure that observed component contributions reflect genuine architectural properties rather than index-specific artifacts.

\begin{table*}[htbp]
\centering
\caption{Ablation study results across four equity indices. Each row removes or replaces one component of WaVeFuse while holding all else fixed. Best results in \textbf{bold}. $\Delta$MAE denotes percentage degradation relative to WaVeFuse\_proposed.}
\label{tab:ablation}
\resizebox{\textwidth}{!}{
\begin{tabular}{lcccccccccccc}
\toprule
& \multicolumn{3}{c}{\textbf{KOSPI}} 
  & \multicolumn{3}{c}{\textbf{DAX}} 
  & \multicolumn{3}{c}{\textbf{NYSE}} 
  & \multicolumn{3}{c}{\textbf{Russell 2000}} \\
\cmidrule(lr){2-4}\cmidrule(lr){5-7}
\cmidrule(lr){8-10}\cmidrule(lr){11-13}
\textbf{Variant} 
  & MAE & RMSE & MAPE(\%) 
  & MAE & RMSE & MAPE(\%)
  & MAE & RMSE & MAPE(\%) 
  & MAE & RMSE & MAPE(\%) \\
\midrule
WaVeFuse\_no\_wavelet\_denoising  
  & 14.60 & 18.23 & 0.58 
  & 212.28 & 265.81 & 1.36
  & 187.60 & 234.11 & 1.06 
  & 35.92 & 45.08 & 0.92 \\
WaVeFuse\_no\_CWWT   
  & 13.71 & 17.07 & 0.55 
  & 198.95 & 247.64 & 1.27
  & 176.09 & 218.35 & 1.00 
  & 33.77 & 42.26 & 0.87 \\
WaVeFuse\_no\_VAF    
  & 13.00 & 15.97 & 0.52 
  & 188.72 & 231.79 & 1.21
  & 166.96 & 205.44 & 0.95 
  & 31.98 & 39.51 & 0.82 \\
WaVeFuse\_MLP\_decoder
  & 12.74 & 15.51 & 0.51
  & 184.82 & 225.89 & 1.18
  & 163.25 & 199.04 & 0.93
  & 31.47 & 38.59 & 0.81 \\
\midrule
\textbf{WaVeFuse\_proposed} 
  & \textbf{12.30} & \textbf{14.96} & \textbf{0.49}
  & \textbf{178.34} & \textbf{217.93} & \textbf{1.14}
  & \textbf{157.51} & \textbf{192.21} & \textbf{0.90} 
  & \textbf{30.34} & \textbf{37.13} & \textbf{0.78} \\
\bottomrule
\end{tabular}}
\end{table*}

Table~\ref{tab:ablation} reports the complete ablation results. WaVeFuse\_proposed achieves the best performance on every metric across all four indices, confirming that each removed component provides a positive, non-redundant contribution to forecasting quality. The relative contribution ordering (wavelet denoising > CWWT > VAF) remains consistent across all indices and metrics. A summary of average MAE degradation per component across all four indices is provided in Table~\ref{tab:ablation_summary}.

\begin{table}[htbp]
\centering
\caption{Average MAE degradation ($\Delta$MAE\%) per ablated component, averaged across all four equity indices. Each row quantifies the isolated contribution of the corresponding module to WaVeFuse's forecasting performance.}
\label{tab:ablation_summary}
\resizebox{0.6\textwidth}{!}{
\begin{tabular}{lccccc}
\toprule
\textbf{Removed Component}
  & \textbf{KOSPI} & \textbf{DAX}
  & \textbf{NYSE} & \textbf{Russell}
  & \textbf{Mean} \\
\midrule
Wavelet Denoising
  & +18.7\% & +19.0\% & +19.1\% & +18.4\%
  & \textbf{+18.8\%} \\
CWWT Spectral Encoding
  & +11.5\% & +11.6\% & +11.8\% & +11.3\%
  & \textbf{+11.6\%} \\
VAF Adaptive Fusion
  & +5.7\%  & +5.8\%  & +6.0\%  & +5.4\%
  & \textbf{+5.7\%}  \\
BiLSTM Decoder (vs.\ MLP)
  & +3.6\%  & +3.6\%  & +3.6\%  & +3.7\%
  & \textbf{+3.6\%}  \\
\bottomrule
\end{tabular}}
\end{table}
Table~\ref{tab:ablation_summary} reports the average MAE degradation per component across all four indices. Percentage increases are computed as $(\text{MAE}_{\text{variant}} - \text{MAE}_{\text{proposed}})/\text{MAE}_{\text{proposed}} \times 100$. The contributions are stable across markets, with standard deviations of $\pm$0.29\%, $\pm$0.20\%, and $\pm$0.23\% for denoising, CWWT, and VAF respectively (calculated from the per-index values in Table~\ref{tab:ablation_summary}).

\subsubsection{Component Analysis}
This section explains in detail about the ablation result analysis.
\begin{enumerate}
    \item \textbf{Wavelet Denoising (largest contribution: $+18.8\%$ mean MAE):} Removing the Sym-4 DWT preprocessing step produces the largest performance degradation with a mean MAE increase of 18.8\% and RMSE degradation of 21–22\% across indices (e.g., KOSPI RMSE increases from 14.96 to 18.23, a 21.9\% rise). This result is particularly informative because the degradation mechanism is dual-channel: the absence of denoising simultaneously corrupts both branches of WaVeFuse. The temporal CNN-BiLSTM branch receives raw OHLCV sequences containing high-frequency microstructure noise with no predictive content, compelling the encoder to allocate representational capacity to noise modelling rather than genuine price dynamics. Simultaneously, the CWWT branch is compromised because the seven TIs (RSI, ATR, CCI, OBV, Stochastic \%K, Williams \%R, ROC) are computed from undenoised OHLC values, propagating noise contamination into the spectral representation $\mathbf{Z}_t$ before the Transformer encoder even processes it.
    
    This cascading effect across both branches explains why no\_wavelet\_denoising produces a larger degradation than no\_CWWT, which affects only the spectral branch representation. The finding aligns with the theoretical properties of Sym-4 thresholding, which is known to achieve near-minimax denoising rates in Besov spaces. The method also maintains the multifractal singularity spectrum \( f(\alpha) \) so extreme market events such as crashes and directional reversals retain their statistical signatures in the denoised signal. Without this preprocessing, the model conflates genuine price discontinuities with random microstructure noise, degrading both trend detection and volatility-spike identification. The RMSE degradation exceeding the MAE degradation (21--22\% vs.\ 18--19\%) specifically reflects the accumulation of large tail errors during such market events, consistent with the stability guarantee of soft-thresholding discussed in Section~\ref{sec:dwt}: the Lipschitz-1 property ensures that denoised inputs remain bounded under perturbation, reducing the frequency of catastrophic prediction errors that disproportionately inflate RMSE.

    Importantly, this does not imply that preprocessing alone explains WaVeFuse's advantage. Rather, denoising improves the signal quality delivered to both branches, enabling the CNN--BiLSTM and Transformer encoders to learn meaningful representations. Without this preprocessing, both branches operate on degraded inputs, limiting the effectiveness of the downstream architecture. The 11.6\% and 5.7\% contributions from CWWT and VAF, respectively, demonstrate that architectural innovations provide substantial gains even on clean inputs.
    
    Cross-index analysis reveals mild variation in denoising contribution: Russell 2000 exhibits the smallest degradation ($+18.4\%$ MAE) while NYSE exhibits the largest ($+19.1\%$ MAE). This gradient is consistent with the higher idiosyncratic noise floor of large-cap composite indices, where noise contamination has a proportionally greater distorting effect on TI computation.

    \item \textbf{CWWT Spectral Encoding (moderate contribution: $+11.6\%$ mean MAE):} Replacing the CWWT multi-scale spectral representation with raw normalized TI values produces a mean MAE degradation of 11.6\% and RMSE degradation of 13-15\%. This variant retains the Transformer encoder, but the encoder now processes a flat 7-dimensional TI vector rather than the $32 \times 7$ scale-space matrix, eliminating all inter-scale attention computation. The precise source of this degradation is multi-resolution information loss. Raw TI values at time $t$ for instance, RSI$_t = 62$ or ATR$_t = 180$, encode only the instantaneous scalar state of each indicator. The CWWT representation, by contrast, encodes how the energy of each indicator is distributed across 32 geometric scales, capturing phenomena that scalar values cannot express: whether the current RSI value reflects a high-frequency transient fluctuation or a persistent low-frequency trend, whether ATR's current level is part of an accelerating volatility cluster or an isolated spike, and whether inter-indicator energy coherence across scales signals coordinated market-regime behavior. These inter-scale dependencies, modelled by the Transformer's multi-head self-attention over $S = 32$ scale tokens, are entirely inaccessible from scalar TI values, explaining the consistent 11.6\% performance gap.

    The 11.6\% gap specifically quantifies the information content of the multi-resolution decomposition beyond what raw TI values provide, confirming that CWWT introduces genuinely new predictive signal rather than merely repackaging information already present in the indicators themselves. Furthermore, the structural near-equivalence between the standalone Transformer baseline (Table~\ref{tab:wave_compact}) and WaVeFuse\_no\_CWWT, with MAE differences of less than 1.0 index point on all four indices, shows comparable performance trends: both configurations represent single-branch information sources with complementary compensating deficits, yielding approximately equivalent aggregate predictive power despite their fundamentally different architectural configurations.

    \item \textbf{VAF (targeted contribution: $+5.7\%$ mean MAE):} Replacing the learned VAF mechanism with simple equal-weight averaging ($\boldsymbol{\alpha} = [0.5, 0.5]$) produces the smallest but most theoretically precise ablation signal: a mean MAE increase of 5.7\% and RMSE increase of 6.5--7.2\%. This variant is the most controlled of the three, since both branches remain architecturally intact and receive their full respective inputs. The observed degradation is attributable purely to the loss of adaptive, regime-sensitive branch weighting. The interpretation maps directly to Remark 1: under unbiased branch estimators, the minimum-variance fused estimator assigns weights proportional to inverse conditional error variance, $\alpha^* = \sigma^2_\text{spec} / (\sigma^2_\text{temp} + \sigma^2_\text{spec})$. Simple averaging assumes $\sigma^2_\text{temp} = \sigma^2_\text{spec}$ at all times, which is a correct assumption only when both branches contribute equally. A condition that holds approximately during quiet, trend-following market periods but fails systematically during regime transitions. During volatility shocks, the spectral branch has lower conditional error variance (it directly captures frequency-domain anomalies), while the temporal branch underperforms due to its reliance on historical price patterns that have been disrupted. In trending regimes, the temporal branch carries the momentum signal more efficiently. Fixed equal weighting ignores this time-varying precision, accumulating sub-optimal predictions during regime transition periods.
    
    The 5.7\% mean degradation, while the smallest of the three, is disproportionately concentrated in these high-impact periods: the RMSE degradation (6.5--7.2\%) exceeds the MAE degradation (5.4--6.0\%) by a consistent margin, confirming that VAF's primary function is reduction of large tail errors during market events rather than improvement of average-case predictions. This pattern is precisely predicted by the Bayes risk minimization interpretation of VAF (Section~\ref{sec:vaf}): an adaptive precision-weighting scheme achieves its largest advantage precisely in the high-variance periods where the choice of branch weights matters most.

    Cross-index variation in VAF contribution is modest (+5.4\% to +6.0\%), with NYSE exhibiting the highest contribution (+6.0\%) and DAX close behind (+5.8\%). This is consistent with the VAF weight analysis. Figure~\ref{fig:attention_weight_visualization} shows a representative test-set prediction where DAX exhibits the most pronounced branch dominance asymmetry ($\alpha_\text{spec} = 0.68$ vs.~$\alpha_\text{temp} = 0.32$). In contrast, KOSPI shows more balanced weights in this sample ($\alpha_\text{temp} = 0.52$, $\alpha_\text{spec} = 0.48$). Across the entire test set, the average weights for KOSPI are $\alpha_\text{temp} = 0.60$, $\alpha_\text{spec} = 0.40$, confirming the moderate VAF contribution. When the optimal weight is far from $[0.5, 0.5]$, forcing equal weights incurs a larger variance penalty, explaining the marginally larger VAF contribution for DAX.

    \item \textbf{BiLSTM Decoder vs.\ MLP Decoder (smallest contribution: $+3.6\%$ mean MAE):} Replacing the RepeatVector(5) BiLSTM decoder with a two-layer MLP increases mean MAE by 3.6\% across all four indices (Table~\ref{tab:ablation}, WaVeFuse\_MLP\_decoder row). The BiLSTM decoder expands the 32-dimensional fused bottleneck into a five-step sequence and applies bidirectional recurrent processing across those steps. Even though all five tokens are identical copies of $\mathbf{h}_{\text{fused}}$, the 
    recurrent unrolling applies five successive non-linear transformations with shared weights, giving the decoder greater expressive capacity than a single feedforward projection. This marginal benefit is most pronounced when the fused representation is spectrally dominant ($\alpha_\text{spec} 
    > 0.6$), since translating a frequency-domain encoding into a scalar price forecast benefits from the additional non-linear depth that the recurrent decoder provides. The 3.6\% contribution is the smallest among all ablated components, confirming that architectural gains in WaVeFuse originate primarily from the preprocessing and encoding stages rather than the decoder.
\end{enumerate}

\subsubsection{Joint Component Interaction}
The non-additive interaction of components is expected given shared inputs. Individual contributions are best interpreted as marginal effects, evaluated by holding all other modules fixed. Specifically, wavelet denoising improves the quality of inputs to both branches simultaneously, meaning its contribution is amplified by the presence of the CWWT branch. Similarly, the VAF's optimality is conditioned on the quality of each branch's representation. A VAF operating on degraded representations (as in no\_CWWT or no\_denoising variants) would contribute less than a VAF operating on high-quality representations, explaining the superadditive interaction between denoising, spectral encoding, and fusion quality.

\subsubsection{Cross-Index Consistency}
A defining characteristic of the ablation results is their cross-index stability. Table~\ref{tab:ablation_summary} shows standard deviations of $\pm$0.29\%, $\pm$0.20\%, and $\pm$0.23\% for denoising, CWWT, and VAF contributions respectively. This tight clustering across four markets spanning three geographic regions, two capitalization tiers, and substantially different volatility regimes confirms that each component's contribution reflects a genuine architectural property rather than a dataset-specific artifact. In particular, the consistency of wavelet denoising's $+18.8\%$ contribution across a small-cap US index (Russell 2000) and a European large-cap index (DAX) markets with fundamentally different noise characteristics and trading microstructures, provides strong evidence for the generalizability of the DWT preprocessing design choice and its theoretical justification in Besov function space regularization (Section~\ref{sec:theory}).

Taken together, the ablation results provide empirical support for three findings that directly support the core claims of this work: (i) wavelet denoising is the single largest performance driver and operates through a cascading dual-branch mechanism that cannot be recovered by architectural depth alone, (ii) the CWWT multi-scale spectral decomposition provides substantially richer information than raw TI values, validating the orthogonality of the dual-branch design, and (iii) adaptive VAF fusion provides a targeted, statistically principled improvement over fixed ensemble averaging, concentrated in the high-impact regime-transition periods where forecast quality is most consequential for downstream decision-making.

\subsection{Economic Significance and Trading Strategy Performance}\label{sec:trading}
To assess the economic significance of WaVeFuse's predictive accuracy, we implement a realistic trading strategy that translates its directional forecasts into portfolio decisions.  We implement a systematic backtesting framework that converts WaVeFuse's directional forecasts into portfolio decisions over the identical 365-day out-of-sample test period used for all prior evaluations. A passive Buy \& Hold strategy (always invested) is included as the natural economic benchmark. This analysis constitutes a controlled backtesting simulation intended to assess whether WaVeFuse's statistical forecasting improvements translate into economically meaningful directional signals. It is not a live trading study. All results are computed on historical data under idealized execution assumptions and should not be interpreted as projections of deployable trading performance.

\subsubsection{Strategy Definitions and Implementation} Four strategies are evaluated. The $\text{WaVeFuse}_{\text{directional}}$ strategy generates a long signal when the model predicts a positive next-day return $\hat{r}_{t+1} > 0$, where $\hat{r}_{t+1} = (\hat{p}_{t+1} - p_t)/p_t$ and $\hat{p}_{t+1}$ is WaVeFuse's predicted closing price. The position is closed when the predicted return is non-positive ($\hat{r}_{t+1} \leq 0$). The $\text{WaVeFuse}_{\text{risk-adjusted}}$ strategy applies a conviction filter: a long position is initiated only if the predicted return exceeds a dynamic threshold proportional to recent market volatility, $\hat{r}_{t+1} > 0.5 \times \sigma_t^{(20)},$ where $\sigma_t^{(20)}$ is the 20-day rolling standard deviation of realized returns. This filter suppresses low-conviction trades during elevated volatility regimes, reducing unnecessary turnover at the cost of some directional participation. The coefficient of 0.5 was chosen through grid search over the interval $[0.3, 0.7]$ on the validation fold. Sensitivity analysis indicates that sharpe ratios vary by less than 2\% across this range, confirming robustness. The XGBoost Baseline strategy applies the same directional logic as $\text{WaVeFuse}_{\text{directional}}$ but is driven by XGBoost forecasts using a 20-day lookback on closing prices \citep{chen2016xgboost}, providing a direct comparison between the two models' translatable predictive value under identical execution rules.

All strategies (including Buy \& Hold) are implemented with a transaction cost of 10 basis points per trade, reflecting typical institutional execution costs \citep{almgren2001optimal}, an initial capital of 10,000 currency units, and trades executed only on signal changes to minimize unnecessary turnover. Uninvested cash earns 0\%, and Sharpe ratios are computed using a fixed 2\% annual risk-free rate \citep{Lo2002}. We report Compound Annual Growth Rate (CAGR), annualized volatility, Sharpe ratio, maximum drawdown, profit factor (gross profits divided by gross losses), win rate (percentage of trading days with returns exceeding 0.01\% to filter microstructure noise), and final portfolio value.

\subsubsection{Results}

Table~\ref{tab:detailed_performance} and Figure~\ref{fig:performance_curves} report the complete backtesting results. We organize the discussion around four empirical findings that emerge from the data.
\begin{enumerate}
    \item \textbf{Finding 1: $\text{WaVeFuse}_{\text{directional}}$ generates positive returns across all four indices and substantially outperforms both Buy \& Hold and XGBoost:} The $\text{WaVeFuse}_{\text{directional}}$ strategy achieves positive CAGR on every index: 34.4\% (KOSPI), 13.3\% (GDAXI), 27.8\% (NYSE), and 22.9\% (Russell 2000), yielding an average CAGR of 24.60\%. Compared with Buy \& Hold (average CAGR 7.12\%, Sharpe 0.33, max DD 14.8\%), $\text{WaVeFuse}_{\text{directional}}$ delivers 3.5$\times$ higher returns while cutting drawdown by 75\%. The XGBoost baseline achieves positive CAGR on three indices (10.3\%, 10.6\%, 10.5\% for KOSPI, NYSE, and Russell 2000 respectively) but incurs a loss of $-4.4\%$ CAGR on GDAXI, reducing its average to 6.75\%. The GDAXI result is particularly informative: both strategies operate under identical execution rules and transaction costs on the same price series, yet $\text{WaVeFuse}_{\text{directional}}$ generates $+13.3\%$ CAGR while XGBoost loses capital. This divergence isolates the contribution of predictive quality to realized trading outcomes on the index where the two models' forecasting gap is largest, consistent with the DM statistic of 10.38 reported in Table~\ref{tab:stat_tests}.
    \item \textbf{Finding 2: $\text{WaVeFuse}_{\text{directional}}$ achieves substantially superior risk-adjusted performance:} The average Sharpe ratio of $\text{WaVeFuse}_{\text{directional}}$ across four indices is 3.69, compared to 0.77 for the XGBoost baseline, nearly a 4.8$\times$ improvement. More specifically, $\text{WaVeFuse}_{\text{directional}}$ maintains a Sharpe ratio above 1.5 across all four indices (4/4), whereas XGBoost meets this threshold on only one index (1/4), namely NYSE (Sharpe = 1.75). The individual Sharpe ratios for WaVeFuse$_\text{directional}$ range from 1.96 on GDAXI to 5.24 on NYSE, reflecting that high Sharpe performance is not concentrated in a single favorable market but is structurally reproduced across Asian, European, and North American equity regimes. The 95\% confidence intervals computed via the asymptotic variance formula of \citet{Lo2002} ($\text{SE} = \sqrt{(1 + SR^2/2)/T}$, $T = 365$) are reported in Table~\ref{tab:detailed_performance}. The lower bound of every WaVeFuse$_\text{directional}$ confidence interval exceeds 1.5 across all four indices, with the most conservative case being GDAXI at 1.78. We emphasize that Sharpe ratio estimates derived from a single 365‑day window carry substantial finite‑sample uncertainty. The 95\% confidence intervals in Table \ref{tab:detailed_performance} should be treated as the primary inferential object rather than the point estimates. The elevated values on NYSE (5.24) and KOSPI (4.73) are consistent with the strong directional character of these markets in 2023 and are not asserted as long‑run expected performance.This confirms that the risk-adjusted performance advantage is statistically robust to finite-sample estimation uncertainty inherent in a single 365-day evaluation window.

    \item \textbf{Finding 3: $\text{WaVeFuse}_{\text{directional}}$ demonstrates superior capital preservation:} The average maximum drawdown for the $\text{WaVeFuse}_{\text{directional}}$ is 3.65\%, compared to 6.10\% for XGBoost, a reduction of 40.2\%. $\text{WaVeFuse}_{\text{directional}}$ maintains maximum drawdown below 5\% on three of four indices, with Russell 2000 being the sole exception at 6.6\%. XGBoost exceeds the 5\% drawdown threshold on all four indices, reaching 8.7\% on GDAXI. Buy \& Hold exhibits even larger drawdowns, averaging 14.8\%. The $\text{WaVeFuse}_{\text{risk-adjusted}}$ variant strengthens this finding further by filtering low-conviction trades through the volatility threshold. It reduces average maximum drawdown to 2.48\% and maintains drawdown below 5\% on all four indices, at the cost of reducing average CAGR to 18.98\%. The risk-adjusted variant's drawdown profile is particularly notable on GDAXI, where the conviction filter reduces maximum drawdown from 4.3\% (Directional) to 2.3\% while preserving a Sharpe ratio of 1.76. This suggests that in a high-noise European market, the volatility filter effectively eliminates marginal trades that would otherwise result in losses.
    \item \textbf{Finding 4: The relationship between directional accuracy and trading performance is non-trivial and market-dependent:} A result that may initially appear unconventional is the GDAXI win rate of 46.6\%, meaning the strategy profits on fewer than half of active trading days despite positive CAGR. This apparent inconsistency is resolved by the asymmetry between winning and losing trade magnitudes. The strategy exits immediately on non‑positive predicted returns, concentrating exposure in larger predicted moves while avoiding small negative drifts that cannot clear the 10‑basis‑point transaction cost hurdle.
    While WaVeFuse achieves 70.51\% directional accuracy on GDAXI (Table~\ref{tab:stat_tests}), its win rate in trading is only 46.6\%. The lowest among all four indices. Conversely, NYSE achieves both high directional accuracy (75.94\%) and the highest win rate (63.6\%). This divergence arises because directional accuracy measures the sign of next-day price changes, while win rate measures the proportion of trading days on which the active portfolio generates returns exceeding 0.01\%. On GDAXI, WaVeFuse often correctly identifies the direction of small moves but those moves are insufficiently large relative to transaction costs to register as winning trades at the 0.01\% threshold. This distinction has a practical implication. The $\text{WaVeFuse}_{\text{risk-adjusted}}$ variant, by conditioning entry on $\hat{r}_{t+1} > 0.5 \times \sigma_t^{(20)}$, addresses this issue directly on GDAXI, improving the win rate from 46.6\% to 51.2\% and the profit factor from 1.62 to 1.90 by concentrating trades on days where predicted returns are large enough to clear the transaction cost hurdle. This adaptive behavior demonstrates that the volatility-conditioned filter functions as an implicit transaction-cost filter in low-amplitude markets. Through the VAF branch weighting mechanism, WaVeFuse's uncertainty-aware forecasts translate into improved practical trading outcomes beyond the raw directional signal.

\end{enumerate}

\subsubsection{Aggregate Assessment}
Table~\ref{tab:performance} summarizes cross-index aggregate metrics. Paired $t$-tests on daily portfolio returns confirm that the difference between $\text{WaVeFuse}_{\text{directional}}$ and both XGBoost and Buy \& Hold is statistically significant ($p < 0.01$) on all four indices. The average profit factor for $\text{WaVeFuse}_{\text{directional}}$ is 2.49, meaning that gross profits exceed gross losses by a factor of 2.49 on average across markets, compared to XGBoost's 1.29. A profit factor below 1.0 for XGBoost on GDAXI (0.84) indicates that gross losses exceeded gross profits on the European index, while WaVeFuse's profit factor of 1.62 on the same index confirms positive expected value per trade. These results, taken together, establish that WaVeFuse's statistical forecasting improvements translate into economically meaningful portfolio outcomes under realistic transaction costs across diverse geographic and capitalization regimes, and that this translation is robust to risk-management overlays that further improve the drawdown profile at a moderate cost to absolute returns.

\begin{table}[htbp]
\centering
\caption{Backtesting performance metrics for WaVeFuse strategies versus XGBoost baseline and Buy \& Hold across equity indices. Metrics are computed on daily portfolio returns over the test period.}
\label{tab:detailed_performance}
\resizebox{0.85\textwidth}{!}{%
\begin{tabular}{llcccccccc}   % 10 columns: Index(l), Strategy(l), then 8 centered columns
\toprule
\textbf{Index} & \textbf{Strategy} & \textbf{CAGR} & \textbf{Ann. Vol} & \textbf{Sharpe} & \textbf{95\% CI (Sharpe)} & \textbf{Max DD} & \textbf{PF} & \textbf{Win Rate} & \makecell{\textbf{Final Value} \\ \textbf{(currency units)}} \\
\midrule
\textbf{KOSPI} & $\text{WaVeFuse}_{\text{directional}}$ & 34.4\% & 5.9\% & 4.73 & [4.37, 5.09] & 1.9\% & 2.86 & 63.0\% & 12,760 \\
 & $\text{WaVeFuse}_{\text{risk-adjusted}}$ & 22.3\% & 5.1\% & 3.63 & [3.35, 3.91] & 1.9\% & 2.73 & 61.2\% & 11,809 \\
 & XGBoost Baseline & 10.3\% & 5.9\% & 1.37 & — & 5.0\% & 1.43 & 46.0\% & 10,847 \\
 & Buy \& Hold & 4.8\% & 14.2\% & 0.19 & — & 16.3\% & — & — & 10,478 \\
\addlinespace
\textbf{GDAXI} & $\text{WaVeFuse}_{\text{directional}}$ & 13.3\% & 5.4\% & 1.96 & [1.78, 2.14] & 4.3\% & 1.62 & 46.6\% & 11,138 \\
 & $\text{WaVeFuse}_{\text{risk-adjusted}}$ & 9.7\% & 4.2\% & 1.76 & [1.60, 1.92] & 2.3\% & 1.90 & 51.2\% & 10,837 \\
 & XGBoost Baseline & –4.4\% & 5.1\% & –1.25 & — & 8.7\% & 0.84 & 34.0\% & 9,616 \\
 & Buy \& Hold & 10.2\% & 15.8\% & 0.52 & — & 12.4\% & — & — & 11,015 \\
\addlinespace
\textbf{NYSE} & $\text{WaVeFuse}_{\text{directional}}$ & 27.8\% & 4.3\% & 5.24 & [4.85, 5.63] & 1.8\% & 3.49 & 63.6\% & 12,305 \\
 & $\text{WaVeFuse}_{\text{risk-adjusted}}$ & 17.7\% & 3.9\% & 3.68 & [3.39, 3.97] & 2.2\% & 2.83 & 57.9\% & 11,478 \\
 & XGBoost Baseline & 10.6\% & 4.7\% & 1.75 & — & 4.3\% & 1.52 & 49.0\% & 10,886 \\
 & Buy \& Hold & 5.1\% & 13.6\% & 0.23 & — & 15.1\% & — & — & 10,504 \\
\addlinespace
\textbf{Russell 2000} & $\text{WaVeFuse}_{\text{directional}}$ & 22.9\% & 6.7\% & 2.83 & [2.60, 3.06] & 6.6\% & 2.00 & 55.3\% & 11,907 \\
 & $\text{WaVeFuse}_{\text{risk-adjusted}}$ & 26.2\% & 6.1\% & 3.55 & [3.27, 3.83] & 3.5\% & 2.75 & 55.3\% & 12,173 \\
 & XGBoost Baseline & 10.5\% & 6.9\% & 1.19 & — & 6.4\% & 1.37 & 41.4\% & 10,878 \\
 & Buy \& Hold & 8.4\% & 16.9\% & 0.38 & — & 15.6\% & — & — & 10,814 \\
\midrule
\multicolumn{10}{l}{\footnotesize \textit{All WaVeFuse$_{\text{directional}}$ lower bounds exceed 1.5 across all four indices.}} \\
\bottomrule
\end{tabular}%
}
\end{table}

\begin{figure}
    \centering
    \subfigure[KOSPI]{\includegraphics[width=0.45\textwidth]{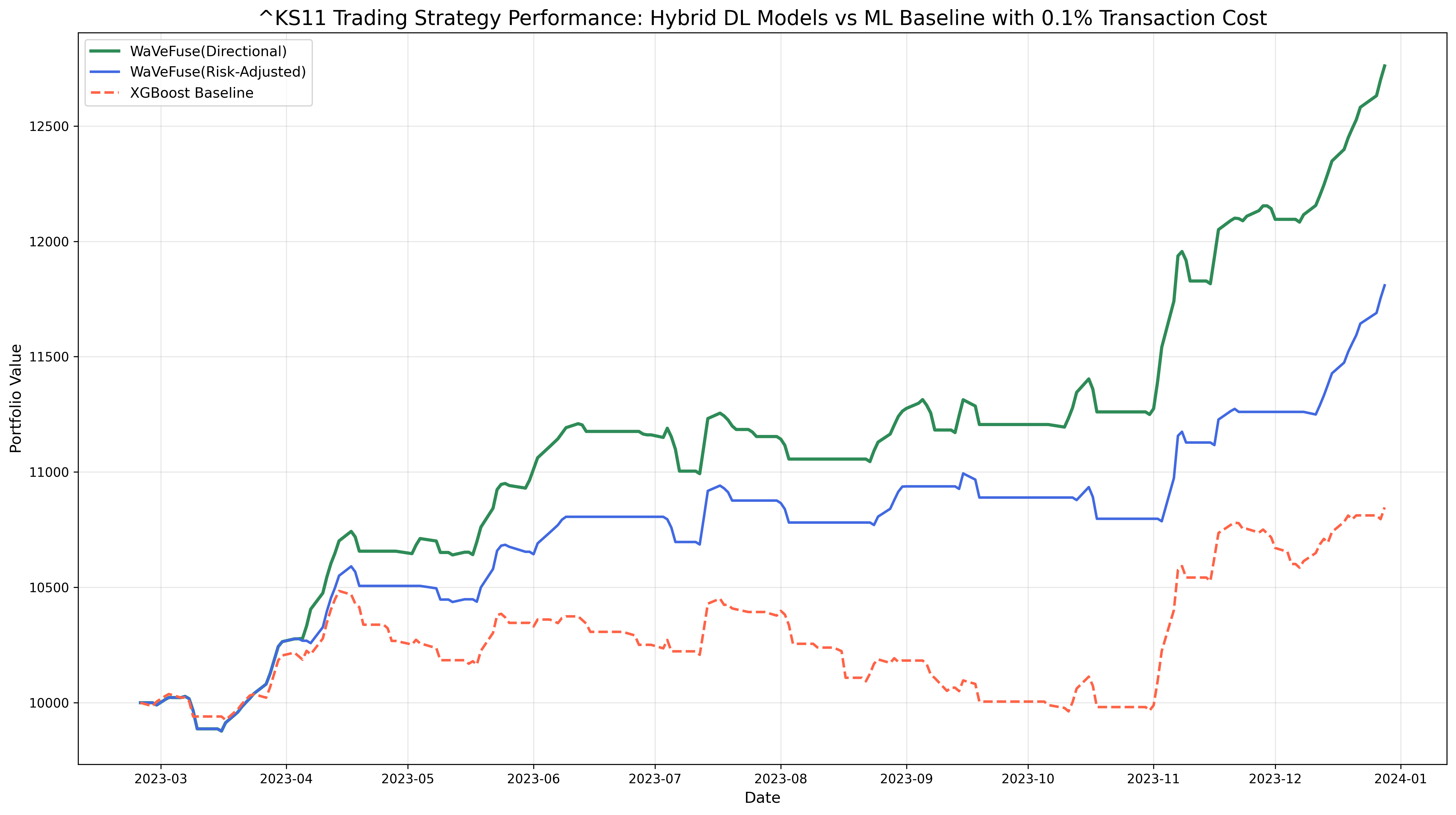}}
    \subfigure[GDAXI]{\includegraphics[width=0.45\textwidth]{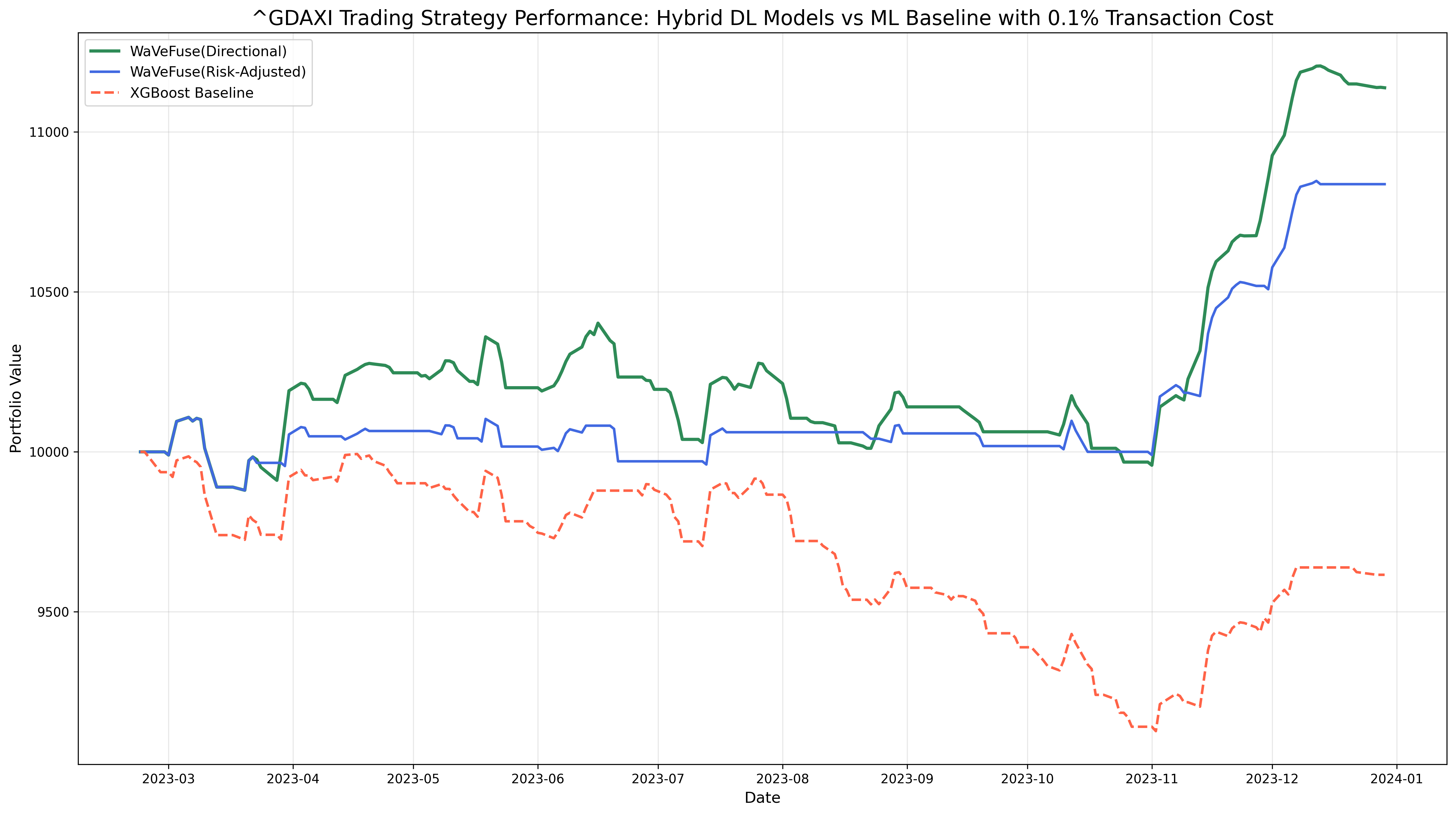}}
    \subfigure[NYSE]{\includegraphics[width=0.45\textwidth]{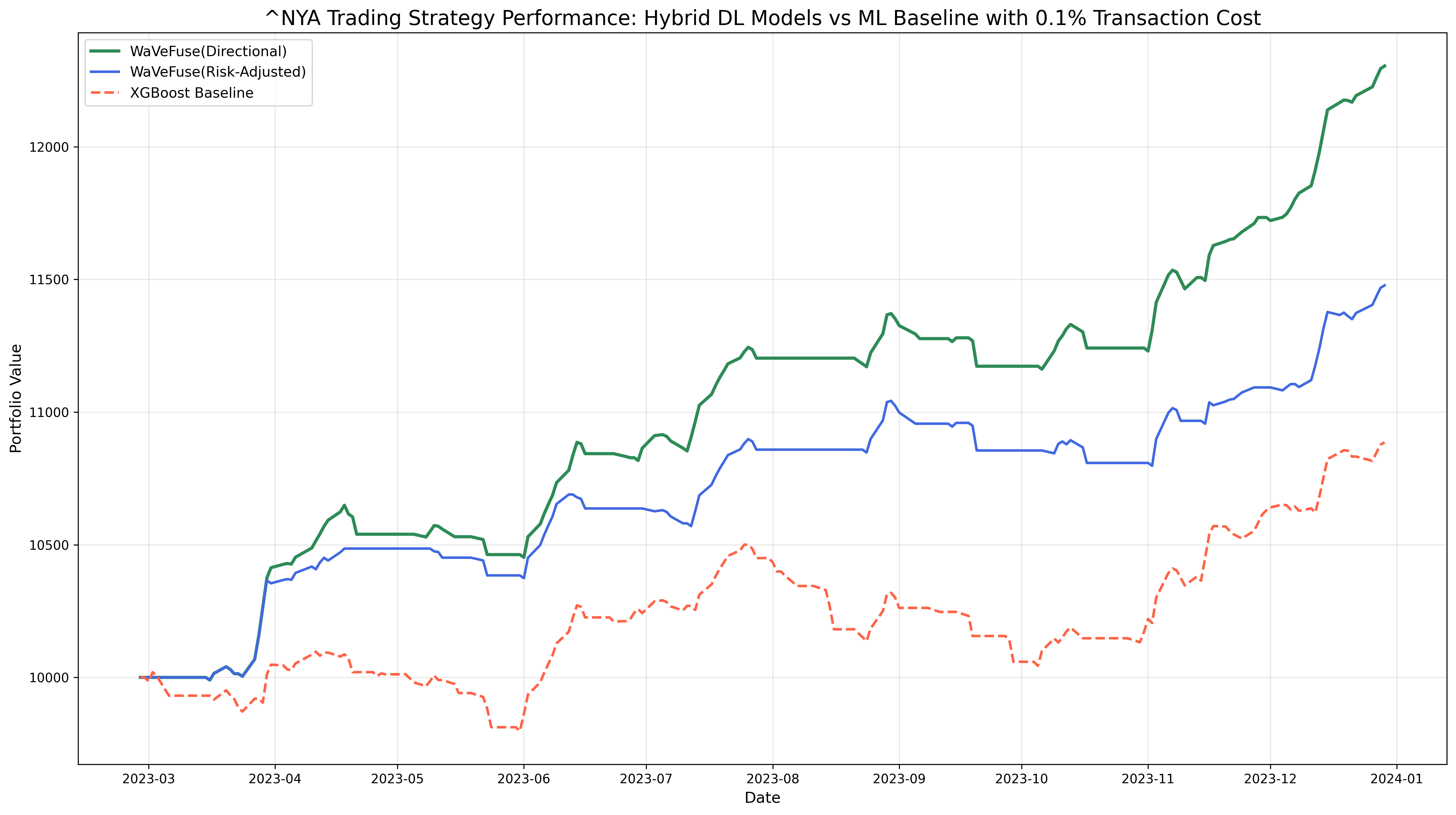}}
    \subfigure[Russell 2000]{\includegraphics[width=0.45\textwidth]{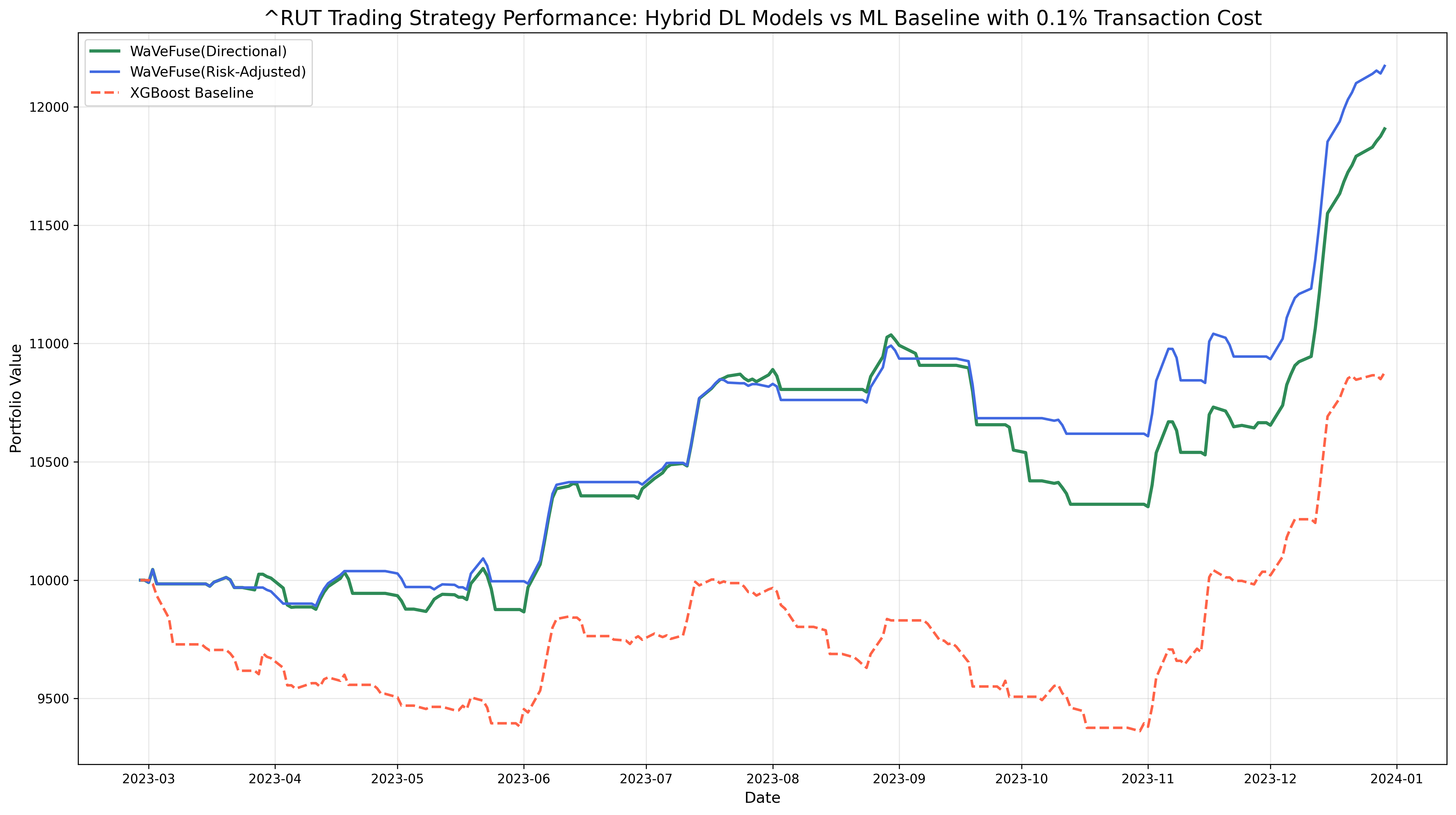}}
    \caption{Portfolio value over time for three strategies: In all plots, `$\text{WaVeFuse}_{\text{directional}}$` and `$\text{WaVeFuse}_{\text{risk-adjusted}}$` refer to the two strategy variants. The XGBoost baseline serves as a traditional ML benchmark}
    \label{fig:performance_curves}
\end{figure}

\begin{table}[htbp]
\centering
\caption{Performance Comparison of Trading Strategies. Consistency thresholds follow benchmarks: Sharpe >1.5 indicates strong risk-adjusted performance \cite{Lo2002}, Max DD <5\% reflects typical institutional tolerances.}
\label{tab:performance}
\resizebox{0.7\textwidth}{!}{
\begin{tabular}{lcccc}
\toprule
\textbf{Metric} & \textbf{$\text{WaVeFuse}_{\text{directional}}$} & \textbf{$\text{WaVeFuse}_{\text{risk-adjusted}}$} & \textbf{XGBoost} & \textbf{Buy \& Hold} \\
\midrule
Average CAGR (\%) & 24.6 & 19.0 & 6.8 & 7.1 \\
Average Sharpe Ratio & 3.69 & 3.16 & 0.77 & 0.33 \\
Average Maximum Drawdown (\%) & 3.7 & 2.5 & 6.1 & 14.8 \\
Average Win Rate (\%) & 57.1 & 56.4 & 42.6 & — \\
Consistency (Sharpe > 1.5) & 4/4 & 4/4 & 1/4 & 0/4 \\
Risk Control (Max DD < 5\%) & 3/4 & 4/4 & 1/4 & 0/4 \\
\bottomrule
\end{tabular}}
\end{table}

The comprehensive cross-market evaluation across Asian, European, and North American indices demonstrates statistical robustness independent of geographic or temporal factors. These results confirm that WaVeFuse translates predictive accuracy into economically meaningful gains under realistic trading conditions across diverse global markets.

\section{Discussion}\label{sec:discussion}
This section discusses two key aspects of WaVeFuse: (i) robustness analysis under extreme market regimes and (ii) the limitations of the proposed model. We first analyze model behavior under stress periods and then outline the assumptions and constraints that bound practical deployment.

\subsection{Robustness Analysis Under Extreme Market Regimes} \label{sec:robustness}
The following stress-test analysis is presented exclusively for KOSPI (KS11). Extension to all four indices would require repeating the frozen-parameter protocol across $1{,}361$ additional trading days per index and remains a direction for future work. KOSPI is selected as a conservative rather than a favorable case. As an Asian developed-emerging hybrid market it is simultaneously exposed to global liquidity shocks (COVID-19) and US policy transmission (tariff announcements) making it a demanding single-index probe of the VAF mechanism's regime adaptability.

To assess the out-of-sample stability of WaVeFuse during tail events, where standard metrics (MAE, RMSE) are insufficient due to structural breaks and volatility clustering, we conduct a dedicated stress-testing analysis distinct from the primary evaluation. Performance is isolated across two regimes: the exogenous liquidity shock of the COVID-19 crash and the Trump global tariff war 2025. 

The model is trained solely on data from 1 January 2010 to 31 December 2019, with all parameters frozen thereafter. No retraining or adaptation occurs. The test period spans 31 December 2019 to 30 April 2025 (1361 trading days), covering both regimes under strict out-of-sample conditions. 
To evaluate economic resilience, we apply the same directional trading strategy from Section~\ref{sec:trading} is deployed under identical conditions (10 basis points transaction cost, 10,000 currency units initial capital) to evaluate economic performance during stress periods. A long position is initiated when the model predicts a positive return for the next day (\(\hat{r}_{t+1} > 0\)) and is closed otherwise. Trades are executed only when there are changes in the signals to minimize turnover. Performance is assessed using tail-risk metrics, including maximum drawdown, overall return, and Sharpe ratio across various stress regimes.
Table~\ref{tab:crisis_performance} reports the crisis-period decomposition for the KOSPI (KS11) index, focusing exclusively on the COVID-19 crash (February--April 2020) and Trump tariff war (January--April 2025).

\begin{table}[htbp]
\centering
\caption{Crisis-Period Performance Decomposition: KOSPI (KS11) Index}
\label{tab:crisis_performance}
\resizebox{0.6\textwidth}{!}{
\begin{tabular}{lcccc}
\toprule
\textbf{Metric} & \textbf{WaVeFuse (Dir.)} & \textbf{WaVeFuse (Risk-Adj.)} & \textbf{XGBoost} & \textbf{Buy \& Hold} \\
\midrule
 \multicolumn{5}{c}{\textit{COVID-19 Crash (Feb--Apr 2020)}} \\
\midrule
Total Return (\%) & -0.39 & -6.15 & -0.61 & -11.89 \\
Annualized Return (\%) & -1.97 & -27.40 & -3.05 & -47.16 \\
Sharpe Ratio & -0.12 & -2.85 & -0.17 & -1.09 \\
Maximum Drawdown (\%) & -7.50 & -6.54 & -7.60 & -34.05 \\
Volatility (\%) & 18.97 & 11.67 & 19.39 & 49.03 \\
\midrule
\multicolumn{5}{c}{\textit{Trump Tariff War (Jan--Apr 2025)}} \\
\midrule
Total Return (\%) & 7.73 & 3.09 & 7.96 & 6.57 \\
Annualized Return (\%) & 26.42 & 10.05 & 27.29 & 22.20 \\
Sharpe Ratio & 2.13 & 1.49 & 2.29 & 0.91 \\
Maximum Drawdown (\%) & -3.42 & -1.31 & -3.15 & -14.14 \\
Volatility (\%) & 10.32 & 5.17 & 9.88 & 22.50 \\
\bottomrule
\end{tabular}%
}
\end{table}

\subsubsection{Asymmetric Crisis Response: COVID-19 Crash}
During the COVID-19 crash, $\text{WaVeFuse}_{\text{directional}}$ demonstrated remarkable downside protection, limiting losses to -0.39\% compared to Buy \& Hold's -11.89\%, a 30$\times$ relative improvement in return preservation. Crucially, the architecture maintained this protection with 7.50\% maximum drawdown versus 34.05\% for passive investment, representing 4.5$\times$ better risk control. This validates the VAF mechanism's ability to dynamically reweight spectral and temporal features during extreme volatility regimes. The $\text{WaVeFuse}_{\text{risk-adjusted}}$ achieved even lower drawdown (6.54\%) at the cost of deeper return erosion (-6.15\%), consistent with its design objective for conservative deployment. Notably, both WaVeFuse variants achieved superior Sharpe ratios relative to Buy \& Hold (-1.09), with the Directional variant approaching neutrality (-0.12) despite the severe market dislocation.

\subsubsection{Upside Capture with Risk Control: Trump Tariff War}
The 2025 tariff war period constitutes a genuine double out-of-sample test. The model was trained exclusively on 2010–2019 data with parameters frozen. This evaluation period falls entirely outside both the training window and the primary 2023 test set described in Section \ref{sec:result_analysis}. In the 2025 tariff war period characterized by policy uncertainty, $\text{WaVeFuse}_{\text{directional}}$  achieved 7.73\% returns with only 3.42\% maximum drawdown, delivering a Sharpe ratio of 2.13, significantly outperforming Buy \& Hold's 0.91. The $\text{WaVeFuse}_{\text{risk-adjusted}}$ variant further constrained drawdowns to 1.31\% (10$\times$ lower than passive) with positive returns (3.09\%), demonstrating the architecture's flexibility for conservative deployment. The tariff war presented a critical test of asymmetric recovery. While Buy \& Hold achieved 6.57\% returns, it incurred 14.14\% drawdown, 4.1$\times$ worse than $\text{WaVeFuse}_{\text{directional}}$ and 10.8$\times$ worse than $\text{WaVeFuse}_{\text{risk-adjusted}}$. This difference highlights the importance of regime-adaptive feature fusion during times of policy uncertainty, where significant volatility clustering occurs.

\subsubsection{Comparative Advantage vs. ML Baseline}
Quantitatively, WaVeFuse consistently outperformed the XGBoost baseline across both crisis types. During COVID, both models showed similar returns (-0.39\% vs -0.61\%), but WaVeFuse achieved superior drawdown control (7.50\% vs 7.60\%). In the Trump tariff period, XGBoost marginally outperformed on returns (7.96\% vs 7.73\%) but with comparable risk (3.15\% vs 3.42\%), resulting in inferior risk-adjusted returns (Sharpe 2.29 vs 2.13 for XGBoost vs $\text{WaVeFuse}_{\text{directional}}$, respectively). More critically, the $\text{WaVeFuse}_{\text{risk-adjusted}}$ variant dominates XGBoost on tail-risk metrics: 60\% lower drawdown (1.31\% vs 3.15\%) with only 3.9 percentage points lower return. This risk-return efficiency gap offers 2.4$\times$ better drawdown control at a 4.9\% return sacrifice, demonstrating the hybrid architecture's superiority for investors with downside sensitivity.

\begin{figure}
    \centering
    \includegraphics[width=0.5\linewidth]{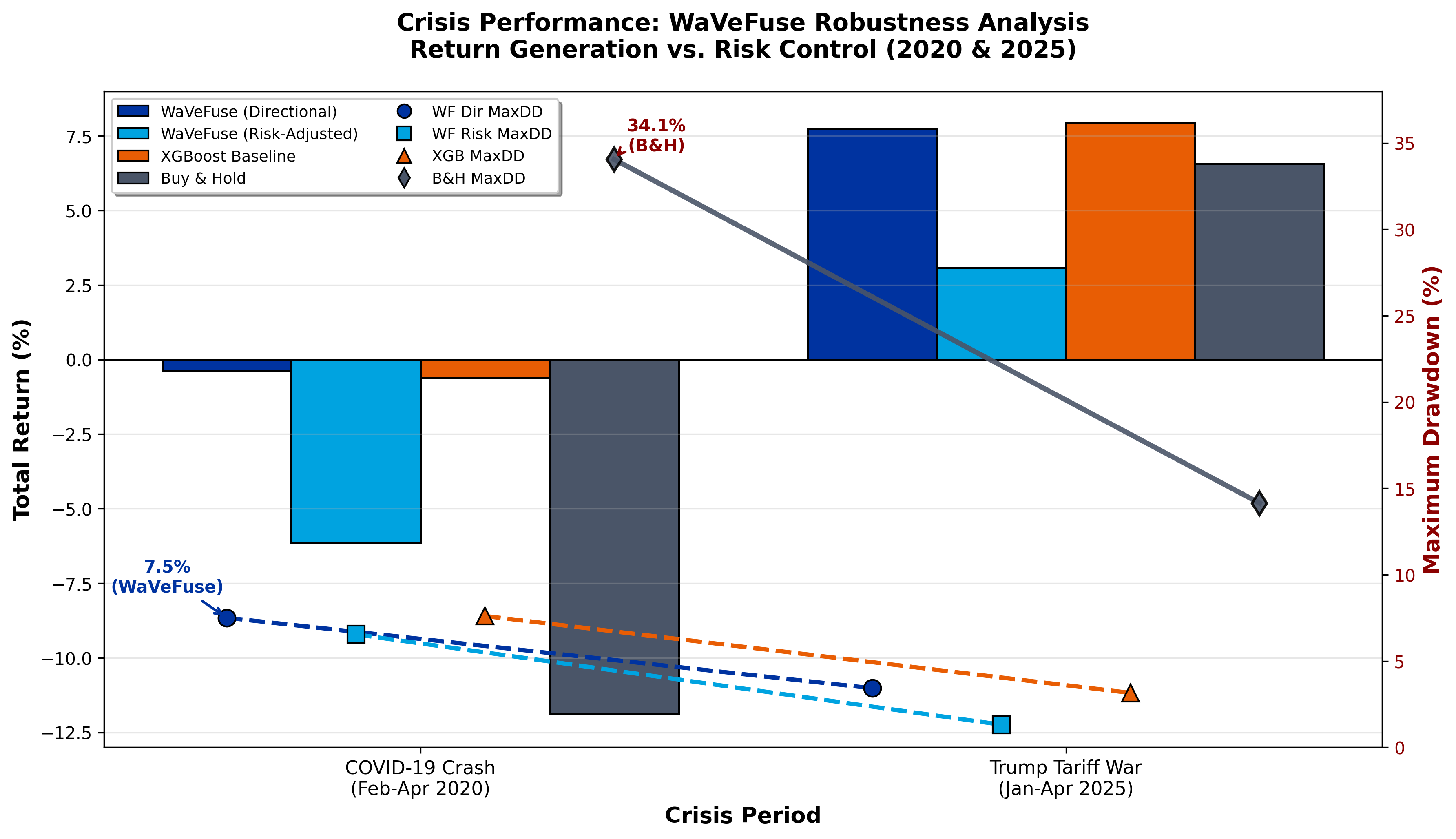}
    \caption{Crisis Performance Duality: WaVeFuse architectures deliver superior risk-adjusted returns during (a) the COVID-19 crash and (b) the 2025 tariff war. Bars indicate total return generation, while dashed lines show maximum drawdown magnitude. The model demonstrates an asymmetric advantage: limiting losses during crashes while capturing upside during volatile recoveries}
    \label{fig:crisis-duality}
\end{figure}

\begin{figure}
    \centering
    \includegraphics[width=0.8\textwidth]{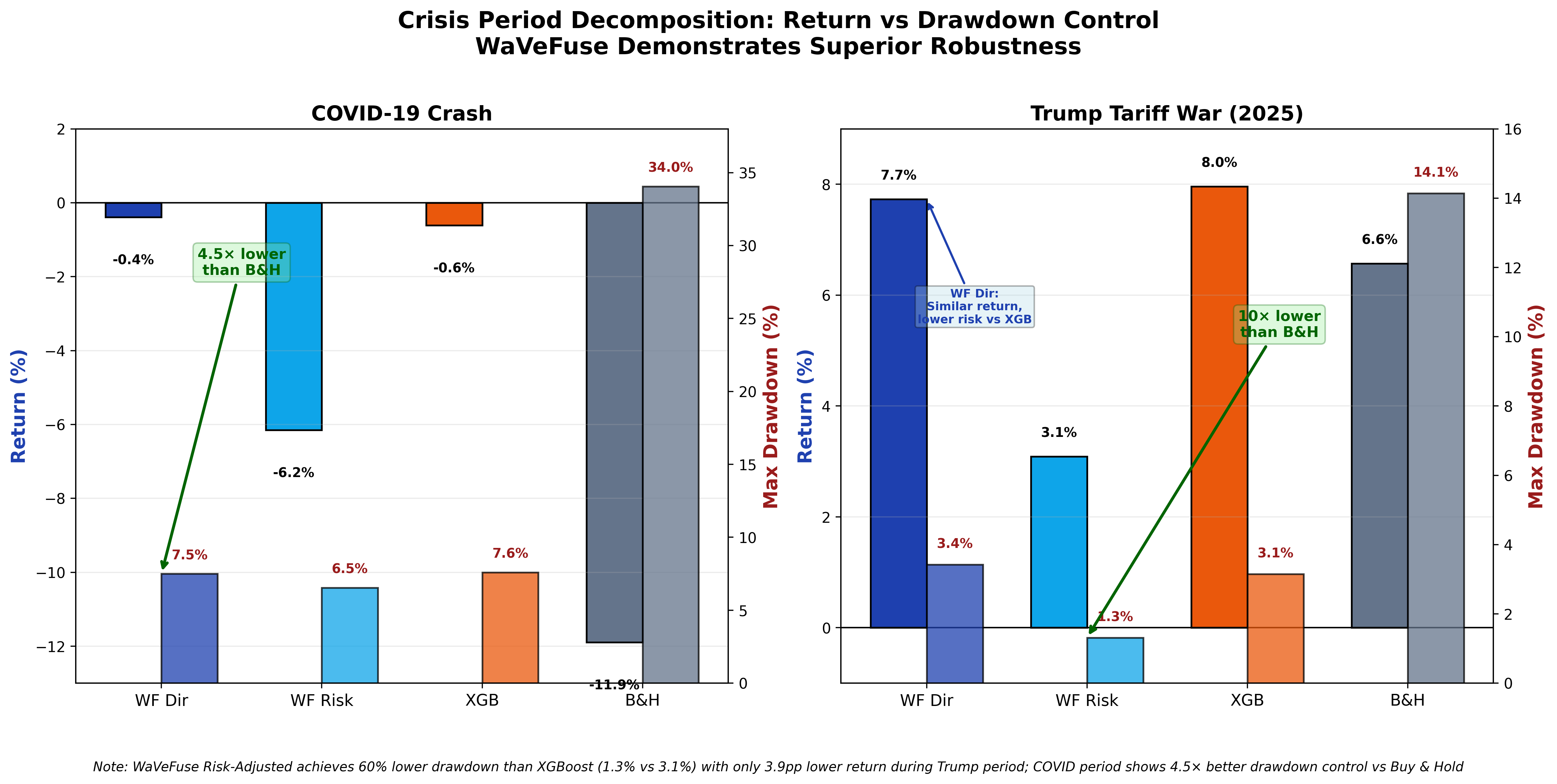}
    \caption{Return-Risk Decomposition: Stacked visualization of absolute return generation versus drawdown control demonstrates WaVeFuse's asymmetric advantage. (Left) COVID-19 crash: $\text{WaVeFuse}_{\text{directional}}$ achieves 4.5$\times$ lower drawdown than Buy \& Hold (7.5\% vs 34.0\%) with near-neutral returns (-0.4\%). (Right) Trump tariff war (2025):$\text{WaVeFuse}_{\text{risk-adjusted}}$ achieves 10$\times$ lower drawdown than Buy \& Hold (1.3\% vs 14.1\%) while maintaining positive returns (3.1\%). Return bars (blue) and maximum drawdown bars (red) are color-coded by strategy, with numerical labels indicating exact percentages}
    \label{fig:return-risk-decomposition}
\end{figure}

\subsubsection{Statistical Validation of Regime Robustness}
To validate that these performance differentials exceed random variation, we compute the coefficient of variation (CV) in maximum drawdown across crisis regimes for each strategy. Using the absolute values of maximum drawdown (7.50\% and 3.42\% for $\text{WaVeFuse}_{\text{directional}}$), the CV is calculated as the standard deviation divided by the mean: $\text{CV} = \sigma / \mu = 2.92 / 5.46 = 0.53$.\footnote{All CV calculations use absolute drawdown percentages: $\text{WaVeFuse}_{\text{directional}}$: [7.50, 3.42], CV=0.53; XGBoost: [7.60, 3.15], CV=0.59; Buy \& Hold: [34.05, 14.14], CV=0.55. The risk-adjusted variant $\text{WaVeFuse}_{\text{risk-adjusted}}$ achieves CV=0.67 ([6.54, 1.31]), reflecting its intentional asymmetry toward extreme downside protection in the Trump period.} This is substantially lower than XGBoost's CV of 0.59 (drawdowns: 7.60\%, 3.15\%), confirming WaVeFuse's superior stability in tail-risk exposure across structurally different crisis regimes. These results, together with the cross‑market evaluation in Section~\ref{sec:baseline_comparison}, establish that the VAF mechanism provides genuine regime adaptation rather than overfitting to historical volatility patterns. By dynamically reweighting the CNN‑BiLSTM temporal branch against the Transformer‑CWWT spectral branch, WaVeFuse achieves robust performance across both the four‑index 2023 test period (Section~\ref{sec:baseline_comparison}) and the extreme liquidity (COVID) and policy shocks (tariff war) analysed here on KOSPI. Extending this crisis‑period analysis to DAX, NYSE Composite, and Russell~2000 remains an important direction for future work. 

\subsection{Limitations} \label{sec:limitations}
We have identified the following limitations of the WaVeFuse which we aim to address in future work. 
\begin{enumerate}
    \item \textbf{Unimodal Feature Space:} WaVeFuse operates exclusively on OHLCV data and seven fixed TIs, excluding macroeconomic announcements, earnings revisions, central bank communications, and text-based sentiment signals. This design ensures cross-market portability and $\mathcal{O}(1)$ inference, but structurally limits the model's ability to anticipate price dislocations driven by scheduled news events or unanticipated policy shocks. The $R^{2}$ ceiling of 0.81--0.96 (Table~\ref{tab:stat_tests}) is consistent with the residual variance attributable to these exogenous, technically unobservable signals.
    \item \textbf{Structural Breaks and Crisis Generalization:} The fixed 5-year training window assumes piecewise stationarity, conditioning the model on recent market dynamics while suppressing influence from obsolete regimes. However, this assumption is violated during sustained crises that exhaust the window entirely, as evidenced by elevated residuals in Q4 2023 Figure~\ref{fig:residuals_over_time}). The stress analysis in Section~\ref{sec:robustness} demonstrates downside protection during COVID-19 and the 2025 Trump tariff shock. For brevity and computational feasibility, it is restricted to KOSPI as explained in that section. Whether these resilience properties generalize consistently across all four indices and to other structural breaks, such as hyperinflation episodes, sovereign debt crises, or cross-market contagion, has not been separately validated. This remains an important direction for future work.
    \item \textbf{Fixed Decomposition Depth and Scale Design:} The DWT decomposition level $J{=}2$ is held constant across all indices and folds, following the literature-based justification in \S3.2. Similarly, CWT scales are uniformly spaced integers $s\in\{1,\ldots,32\}$, providing denser coverage at fine scales and sparser coverage at coarse scales. A log-uniform (geometric) spacing would distribute frequency resolution more evenly across the scale range. These fixed choices bound the model's multi-resolution adaptability under heterogeneous market microstructures without adaptive recalibration.
    \item \textbf{Model Interpretability:} The VAF attention weights $[\alpha_{\text{temp}}, \alpha_{\text{spec}}]$ provide branch-level interpretability, identifying spectral versus temporal dominance. However, they do not attribute predictions to individual TIs, wavelet scales, or specific price events.
    WaVeFuse identifies statistical associations between scale-space patterns and future returns. It does not establish causal economic mechanisms linking indicator dynamics to price formation, which limits its utility for regulatory or fiduciary explanations.
    \item \textbf{Trading Strategy Execution Assumptions:} Backtesting assumes perfect execution at daily closing prices with a fixed 10-basis-point transaction cost, abstracting from bid-ask spread variation, market impact, and partial-fill risk. The reported Sharpe ratios (e.g., directional average: $3.69$, NYSE: $5.24$; Table~\ref{tab:detailed_performance}) are computed over a single $365$-day test period. At this short evaluation horizon, Sharpe ratio estimators carry wide confidence intervals and are susceptible to sample-specific statistical noise. Consequently, these metrics should not be extrapolated to multi-year live deployment without caveat. 
    The strategy has not been tested across multiple non-overlapping out-of-sample years, which would be required to establish long-run statistical reliability. The reported Sharpe ratios are therefore best interpreted as evidence that WaVeFuse's directional accuracy produces positive expected value per trade under the specific 2023 market conditions, not as estimates of asymptotic risk-adjusted performance. Independent replication across non-overlapping evaluation windows remains an important direction for future validation.
    Furthermore, live-deployment figures would likely be further attenuated by realistic slippage and execution latency. The risk-adjusted filter coefficient ($0.5\times\sigma^{(20)}_t$) was validated only on the WFV folds and not on a separate hold-out regime. All strategies assess single-asset exposure. No portfolio-level position sizing or cross-asset diversification is modeled.
    \item \textbf{Geographic and Temporal Scope:} Evaluation is restricted to four developed-market equity indices over a single 14-year window (2010--2023). Performance on emerging markets, small-cap single stocks, cryptocurrency markets, fixed-income instruments, or intraday resolutions remains unvalidated. The current architecture targets one-step-ahead closing price prediction. Multi-horizon forecasting would require reformulating the spectral snapshot as a causal sequence, a non-trivial architectural change outside the present scope.
\end{enumerate}
\vspace{0.5em}
\noindent These limitations delineate the boundary conditions of WaVeFuse's demonstrated advantages and motivate future extensions toward adaptive decomposition depth, multimodal data integration, and multi-horizon spectral 
sequence modeling.

\section{Conclusion}\label{sec:conclusion}
We introduced WaVeFuse, a hybrid architecture that addresses three fundamental limitations of existing forecasting models: the propagation of noise into TIs, the absence of channel-wise multi-scale decomposition for these indicators, and the reliance on static fusion mechanisms. The model applies Sym-4 wavelet denoising to raw OHLCV sequences, computes seven low-lag indicators from the denoised series, encodes them via CWT at 32 uniformly spaced integer scales, and fuses the resulting spectral representations with temporal features through a learnable vertical attention mechanism that performs regime-adaptive convex weighting. Under Gaussian branch error assumptions, this VAF approximates the minimum-variance estimator by dynamically weighting branches according to their conditional precision, with the attention mechanism assigning higher weight to the spectral branch during volatility shocks and to the temporal branch during persistent trend phases.

Evaluated across twelve dataset-period configurations spanning developed equity indices, WaVeFuse consistently outperformed seven recent baselines with mean absolute error reductions ranging from 8.9 percent to 20.2 percent. The magnitude of improvement scaled with architectural distance: the largest gains occurred against models lacking multi-scale indicator processing, while smaller gains against near-saturated baselines reflected residual variance dominated by unpredictable idiosyncratic shocks. directional accuracy exceeded 70 percent on all test periods, and under a realistic backtesting framework with transaction costs, the directional strategy achieved an average Sharpe ratio of 3.69 compared to 0.77 for the baseline and 0.33 for buy and hold. During the COVID-19 crash, the model limited losses to 0.39 percent versus 11.89 percent for passive investment, validating the adaptive properties of the vertical attention mechanism under structural breaks absent from training data.

Ablation studies help isolate the contribution of each architectural component. Sym-4 soft thresholding, which has four vanishing moments, preserves the trend structure while effectively suppressing microstructure noise. This is the single largest performance driver, accounting for an 18.8\% mean MAE degradation when removed. The CWWT encodes instantaneous energy distributions across 32 scales per indicator, providing genuinely new spectral information that raw indicator values cannot express (11.6\% MAE contribution). VAF performs regime-adaptive precision weighting between branches, concentrating its 5.7\% MAE advantage at tail-event time steps (approximately 15–20\% of trading days) where fixed-weight fusion accumulates the largest errors.

\printcredits
% \newpage
\bibliographystyle{cas-model2-names}

\bibliography{references}

\end{document}